%% file: msml.tex
\documentclass[12pt,final]{msml2021}

\usepackage[british]{babel}
\usepackage{bbm}  
\usepackage{color}
\usepackage{enumitem}
\usepackage[T1]{fontenc}
\usepackage{graphbox}
\PassOptionsToPackage{hyphens}{url}\usepackage{hyperref}
\usepackage[utf8]{inputenc}
\usepackage{csquotes}  
\usepackage{nicefrac}
\usepackage{mathtools}
\usepackage{microtype}      
\usepackage{xcolor}

\usepackage{physics, bm}
\input{settings_custom}

\makeatletter
 \let\Ginclude@graphics\@org@Ginclude@graphics 
\makeatother

\title[The Gaussian equivalence of generative models for shallow networks]{The Gaussian equivalence of generative models \\ for learning with shallow neural networks}

 \usepackage{times}

\msmlauthor{%
 \Name{Sebastian Goldt} \Email{sebastian.goldt@sissa.it}\\
 \addr International School of Advanced Studies (SISSA), Trieste, Italy \\
 \Name{Bruno Loureiro}  \Email{bruno.loureiro@epfl.ch}\\
 \addr IdePHICS lab. Ecole F\'ed\'erale Polytechnique de Lausanne\\
 \Name{Galen Reeves}  \Email{galen.reeves@duke.edu} \\
 \addr Department of ECE and Department of Statistical Science, Duke University\\
 \Name{Florent Krzakala} \Email{florent.krzakala@epfl.ch} \\
 \addr IdePHICS lab. Ecole F\'ed\'erale Polytechnique de Lausanne\\
 \Name{Marc M\'ezard}  \Email{marc.mezard@ens.fr} \\
 \addr Laboratoire de Physique de l’Ecole Normale Supérieure, Université
  PSL, CNRS,\\ Sorbonne Université, Université Paris-Diderot, Sorbonne Paris
  Cité, Paris, France\\
 \Name{Lenka Zdeborov\'a}  \Email{lenka.zdeborova@epfl.ch}\\
 \addr SPOC lab. Ecole F\'ed\'erale Polytechnique de Lausanne
}

\begin{document}

\maketitle

\input{abstract}


\input{main_text}

\input{acknowledgements}

\bibliography{nn}

\appendix

\numberwithin{equation}{section}
\clearpage

\input{appendix}

\end{document}

%% file: settings_custom.tex
\definecolor{C0}{HTML}{1f77b4}
\definecolor{C1}{HTML}{ff7f0e}
\definecolor{C2}{HTML}{2ca02c}
\definecolor{C3}{HTML}{d62728}
\definecolor{C4}{HTML}{9467bd}
\definecolor{C5}{HTML}{8c564b}
\definecolor{C6}{HTML}{e377c2}
\definecolor{C7}{HTML}{7f7f7f}
\definecolor{C8}{HTML}{bcbd22}
\definecolor{C9}{HTML}{17becf}

\newcommand{\cF}{\mathcal{F}}

\newtheorem{assumption}{Assumption}

\newcommand{\argmin}{\textrm{argmin}}

\def\mat#1{\text{#1}}
\renewcommand{\vec}[1]{\bm{#1}}

\newcommand{\one}{\boldsymbol{1}} 
 
\newcommand{\Id}{\mathrm{I}}

\newcommand{\bEx}{\ensuremath{\mathbb{E}}}

\newcommand{\ex}[1]{\ensuremath{\mathbb{E}\left[ #1\right]}}

\DeclareMathOperator{\cov}{\mathsf{Cov}}

\DeclareMathOperator{\pmse}{\sf pmse}
\renewcommand{\var}{\mathsf{Var}}

\newcommand{\ussD}{\frac{1}{\sqrt{D}}}
\newcommand{\ussdelta}{\frac{1}{\sqrt{\delta}}}
\newcommand{\usN}{\frac{1}{N}}
\newcommand{\ussN}{\frac{1}{\sqrt{N}}}
\renewcommand{\dd}{\operatorname{d}\!}
\newcommand{\EE}{\mathbb{E}\,}
\newcommand{\extr}{\mathrm{extr}}
\renewcommand{\erf}{\mathrm{erf}}
\renewcommand{\order}[1]{O(#1)}
\def\ind{{\mathbbm{1}}}

\newcommand{\gen}{\mathcal{G}}
\newcommand{\reals}{\mathbb{R}}

\newcommand{\normal}{\mathcal{N}}

\newcommand\samples{T}
\newcommand\sindex{\mu}

%% file: abstract.tex
\begin{abstract}
  Understanding the impact of data structure on the computational tractability
  of learning is a key challenge for the theory of neural networks.  Many
  theoretical works do not explicitly model training data, or
  assume that inputs are drawn component-wise independently from some simple probability
  distribution. Here, we go beyond this simple paradigm by studying the performance of
  neural networks trained on data drawn from \emph{pre-trained generative models}. This is possible due to a Gaussian equivalence stating that the key metrics of interest, such as the training and test errors, 
  can be fully captured by an appropriately chosen Gaussian model. We provide three strands of rigorous, analytical and numerical evidence corroborating this equivalence. First, we establish rigorous conditions for the Gaussian equivalence to hold in the case of single-layer generative models, as well as deterministic rates for convergence in distribution. Second, we leverage this equivalence to derive a closed set of equations describing the generalisation performance of two widely studied machine learning problems: two-layer neural networks trained using one-pass stochastic gradient descent, and full-batch pre-learned features or
  kernel methods. Finally, we perform experiments demonstrating how our theory
  applies to deep, pre-trained generative models. These results open a viable path to
  the theoretical study of machine learning models with realistic data.
\end{abstract}

%% file: main_text.tex
\section{Introduction}
\label{sec:intro}

Consider a supervised learning task where we are given a stream of samples drawn
i.i.d.\ from an unknown distribution $q(x, y)$. Each sample consists of an input
vector $x = (x_i) \in \reals^{N}$ and a response or label $y \in \reals$. Our
goal is to learn a function $\phi_\theta : \reals^N \to \reals$ with parameters $\theta$ that provides an
estimate of~$y$ given~$x$. The performance of such
a model $\phi_\theta$ at this task is assessed in terms of its prediction or test error
$\mathsf{pe}(\theta) = \EE \ell\left[ y, \phi_{\theta}(x) \right]$, where the
expectation is over the data distribution~$q(x, y)$ for a fixed set of
parameters $\theta$ and some loss function $\ell$.  A lot of attention has
recently focused on the importance of training to find models $\phi_\theta$ with low
test error, and specifically on the role of stochastic gradient descent and
various regularisations. Analysing the impact of the data distribution~$q(x, y)$
on learning is equally important, yet it is not well understood.

In fact, theoretical works on learning in statistics or theoretical
computer science traditionally try to make only minimal assumptions on the class
of distributions $q(x, y)$~\cite{Mohri2012, Vapnik2013} or consider the case
where data are chosen in an adversarial (worst-case) manner. In a complementary
line of work that emanated originally from statistical
physics~\cite{Gardner1989, Seung1992, Watkin1993,Engel2001, Zdeborova2016},
inputs are modelled as high-dimensional vectors whose elements are drawn i.i.d.\
from some probability distribution. Their labels are either assumed to be
random, or given by some random, but fixed function of the inputs, see
Fig.~\ref{fig:setup}~(a). This approach, known as the teacher-student setup, has
recently experienced a surge of activity in the machine learning
community~\cite{zhong2017recovery, tian2017analytical,
  du2018gradient,soltanolkotabi2018theoretical,aubin2018committee,
  saxe2018information, Baity-Jesi2018, goldt2019dynamics,
  ghorbani2019limitations, yoshida2019datadependence, gabrie2020meanfield,
  bahri2020statistical, zdeborova2020understanding, advani2020highdimensional}

\begin{figure}[t!]
  \centering
  \includegraphics[width=.99\textwidth]{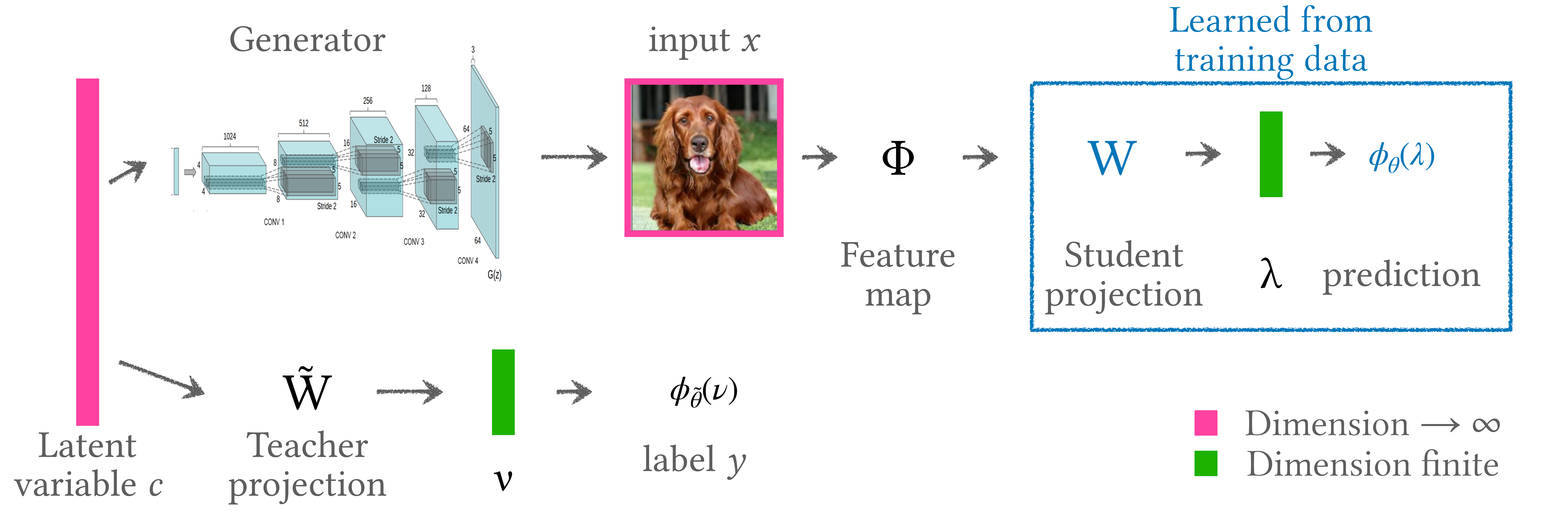}
  \caption{\label{fig:setup} \textbf{The deep hidden manifold: going beyond the
      i.i.d.\ paradigm for generating data in the teacher-student setup.}
    We analyse a setup where samples $(x, y)$ are generated by first drawing a
    latent vector $c \sim \normal( 0 , I_D)$. The input~$x$ is obtained by
    propagating the latent vector through a (possibly deep) generative network, $x=\gen(c)$. The
    label $y$ is given by the response of a two-layer teacher network to the
    latent vector. We then analyse in a closed form learning via a two-layer neural network, or with a single layer neural network after a projection trough a fixed, but not necessarily random, feature map. The sketch of the generator is taken
    from~\citet{radford2015unsupervised}, whose deep convolutional GAN is one of
    the generators we use in our experiments in
    Sec.~\ref{sec:experiments}.}
\end{figure}

\paragraph{The deep hidden manifold} In this manuscript we go beyond the i.i.d.\ paradigm of
the teacher-student setup by extending the hidden manifold model analysed
in~\cite{goldt2019modelling, gerace2020generalisation}. Fig.~\ref{fig:setup}
gives a visual overview of the components of the model. We draw the inputs $x$
from a generative model $\gen: \reals^D \to \reals^N$ of depth $L$. These models
transform random uncorrelated latent variables $c=(c_r)\in\reals^{D}$ into
correlated, high-dimensional inputs which follow a given target distribution via
\begin{equation}
  \label{eq:generator}
  x = \gen(c) = \gen^L \cdots \gen^3 \circ \gen^2 \circ \gen^1(c), \qquad c \sim
  \mathcal{N}(0, I_D),
\end{equation}
where $\circ$ denotes the chaining of layers $\gen^\ell$, which could be
fully-connected, convolutional~\cite{fukushima1982neocognitron,
  lecun1990handwritten}, applying batch norm~\cite{ioffe2015batch} or an
invertible mapping $\gen^\ell: \mathcal{R}^D \to \mathcal{R}^D$ as they are used
normalising flows. We thus replace i.i.d.~Gaussian inputs with realistic images
such as the one shown in Fig.~\ref{fig:setup}. While~\cite{goldt2019modelling,
  gerace2020generalisation} only studied generative models with a single layer
of weights, we allow the generator to be of arbitrary depth $L$, thus including
important models such as variational auto-encoders~\cite{kingma2014auto},
generative adversarial networks (GAN)~\cite{goodfellow2014generative}, or
normalising flows~\cite{tabak2010density, tabak2013family,
  rezende2015variational}. 

\textbf{The label} for each input is obtained from a two-layer teacher network
with $M$ hidden neurons and parameters
$\tilde \theta = (\tilde v \in \reals^M, \tilde W \in \reals^{M \times D})$
acting on the \emph{latent representation} $c$ of the input,
\begin{align}
  \label{eq:phi_teacher}
  y = \sum_{m=1}^M  \tilde v^m\,  \tilde g \left( \nu^m \right) , \quad
  \nu^m \equiv \ussD \sum_{r=1}^D \tilde w^m_r c_r.
\end{align}
The intuition here comes from image classification, where the label of an image
does not depend on every pixel~$x$, but the higher-level features of the
image, which should be better captured by its lower-dimensional latent
representation, like in conditional generative
models~\cite{mirza2014conditional, brock2018large}. 
We call this the \emph{deep hidden manifold model}.

\textbf{The two models of learning that we analyse} \hspace*{.5em} The advantage of the vanilla
teacher-student setup is that it lends itself well to analytical studies, at the
detriment of having unrealistic inputs. The deep hidden manifold allows us to
study realistic inputs, but can we still analyse it?  We provide two distinct
positive answers to this question for two common parametric models
$\hat{y}=\phi_{\theta}(x)$ trained on a dataset with i.i.d.~samples
$\mathcal{D}_{\samples}=\{\left(x^{\sindex},y^{\sindex}\right)\}_{\sindex=1}^{T}$ generated by the deep
hidden manifold $q$. First, we provide a sharp asymptotic analysis of
\textbf{full-batch learning with pre-learned features}~\cite{rahimi2008random}:
\begin{align}
  \label{eq:rf}
    \phi_{\theta}(x) = g\left(\lambda\right), \qquad \lambda = \frac{1}{\sqrt{\tilde{N}}}\hat{w}^{\top}\sigma\left(Fx\right),
\end{align}
where $F \in \reals^{\tilde N \times N}$ defines the feature map
$\Phi_{\mat{F}}=\nicefrac{\sigma(F~\cdot~)}{\sqrt{N}}\colon \mathbb{R}^{N}\to\mathbb{R}^{\tilde{N}}$, which is not necessarily random. We obtain the weights $\hat{w}\in\mathbb{R}^{\tilde{N}}$ by minimising the empirical risk in feature space:
\begin{align}
    \label{eq:erm}
    \hat{w}_{\samples}=\underset{w\in\mathbb{R}^{\tilde{N}}}{\argmin}\left[\sum\limits_{\sindex=1}^{\samples}\ell\left(y^{\sindex},w^{\top}\Phi_{F}(x^{\sindex})\right)+\frac{\lambda}{2}||w||^2_{2}\right]
\end{align}
\noindent with a convex loss function $\ell$ and a ridge penalty term
$\lambda>0$. In this model, the asymptotic limits is defined by taking
$\samples, N,\tilde{N}\to\infty$ with fixed ratios
$\tilde{N}/N, \samples/\tilde{N} \sim O(1)$.

Second, we provide an asymptotic analysis of \textbf{one-pass stochastic
  gradient descent in a two-layer neural network} with $K \sim O(1)$ hidden
units:
\begin{align}
\label{eq:2layer}
  \phi_{\theta}(x )  = \sum_{k=1}^K  v^k\,  g \left( \lambda^k\right) , \quad
  \lambda^k \equiv \ussN \sum_{i=1}^N w^k_i x_i,
\end{align}
where we take $N\to\infty$.  In this case, the network is trained end-to-end
with stochastic gradient descent on the quadratic loss using a previously unseem
sample at each step $\sindex$ of training:
\begin{align}
  \label{eq:sgd}
  \dd w^k_i \equiv \left(w^k_i\right)_{\sindex+1}- \left(w^k_i\right)_{\sindex}
  =-\frac{\eta}{\sqrt{N}} v^k \Delta g'(\lambda^k) x_i, \qquad
  \dd v^k = - \frac{\eta}{N} g(\lambda^k) \Delta.
\end{align}
where
$\Delta = \sum_{j = 1}^K v^j g(\lambda^j) -\sum_{m=1}^M \tilde v^m
\tilde{g}(\nu^m)$. Note the different scaling of the
  learning rate $\eta$, which guarantees the existence of a well-defined limit
  of the SGD dynamics as $N\to\infty$.

\paragraph{Test error and Gaussian equivalence property} In both cases, the
learner is thus given a dataset $\mathcal{D}_T = \{(x^{\sindex}, y^{\sindex})\}_{\sindex=1}^T$ consisting
of $T$ i.i.d.\ samples from $q(x,y)$. The classifier $\phi_\theta$ with
parameters $\theta = (W,v)$ either acts directly on the inputs $x$ or on a
feature map $\Phi \colon \reals^N \to \reals^{\tilde{N}}$. The learning
algorithm produces $\theta_T$ based on the training data. The model $\phi$ is
evaluated using the prediction MSE, which for each $\theta$ is
\begin{equation}
  \label{eq:eg-teacher-student}
  \pmse(q, \theta) \equiv
  \frac{1}{2} \int_{\reals^N \times \reals}   (\phi_\theta(\Phi(x)) - y)^2 \, \dd q(x,y) 
\end{equation}
The key observation in our analysis is that for both models~\eqref{eq:rf}
and~\eqref{eq:2layer} and the teacher~(\ref{eq:phi_teacher}), the respective
inputs only enter via the ``pre-activations'' $\lambda = (\lambda^k)$ and
$\nu = (\nu^m)$. We can therefore replace the high-dimensional average over
$q(x, y)$ by a low-dimensional average over the joint distribution
$P_\theta(\lambda, \nu)$ of $(\lambda, \nu)$, which is a function of
$\theta = (W,v)$:
\begin{align}
  \pmse(q, \theta) \equiv
  \frac{1}{2} \int_{\reals^K \times \reals^M}   \left(\sum_{k=1}^K  v^k\,  g (\lambda^k)  -\sum_m^M\tilde v^m \tilde g(\nu^m)\right)^2\, \dd P_\theta(\lambda, \nu) 
\end{align}
The complexity of the high-dimensional distribution $q$ is thus encapsulated by
the low-dimensional distribution $P_\theta(\lambda, \nu)$. 
If the student weights $W$ are drawn element-wise i.i.d.~from some distribution
irrespective of the training data and the (transformed) inputs of the student
$\tilde x$ are weakly correlated on average,
then~$(\lambda, \nu)$ are jointly Gaussian with high probability over $W$
if. Equivalently, we can require some spectral condition on the covariance matrix of $\tilde x$.

To be precise, consider a sequence of models and parameters $\left(q, \theta\right)$, where we let the dimension the latent space $D$, the dimension of the data $N$, and the dimension of the features $\tilde{N}$ scale to infinity at the same rate, while keeping the dimensions of $(\lambda, \nu)$ fixed. 
The \textbf{Gaussian equivalence property} (GEP) is said to hold if $P_{\theta}(\lambda, \nu)$
is asymptotically Gaussian, i.e.,
$d(P_{\theta}, P^*_{\theta}) = o_N(1)$ where~$P^*_{\theta}$ is
the Gaussian probability distribution with the same first and second moments and 
$d(\cdot, \cdot)$ is a metric that metrizes convergence in distribution and in second moments.

The Gaussian Equivalence property simplifies the analysis significantly, since
it allows for the $\pmse$ to be evaluated asymptotically in terms of the finite
dimensional Gaussian integral
\begin{align}
 \label{eq:eg-order-parameters}
  \pmse\left(q, \theta\right)  \to 
  \frac{1}{2} \int_{\reals^K \times \reals^M}   \left(\sum_{k=1}^K  v^k\,  g (\lambda^k)  -\sum_m^M\tilde v^m \tilde g(\nu^m)\right)^2\, \dd P^*_\theta(\lambda, \nu).
\end{align}
The $\pmse$ is thus a function of only the second moments of $(\lambda, \nu)$:
\begin{align}
  \label{eq:order-parameters}
  Q^{k\ell} \equiv \EE \lambda^k \lambda^\ell, \qquad 
  R^{km} \equiv \EE \lambda^k \nu^m, \qquad
  T^{mn} \equiv \EE \nu^m \nu^n,
\end{align}
and of the second-layer weights $v^k$ and $\tilde v^m$ in the case of two-layer
neural networks. This reduction of the high-dimensional
average~\eqref{eq:eg-teacher-student} to an expression in terms of an
$\order{1}$ number of ``order parameters'' is central to the vast literature
analysing the vanilla teacher-student setup~\cite{Gardner1989, Seung1992,
  Watkin1993, Biehl1995, Saad1995a, Engel2001}.

Surprisingly, here we find that this reduction also holds if the weights of the
student are obtained from the training data using the algorithms~\eqref{eq:sgd}
and~\eqref{eq:erm}. Hence, despite the correlations of the weights to the
correlated inputs, a characterisation of the $\pmse$ for models like
Eq.~\eqref{eq:rf} and~\eqref{eq:2layer} in terms of scalar order parameters
remain true for many generative data models, including common trained deep
generative networks, \emph{during} learning. This observation can be formalised
in the following conjecture, which is the central claim of our paper:
\begin{conjecture}[Deep Gaussian Equivalence Conjecture]
  \label{thm:deep-get}
  Suppose that 1) the teacher weights $\tilde{W}$ are generated i.i.d.\ and 2)
  $\cov(\Phi(x))$ satisfies some weak correlation property.  Let
  $\hat{\theta}_{\samples}$ be obtained from either online SGD~\eqref{eq:sgd} or
  empirical risk minimisation~\eqref{eq:erm}. Then, the GEP holds in
  the sense that for some probability distance $d(\cdot, \cdot)$, we have
  \begin{align}
    d\left(P_{\theta_{\samples}}, P^*_{\theta_{\samples}}\right) \to 0 \quad  \text{in probability} 
  \end{align}
   as $N,D,\samples \to \infty$ with $N, \samples = \Theta(D)$ and $M,K = O(1)$. Here, the
  probability is taken with respect to the randomness $q$ (i.e, the teacher weights
  and any other random components in the generator), the feature map $\Phi$, which may or may not be random, and the training data $\mathcal{D}_T$.
\end{conjecture}

We believe it is an exciting research direction to establish the limits of
Conjecture~\ref{thm:deep-get}. In this manuscript we give the first steps in
this direction by presenting three strands of rigorous (Sec.~\ref{sec:get}),
analytical (Sec.~\ref{sec:applications}) and numerical
(Sec.~\ref{sec:experiments}) evidence that the conjectured ``deep GEC'' holds
true for different tasks on shallow networks and for a wide range of deep,
pre-trained generative models. In particular, we provide:
\begin{enumerate}[label=(\roman*)]
\setlength\itemsep{-0.5em}
\item A rigorous proof of Conjecture~\ref{thm:deep-get} for a
  single-layer generator of the form $\gen(c)=\sigma(A c)$ where $A$ is a matrix
  with pre-trained weights, and $\sigma$ is a point-wise non-linearity. Our
  \textbf{Gaussian equivalence theorem} (GET, Thm.~\ref{thm:get}) gives
  sufficient conditions on the weights $A$ under which a given low-dimensional
  projection of the input~$x$, such as $\lambda, \nu$, is approximately
  Gaussian. We thus put the Gaussian equivalence property used
  in~\cite{goldt2019modelling, gerace2020generalisation} on a rigorous basis.
\item\label{contrib:ode} An exact analytical description of  the evolution of the test error of a two-layer neural network trained using
  one-pass (or online) SGD (Sec.~\ref{sec:odes}), whose predictions exactly
  match simulations with convolutional GANs and normalising flows pre-trained on
  CIFAR10 (Sec.~\ref{sec:experiments}).
\item\label{contrib:replica} A set of scalar self-consistent equations
  describing the test error for full-batch learning of $\samples$ i.i.d.~samples using
  regression with $\tilde N$ features in the regime where
  $N, \tilde N, D, \samples\to\infty$ with $\samples/\tilde{N}, N/\tilde N = O(1)$
  (Sec.~\ref{sec:replicas}). As before, we confirm the accuracy of this theoretical prediction with experiments of convolutional GANs pre-trained on
  CIFAR100
  (Sec.~\ref{sec:experiments}).
\end{enumerate}

\paragraph{Further related work} Several works have recognised the importance of
data structure in machine learning, and in particular the need to go beyond the
simple component-wise i.i.d.\ modelling for neural
networks~\cite{bruna2013invariant, PatelNIPS2016, mossel2016deep,
  gabrie2018entropy}, recurrent neural networks~\cite{mezard2017mean} and
inference problems such as matrix factorisation~\cite{hand2018phase,
  aubin2019spiked}. \citet{ansuini2019intrinsic} demonstrated that a network's
ability to transform data into low-dimensional manifolds was predictive of its
classification accuracy.

While we will focus on the prediction error, a few recent papers studied a
network's ability to \emph{store} inputs with lower-dimensional structure and
random labels:~\citet{Chung2018} studied the linear separability of general,
finite-dimensional manifolds and their interesting consequences for the training
of deep neural networks~\cite{chung2018learning, cohen2020separability}, while
Cover's classic argument~\cite{Cover1965} to count the number of learnable
dichotomies was recently extended to cover the case where inputs are grouped in
tuples of $k$ inputs with the same label~\cite{Rotondo2019a,
  borra2019generalization}. \citet{koehler2019comparative} studied the
expressive power of ReLU networks compared to polynomial kernels under a data
model where the teacher is a linear function of $c$, and the inputs are a noisy
linear projection of the latent variables. Recently
\citet{yoshida2019datadependence} analysed the dynamics of online learning for
data having an arbitrary covariance matrix, finding an infinite hierarchy of
ODEs (cf.~Sec.~\ref{sec:odes}).

Gaussian equivalent models are currently attracting a lot of interest. During the revision of this work, we became aware of a recent
alternative proof of the GET by \citet{hu2020universality} for a slightly
different setup. A parallel line research analysed random features regression
using random matrix theory (RMT)~\cite{louart2018random, fan2019spectral}. The
equivalent mapping to a Gaussian model with appropriately chosen covariance was
explicitly stated and used in~\cite{mei2019generalization,
  montanari2019generalization} and extended to a broader setting encompassing
data coming from a GAN in~\cite{seddik2019kernel, seddik2020random}. We will
discuss these works in relation to our results  in
Sec.~\ref{sec:get-discussion}.

\paragraph{Reproducibility} We provide code to solve the equations of
Sec.~\ref{sec:applications} and the experiments of Sec.~\ref{sec:experiments}
online at~\url{https://github.com/sgoldt/gaussian-equiv-2layer}.

\section{The Gaussian Equivalence Theorem}
\label{sec:get}

We start with the study of a simple generator where inputs are generated
according to
\begin{equation}
  \label{eq:proof-generators}
  x_n=\gen_n(c) = \sigma(a_n^\top c)=\sigma\left(\sum_{r=1}^D a_{rn}c_r\right)
\end{equation}
where $N\to\infty$, $D\to\infty$ at fixed $\delta=D/N$, and
$\sigma: \reals \to \reals$ is a non-linear function and
$A = [a_1, \dots, a_n]^\top$ is the weight matrix of the generator. This is
precisely the setting of the hidden manifold model of~\cite{goldt2019modelling,
  gerace2020generalisation}, and generators of the
form~\eqref{eq:proof-generators} cover a number of important cases beyond the
hidden manifold model: \emph{(i)} random feature models~\cite{rahimi2008random,
  rahimi2009weighted}, which regard the latent variable $c$ as the true
underlying data and $x$ as features constructed from $c$ that are used as inputs
for the prediction algorithm (cf.~Sec.~\ref{sec:replicas}); \emph{(ii)}~Gaussian
feature models, where the inputs $x$ are jointly Gaussian with the latent
variables $c$; and \emph{(iii)}~the classic teacher-student
setup~\cite{Gardner1989, Seung1992, Engel2001}, where the features $x$ are equal
to the latent variables~$c$.

The inputs generated by such a generator are not
Gaussian. However, our first main result, the \textbf{Gaussian Equivalence
  Theorem}, guarantees that the local fields~$(\lambda, \nu)$ are still jointly
Gaussian, and hence a description in terms of order parameters like
Eq.~(\ref{eq:eg-order-parameters}) possible, even if inputs are drawn from this
generator. More precisely, the theorem gives verifiable conditions on $\sigma$
and the weight matrices of the student, teacher and generator networks, under
which a low-dimensional projection of the inputs, such as $\lambda$ and $\nu$,
is approximately Gaussian.

\subsection{Statement of the theorem}
Given probability measures $P$ and $Q$ on $\reals$,
define
\begin{align}
  d(P,Q) \equiv \sup_{f \in \mathcal F} \left| \mathbb{E}_P[f] - \mathbb{E}_Q[f]\right|, 
\end{align}
where $\mathcal F = \{ f\, : \, \|f''\|_\infty, \|f'''\|_\infty \le 1\}$ is the
set of thrice-differentiable functions with bounded second and third
derivative and $\|f\|_\infty$ is the uniform norm of $f$. Given probability measures $P$ and $Q$ on $\reals^d$ the
maximum-sliced (MS) distance is defined by
\begin{align}
  d_{\mathrm{MS}}(P,Q) \equiv \sup_{\alpha \, : \, \|\alpha\| \le 1} d(\alpha^\top P, \alpha^\top Q) 
\end{align}
where $\alpha^\top P$ denotes the one-dimensional distribution corresponding to
the projection of $P$ into the direction of $\alpha$. It can be verified that
the MS distance is a metric~\cite{kolouri2019generalized} and that convergence
with respect to $d_{\mathrm{MS}}$ implies convergence in distribution as well as
convergence of second moments. Our result requires the following regularity assumptions:
\begin{itemize}[itemsep=1pt,topsep=0pt]
\item[A1)] Row normalisation $\|a_n\| =
  1$; 
\item[A2)] Smoothness: the non-linearity $\sigma$ is thrice differential with 
  $\ex{|\sigma(u)|^4}$, $\ex{ |\sigma'(u)|^2}$, and $\ex{ |\sigma''(u)|^2}$
  all $\order{1}$ for $u\sim\normal(0, 1)$;
\item[A3)] Bounded student weights: $w_n^k = \order{1}$. 
\end{itemize}
Note that the smoothness assumption on the non-linearity $\sigma$ can be relaxed
to the assumption that $\sigma$ is Lipschitz continuous, with the only
consequence being a loss in the rate of convergence. The basic idea is that any
Lipschitz function can be approximated by a function that satisfies the
smoothness assumptions, see e.g.~\cite[Proposition 11.58]{odonnell2014analysis}.
The dependence on $\sigma$ is quantified in terms the first, second, and third Hermite
coefficients, which  are defined by
\begin{align}
    \hat{\sigma}(1) \equiv  \ex{\sigma(u) u }, \quad     \hat{\sigma}(2) \equiv \frac{1}{\sqrt{2}} \ex{\sigma(u) (u^2 -1)}, \quad     \hat{\sigma}(3) \equiv \frac{1}{\sqrt{6}} \ex{ \sigma(u)(u^3 - 3 u)},
\end{align}
where the expectation is taken with respect to a standard Gaussian random variable $u$. Furthermore, let $\rho = A A^\top$ and $\tilde{\rho} = \rho - I_N$ and define
the $N \times N$
matrices:
\begin{align}
  M_{1}  =   \frac{1}{\sqrt{N}}\left( \hat{\sigma}^2(1)    \tilde{\rho}^2  +
  \hat{\sigma}^2(2)  \tilde{\rho}^2  \circ \rho  \right),\qquad
  M_2  = \hat{\sigma}^2(2)  \left( \tilde{\rho} \circ \tilde{\rho} \right)^2 + \hat{\sigma}^2(3)  \left( \tilde{\rho} \circ \tilde{\rho} \right)^2 \circ  \rho,
\end{align}
where $\circ$ denotes the Hadamard entrywise product. Each of these matrices is
positive semi-definite, by the Schur product
theorem~\cite[Sec. 7.5]{horn2012matrix}, and thus has a unique positive
semi-definite square root. We then have:
\begin{theorem}[Gaussian Equivalence Theorem]
  \label{thm:get}
  Let $P$ be the distribution of the pair $(\lambda, \nu)$ and let $\hat{P}$ be the Gaussian distribution with the same first and second moments.   Under  Assumptions A1-A3,
  \begin{align}
  \label{eq:get}
    d_{\mathrm{MS}}(P,\hat{P}) = O\left( 
     \left \|\tfrac{1}{\sqrt{N}} W M^{1/2}_1 \right \|^2 +
      \left\| \tfrac{1}{\sqrt{N}} W M_2^{1/2} \right\| + \tfrac{1}{\sqrt{N}} \left \|\tfrac{1}{\sqrt{D}} \tilde{W} A^\top
      \right \|^2 + \frac{1+ \sum_{i\ne j} (a_i^\top a_j)^4}{\sqrt{N}} \right).
    \end{align}
  \end{theorem}
  We provide the proof of Theorem~\ref{thm:get} in
  Sec.~\ref{sec:get-proof}. 

\subsection{Discussion}
\label{sec:get-discussion}

Theorem~\ref{thm:get} can be viewed as a multivariate central limit theorem
(CLT) for weakly dependent random
variables. 
The terms involving the matrices $M_1$ and $M_2$ quantify the impact of the
dependencies in~$x$. Note for example that if the columns of $A$ are
uncorrelated, then both of these terms are zero and Theorem~\ref{thm:get}
recovers a variation of the classical Berry--Esseen
Theorem~\cite[Chapter~11.5]{odonnell2014analysis}. The significance of
Theorem~\ref{thm:get} is that it provides a simple and verifiable sufficient
condition for the joint Gaussianity of ($\lambda, \nu$) for pre-trained, and
hence correlated generator weights. The basic idea is that in order for Gaussianity to hold, the weight matrices
should avoid any directions in the matrices $M_1$ and $M_2$ associated with
eigenvalues that are not converging to zero.
 
To appreciate how the spectral properties of $M_1$ and $M_2$ depend on $A$ and
$\sigma$, it is useful to consider some examples. We give two quick examples
below; we discuss these examples in detail in Sec.~\ref{sec:conditions-get},
where we analyse how the leading eigenvalues and eigenvectors of $M_1$ and $M_2$
depend on $A$ using analytical and numerical arguments.

\begin{example}[IID $A$]
  If the entries of $A$ are i.i.d.\ sub-Gaussian, then $\|M_1\| = \order{1/\sqrt{N}}$
  with high probability. If $\hat{\sigma}^2(2)$ is nonzero, then $M_2$ has one
  eigenvalue that is $\order{1}$ associated with the all-ones vector and the rest are
  $\order{1/N}$. If $\hat{\sigma}^2(2)=0$, which occurs whenever $\sigma$ is an odd
  function, then $\|M_2\| = \order{1/N}$. Thus, if $\hat{\sigma}(2)= 0$ or
  $\|\tfrac{1}{\sqrt{N}}W \one\| = \order{1/N}$ it follows that
  $d_{\mathrm{MS}}(P,\hat{P}) =\order{1/\sqrt{N}}$ with high-probability over $A$.
\end{example}
\begin{example}[Deterministic $A$] Next consider the case $(D \ge N)$ where
\begin{align*}
    AA^\top = I_N + \frac{c}{\sqrt{N}} (\one_N - I_N)
\end{align*}
for some fixed constant
$c$. 
Suppose that $\sigma(k), k=1, 2,3$ are nonzero. Direct calculation reveals that
$M_1$ has one eigenvalue $\order{\sqrt{N}}$ with the rest $\order{1/\sqrt{N}}$
and $M_2$ has one eigenvalue $\order{1}$ with the rest $\order{1/N}$. In both
cases, the leading eigenvector is proportional to the all ones vector. Thus if
$\|\tfrac{1}{\sqrt{N}}W \one\| = \order{1/N}$ then
$d_{\mathrm{MS}}(P,\hat{P}) = \order{1/\sqrt{N}}$.
\end{example}

The idea that most low-dimensional projections of a high-dimensional
distribution are approximately random has a rich history~\cite{sudakov:1978,
  diaconis:1984, hall:1993, bobkov:2003, meckes:2010aa, reeves:2017c}. In this
line of work, ``most'' is quantified in terms of high-probability guarantees
with respect to a random weight matrix $W$ that is independent of $x$. For
example, if the entries of $W$ are i.i.d.\ standard Gaussian, then the necessary
and sufficient conditions for convergence to a Gaussian are that
1)~$\nicefrac{1}{n} \|x\|^2$ concentrates about is mean 2)~and
$\nicefrac{1}{n}\|\cov(x)\|_F^2 \to 0$ (assuming zero mean). In the setting of
this paper, it can be verified that these properties are implied by assumptions
A1 and A2. The added benefit of Theorem~\ref{thm:get} is that ``most'' is now
quantified deterministically in terms of the number of the eigenvalues of $M_1$
and $M_2$.

The last term in \eqref{eq:get} imposes a constraint on the \emph{average} pairwise correlation between the columns of $A$. Specifically, this term converges to zero provide that $\sum_{i\ne j} (a_i^\top a_j)^4 = o(\sqrt{N})$. Importantly, this constraint still allows for  allows for the possibility that a subset of the entries of~$A$ have correlation of order one. By contrast, previous work in this setting
requires either randomly generated features or a much stronger incoherence
constraint on the \emph{maximum} correlation between any two entries. The
generality provided by A1 is crucial to our target applications since it allows
for ``sufficiently small'' subsets to have arbitrary dependence structure. This
is also a key difference to the proof of a similar result
by~\citet{hu2020universality} that appeared during the revision of this
manuscript. 


Our analysis also highlights the dependence of the first few terms in the
Hermite expansion of~$\sigma$. While \citet{hu2020universality} assume that
$\sigma$ is odd, which leads to $\hat{\sigma}(2)=0$, our analysis highlights the
crucial role of $\hat{\sigma}(2)$: if it is non-zero, as is the case for ReLU,
then correlation in $\lambda$ is described not by the linear dependence with
$\nu$, but by a quadratic dependence, leading to more stringent conditions for
the validity of the CLT.

In a different direction, Gaussian behaviour associated with random
choices of the parameter $A$ have also been studied in the context of infinitely
wide networks~\cite{neal1995bayesian, lee2018deep, matthews2018gaussian}.
Specifically, if the entries of $A$ are i.i.d.\ Gaussian random variables it
follows that $\lambda \mid \nu$ can be viewed as Gaussian processes indexed by
$\nu$. Combined with the Gaussianity of $\nu$, this establishes the GET under
general conditions on the generator. However, this analysis relies crucially on
the assumption that $A$ is generated independently of everything else. This
assumption precludes the application to pre-trained generators.

A recent line research has derived Gaussian equivalence theorems for generators
with random weights using random matrix theory
(RMT)~\cite{hachem2007deterministic, cheng2013spectrum, pennington2017nonlinear,
  louart2018random, fan2019spectral}. The equivalent mapping to a Gaussian model
with appropriately chosen covariance was explicitly stated and used
in~\cite{mei2019generalization, montanari2019generalization} and extended to a
broader setting encompassing data coming from a GAN in~\cite{seddik2019kernel,
  seddik2020random}. Similar to the analysis in this paper, the high-level idea
is that certain integrals with respect to the data distribution $q(x,y)$ can be
replaced by integrals over an appropriately defined Gaussian approximation. The
main difference is the class of functions considered. Specifically,
Theorem~\ref{thm:get} provides guarantees for any sufficiently smooth function
applied to a given low-dimensional projections of the features $(x,c)$. This
form of approximation is needed to justify the integro-differential equations
derived in Sec.~\ref{sec:odes}.  By contrast, the RMT approach provides
guarantees for a restricted set of functions applied to high-dimensional
matrices derived from samples of $(x,c)$. For example, these results provide
equivalence of the empirical spectral measures of these random matrices as well
as the test error associated with specific learning algorithms. The
results in this paper thus neither imply previous work, nor are they, to the
best of our knowledge, implied by it.

\section{Analysis of neural networks learning
  on data from deep generators}
\label{sec:applications}

We now turn to two applications of the deep GEC that allow us to analyse
learning in paradigmatic model systems in detail, and at the same time help us
gather experimental evidence for the deep GEC. We will first derive a set of
equations that describe the evolution of the test error of a two-layer neural
network trained using one-pass (or online) SGD on the deep hidden manifold model
(Sec.~\ref{sec:odes}). We also use the deep GEC to analyse full-batch learning
with pre-learned features in Sec.~\ref{sec:replicas}. Our experiments in
Sec.~\ref{sec:experiments} will show perfect agreement between the theory
derived using the deep GEC and simulations with deep, pre-trained generators,
giving further credibility to our conjecture.

\subsection{Generalisation dynamics of two-layer networks using online SGD}
\label{sec:odes}

We first study a two-layer neural network~(\ref{eq:2layer})
trained end-to-end using online stochastic gradient descent~(\ref{eq:sgd}).
Since the deep GEC guarantees that the local fields $(\lambda, \nu)$ are jointly
Gaussian, permitting to express the $\pmse$ of a given student and teacher in
terms of only the ``order parameters'' $Q, R, T, v$ and $\tilde v$~\eqref{eq:order-parameters}. In order to compute the
$\pmse$ at all times during training, it is thus sufficient to track the
evolution of the order parameters during training, which is the goal of this
section. 

We will make the crucial assumption that at each step of the algorithm, we use a
previously unseen sample $(x, y)$ to compute the updates in
Eq.~\eqref{eq:sgd}. This limit of infinite training data is variously known as
online learning or one-shot/single-pass SGD. Using this assumption, the dynamics
of two-layer networks in the classic teacher-student setup with i.i.d.~Gaussian
inputs have been analysed in seminal works by~\citet{Biehl1995}
and~\citet{Saad1995a}; see also~\cite{Saad1995b, saad2009line} for further
results and~\cite{goldt2019dynamics} for a recent proof of these
equations. Here, we generalise this type of analysis to two-layer networks
trained on inputs coming from the deep hidden manifold model. Note that this
online-learning framework has also been used by a number of recent works
studying the dynamics of networks with finite $N$ and large hidden layer
$K\to\infty$~\cite{Mei2018, Rotskoff2018, Chizat2018, Sirignano2018}.


We derived a closed set of integro-differential equations that describe the
evolution of all order parameters using Conjecture~\ref{thm:deep-get}. We
provide a self-contained discussion of these equations here, and relegate the
detailed derivation to Sec.~\ref{sec:ode-derivation}. Remarkably, the generator
$\gen(c)$ only enters the equations via the input-input and the input-latent
covariance,
\begin{equation}
  \label{eq:covariances}
  \Omega_{ij} = \EE x_i x_j, \qquad \Phi_{ir} = \EE x_i c_r.
\end{equation}
The order parameter~$Q$~\eqref{eq:order-parameters} can be written as
$Q^{k\ell}\equiv \EE \lambda^k \lambda^\ell \sim \sum w_i^k \Omega_{ij}
w_j^\ell$. A key step in the analysis is to diagonalise this sum by projecting
the student weights into the eigenspace of $\Omega$
(cf.~Sec.~\ref{sec:ode-derivation}). We can then consider the integral
representation
\begin{equation}
  \label{eq:Q_int}
  Q^{k\ell}= \int \dd \sindex_\Omega(\rho) \; \rho \; q^{k\ell}(\rho).
\end{equation}
where $\sindex_{\Omega}(\rho)$ is the spectral density of $\Omega$ (which is known
and fixed at all times since it is a property of the generator $\gen$), and
$q^{kl}(\rho)$ is a density whose time evolution can be characterised in the
thermodynamic limit. In the canonical teacher-student model with i.i.d.\ inputs
$x$, introducing such a density is not necessary since the input-input
covariance is trivial, $\Omega_{ij}=\delta_{ij}$. As we go to the thermodynamic
limit $N\to\infty$, we can identify a continuous time-like parameter
$t\equiv \sindex / N$ and find that the density $q^{k\ell}(\rho)$ evolves according
to
\begin{align}
  \label{eq:eom-q}
  \begin{split}
    \frac{\partial q^{k\ell}(\rho)}{\partial t} & = -\eta \left( \rho
      \sum_{j\neq k}^K \left[ v^k v^j q^{k\ell}(\rho) h_{(1)}^{kj}(Q) + v^k v^j
        q^{j
          \ell}(\rho) h_{(2)}^{kj}(Q) \right] + \rho v^k v^k q^{k \ell}(\rho) h_{(3)}^k(Q) \right. \\
    & \hspace*{4em} - v^k \sum_{n}^M \left[ \rho\tilde v^n q^{k\ell}(\rho)
      h_{(4)}^{kn}(Q, R, T) +
      \ussdelta \tilde v^n r^{\ell n}(\rho)h_{(5)}^{kn}(Q, R, T) \right] \\
    & \hspace*{4em} + \text{all of the above with } \ell\to k, k\to\ell \Bigg) +
    \eta^2 \gamma v^k v^\ell h_{(6)}^{k\ell}(Q, R, T, v, \tilde v).
 \end{split}
\end{align}
where $\gamma \equiv \sum_\tau \rho_\tau / N$ and $\delta\equiv D / N$.  The
functions $h_{(1)}^{kj}$ etc.\ are scalar, non-linear functions that only
involve averages over the pre-activations $\lambda$ and $\nu$ such as
$\EE g(\nu^m) g'(\lambda^k)\lambda^j$, see Eq.~\eqref{eq:h}. After invoking the
deep GEC, these averages can be expressed in terms of the order
parameters~\eqref{eq:order-parameters}, and hence the equation closes. Likewise,
we also consider the projection of
$\omega_i^m\equiv \sum_r \Phi_{ir} \tilde w^m_r$ into the eigenspace of $\Omega$
and consider the integral representation
\begin{equation}
  \label{eq:R_int}
  R^{km} = \ussdelta \int \dd \sindex_\Omega(\rho)\;  r^{km}(\rho).
\end{equation}
We find that $r^{km}(\rho)$ evolves as
\begin{multline}
  \label{eq:eom-r}
  \frac{\partial r^{km}(\rho)}{\partial t} = -\eta v^k \left( \rho
    \sum_{j\neq k}^K \left[ v^j r^{km}(\rho) h_{(1)}^{kj}(Q) + v^j \rho r^{j
        m}(\rho) h_{(2)}^{kj}(Q) \right] + v^k
    \rho r^{k m}(\rho) h_{(3)}^k(Q) \right. \\
  \left. - \sum_{n}^M \left[ \rho \tilde v^n r^{km}(\rho)  h_{(4)}^{kn}(Q, R, T) +
      \ussdelta \tilde v^n h_{(5)}^{kn}(Q, R, T) \right] \right).
\end{multline}
Finally, the equation for $v$ can be obtained directly from the SGD
update~\eqref{eq:sgd} and reads
\begin{equation}
  \label{eq:eom-v}
  \frac{\dd v^k}{\dd t} = \eta \left[ \sum_n^M \tilde v_n h_{(7)}^{kn}(Q, R) - \sum_j^K
    v^j h_{(7)}^{kj}(Q) \right].
\end{equation}



\paragraph{Discussion} The importance of the spectral properties of the data was
recognised for learning in linear neural networks~\citet{baldi1989neural,
  le1991eigenvalues, krogh1992generalization, saxe2014exact}.
\citet{yoshida2019datadependence} extended the ODE analysis for non-linear
networks to inputs with a covariance matrix having $O(1)$ non-degenerate
eigenvalues, while implicitly assuming that inputs have a Gaussian
distribution. \citet{goldt2019modelling} analysed online learning in the hidden
manifold for a single-layer generator of the form $x=\sigma(A c)$; their result
also involved more order parameters than our analysis.  Our approach handles a
more general data structure, in the sense that inputs can have arbitrary
covariance matrices $\Omega$ and $\Phi$. More importantly, the GET
(Thm.~\ref{thm:get}) rigorously guarantees that we can analyse the SGD dynamics
even for inputs that are drawn from pre-trained generative models such as
Eq.~\eqref{eq:proof-generators} and hence do not follow a Gaussian
distribution. Our experiments in the next section show how this analysis also
holds for deep, pre-trained generative models such as normalising flows (see
Sec.~\ref{sec:experiments} for the discussion and Fig.~\ref{fig:pretrained} for
an example of the images generated by these models).

\paragraph{Solving the equations of motion} The equations of
motion~(\ref{eq:Q_int}-\ref{eq:eom-v}) are valid for any choice of generator
network and for any teacher and student activation functions $g(x)$ and
$\tilde g(x)$ as long as the deep GEC holds. To solve the equations for a
particular setup, one needs to estimate the covariance matrices $\Omega$ and
$\Phi$, and to evaluate the functions $h_{(1)}^{kj}$ etc.\ that are given in the
appendix. By choosing $g(x)=\tilde g(x)=\erf(x/\sqrt{2})$, all these functions
have exact analytical expressions~\cite{Saad1995a}. We provide robust Monte
Carlo estimators of the covariance matrices of any generative network in
pyTorch~\cite{paszke2019pytorch} and a numerical implementation of the equations
of motion at~\url{https://github.com/sgoldt/gaussian-equiv-2layer}.

\subsection{Full-batch analysis of learning a generalised linear model with pre-learned features}
\label{sec:replicas}
We now discuss a second task in which the deep GEC~\ref{thm:deep-get} can be
used to give a sharp analysis of the asymptotic performance: full-batch learning with
pre-learned or random features. In this task, a batch of $T$ i.i.d. samples
$\mathcal{D}_{\samples} = \{(x^{\sindex}, y^{\sindex})\}_{\sindex=1}^{\samples}$ from $q$ are projected
using a \emph{feature map}
$\tilde{x}=\tilde{N}^{-1/2}\sigma(Fx)\in\mathbb{R}^{\tilde{N}}$. The restrictions that we place on the projection matrix $F$ are exactly the same that we put on the weights of the one-layer generator $A$ in our proof of the GET, see Sec.~\ref{sec:get}.

The features $\tilde x$ are is then
fitted with the \emph{generalised linear model}
$ \hat{y} = \phi_\theta(x)=
g\left(\sum_{n=1}^{\tilde{N}}w_{n}\tilde{x}_{n}\right)$, where we can take
$g(x)={\rm{sign}}(x)$ for a classification problem or $g(x)=x$ for regression
for example. The weights $\hat{w}\in\mathbb{R}^{\tilde{N}}$ are learned by
minimising the empirical risk~(\ref{eq:erm}).
Note that a for a convex loss function $\ell$, the regularised risk is strongly
convex and admits one and only one solution. One interesting special case of
this model are random features, since for random $\mat{F}$, in the limit
$\tilde{N}\to\infty$, the expected scalar product in feature space converges to
a kernel~\cite{rahimi2008random}:
\begin{align}
    \frac{1}{\tilde{N}}\mathbb{E}_{\mat{F}}\left[\sigma\left(\mat{F}x_{1}\right)^{\top} \sigma\left(\mat{F}x_{2}\right)\right] \underset{\tilde{N}\to\infty}{\to} K(x_{1},x_{2}).
\end{align}
It is out of the scope of this work to describe this construction in full
generality, and we refer the curious reader to~\cite{rahimi2008random,
  rahimi2009weighted} for details on how the kernel depends on the choice of
$\Phi_{\mat{F}}$. The important point here is that studying kernel regression is
equivalent to studying linear regression on feature space at
$\tilde{N}\to\infty$. There has been a surge of interest in kernel methods
recently, as it was shown that deep neural networks are equivalent to random
features in the so-called lazy regime~\cite{Jacot2018, chizat2019lazy}.

Since the feature map
$\Phi_{\mat{F}}=\tilde{N}^{-1/2}\sigma\left(\mat{F}~\cdot~\right)$ is
pre-learned, for the purpose of the theoretical analysis it can be
incorporated as an additional layer to the generative model for data:
$\tilde{x} = \Phi_{\mat{F}}(x) = \left(\Phi_{\mat{F}}\circ
  \mathcal{G}\right)(c)$, where $\mathcal{G}$ can be any of the generative
models discussed previously. With this observation in mind, without loss of
generality we can restrict our attention to the study of generalised linear
models with data coming from a deep generative model (which includes the feature
map). Up to a rescaling, the generalised linear model is equivalent to $K=1$ in
model \eqref{eq:2layer}, and in this section we also restrict the analysis to
$M=1$ in eq.~\eqref{eq:phi_teacher}. Therefore, the target outputs are simply
generated from the latent vector $c\sim\mathcal{N}(0, \mat{I}_{D})$ as in
Eq.~(\ref{eq:phi_teacher}), 
which are then fitted by the network $\phi_\theta(\tilde x)$ by minimising the
regularised empirical risk \eqref{eq:erm}.

Let $\mathcal{D}_{S}=\{(x, y)_{\sindex=1}^\samples\}$ be a data set with $\samples$
i.i.d. samples from $q$. Define the sample complexity
$\alpha = \samples/\tilde N$ and the latent-to-input aspect ratio
$\delta = D/\tilde N$. As in the online analysis in Section~\ref{sec:odes}, the
deep GEC~\ref{thm:deep-get} can be used to write an asymptotic formula for the
performance of the estimator $\phi_\theta(\tilde x)$ in the limit where
$D, \samples, \tilde N\to\infty$ and the ratios $\alpha,\delta = O(1)$:
\begin{align}
    \epsilon_{g} = \mathbb{E}_{(x,y)\sim q} \pmse(y,\hat{y}(x)) \underset{N\to\infty}{\to} \frac{1}{2}\mathbb{E}_{(\nu,\lambda)\sim\mathcal{N}(0,\Sigma)}(\tilde g(\nu)-g(\lambda))^2
    \label{eq:test_error}
\end{align}
where $(\nu,\lambda)\sim\mathcal{N}(0,\Sigma)$ are jointly Gaussian variables
with covariance
$ \Sigma = \begin{pmatrix} \rho & m^{\star}\\ m^{\star} &
  q^{\star}\end{pmatrix},\notag $ and
\begin{align}
    \rho = \frac{1}{D}||\tilde{w}||^2_{2}, && m^{\star} = \frac{1}{\sqrt{ND}}\hat{w}^{\top}\Phi \tilde{w}, && q^{\star} = \frac{1}{N}\hat{w}^{\top}\Omega \hat{w}.
\end{align}
The covariances $\Phi, \Omega$ are the moments of the equivalent Gaussian distribution, and were defined explicitly in eq.~\eqref{eq:covariances}. In principle, $(m^{\star}, q^{\star})$ should be computed from the estimator $\hat{w}\in\mathbb{R}^{\tilde{N}}$. Surprisingly, we can also use the deep GEC to derive a set of self-consistent equations with solution giving directly $(m^{\star}, q^{\star})$: 
\begin{align}
	\begin{cases}
		\hat{V} = \alpha \mathbb{E}_{\xi\sim\mathcal{N}(0,1)}\left[\int_{\mathbb{R}}\dd y~\tilde{\mathcal{Z}}_{y} \left(\frac{1-\partial_{\omega}\eta}{V}\right)\right]\\
		\hat{q} = \alpha \mathbb{E}_{\xi\sim\mathcal{N}(0,1)}\left[\int_{\mathbb{R}}\dd y~\tilde{\mathcal{Z}}_{y} \left( \frac{\eta	-\omega}{V}\right)^2 \right]\\
		\hat{m} = \frac{\alpha}{\sqrt{\delta}} \mathbb{E}_{\xi\sim\mathcal{N}(0,1)}\left[\int_{\mathbb{R}}\dd y~\partial_{\omega}\tilde{\mathcal{Z}}_{y}\left(\frac{\eta-\omega}{V}\right) \right]
	\end{cases} && 
	\begin{cases}
		V =  \frac{1}{\tilde{N}}\tr\left(\lambda\mat{I}_{N}+\hat{V}\Omega\right)^{-1}\Omega\\
		q = \frac{1}{\tilde{N}}\tr\left[\left(\hat{q}\Omega+\hat{m}^{2}\Phi\Phi^{\top}\right)\Omega\left(\lambda\mat{I}_{N}+\hat{V}\Omega\right)^{-2}\right]\\
		m= \frac{\hat{m}}{\tilde{N}\sqrt{\delta}}\tr \Phi\Phi^{\top}\left(\lambda\mat{I}_{N}+\hat{V}\Omega\right)^{-1}
	\end{cases}\label{eq:saddlepoint}
\end{align}
\noindent with:
\begin{align}
\tilde{\mathcal{Z}}_{y}(y,\tilde{\omega},\tilde{V}) = \int_{\mathbb{R}}\frac{\dd x}{\sqrt{2\pi \tilde{V}}}e^{-\frac{(x-\tilde{w})^2}{2\tilde{V}}}\delta\left(y-\tilde{g}(x)\right), &&\eta(y,\omega,V) = \underset{x\in\mathbb{R}}{\argmin}\left[\frac{(x-\omega)^2}{2V}+\ell(y,x)\right]\notag
\end{align}
\noindent and $\omega = \sqrt{q}\xi$, $V=\rho-q$, $\tilde{w} = m/\sqrt{q}\xi$, $\tilde{V}=\rho-m^2/q$. Although this formula appears cumbersome, it only depends on scalar parameters and on the spectral distribution of $\Phi\Phi^{\top}$ and $\Omega$. It therefore reduces the high-dimensional computation of $\epsilon_{g}$ to solving a low-dimensional 
system of equation which for a given generator $\mathcal{G}$, loss function $\ell$ and non-linearities $(g,\tilde{g})$ can be easily done by iteration. For random generators, the spectral distributions of $\Phi\Phi^{\top}$ and $\Omega$ can be computed analytically. But this formula also holds for the case of real, trained deep generative models, in which case the spectrum of $\Phi\Phi^{\top}$ and $\Omega$ are computed numerically via robust Monte-Carlo simulations exactly as in Section~\ref{sec:odes}. Note that this result generalises the formula from \citet{gerace2020generalisation} for
a single-layer generator which was rigorously proved recently by
\citet{dhifallah2020precise}. Although it is an open problem to prove it
rigorously in the current setting, we verified that it perfectly matches
simulations for different loss functions and for all generative architectures
discussed here. See Fig.~\ref{fig:replicas} in Section~\ref{sec:experiments} for one example. This provides another strong evidence for conjecture~\ref{thm:deep-get} - as it shows that a formula only depending on second order statistics is able to completely capture the asymptotic performance of random features trained on data from a trained generative model. 
\begin{figure}[t!]
  \centering
  \includegraphics[width=.48\textwidth]{{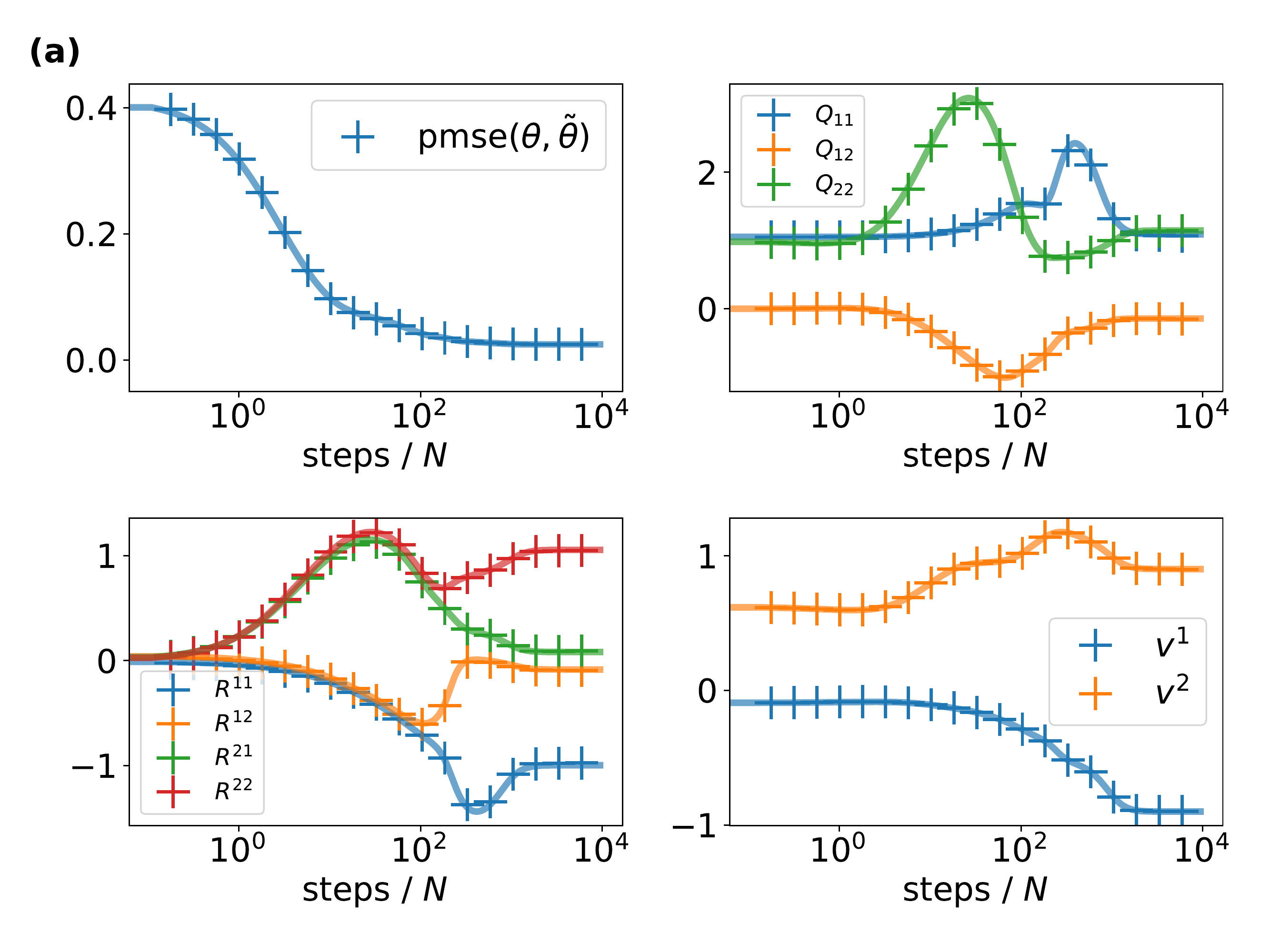}}
  \includegraphics[width=.48\textwidth]{{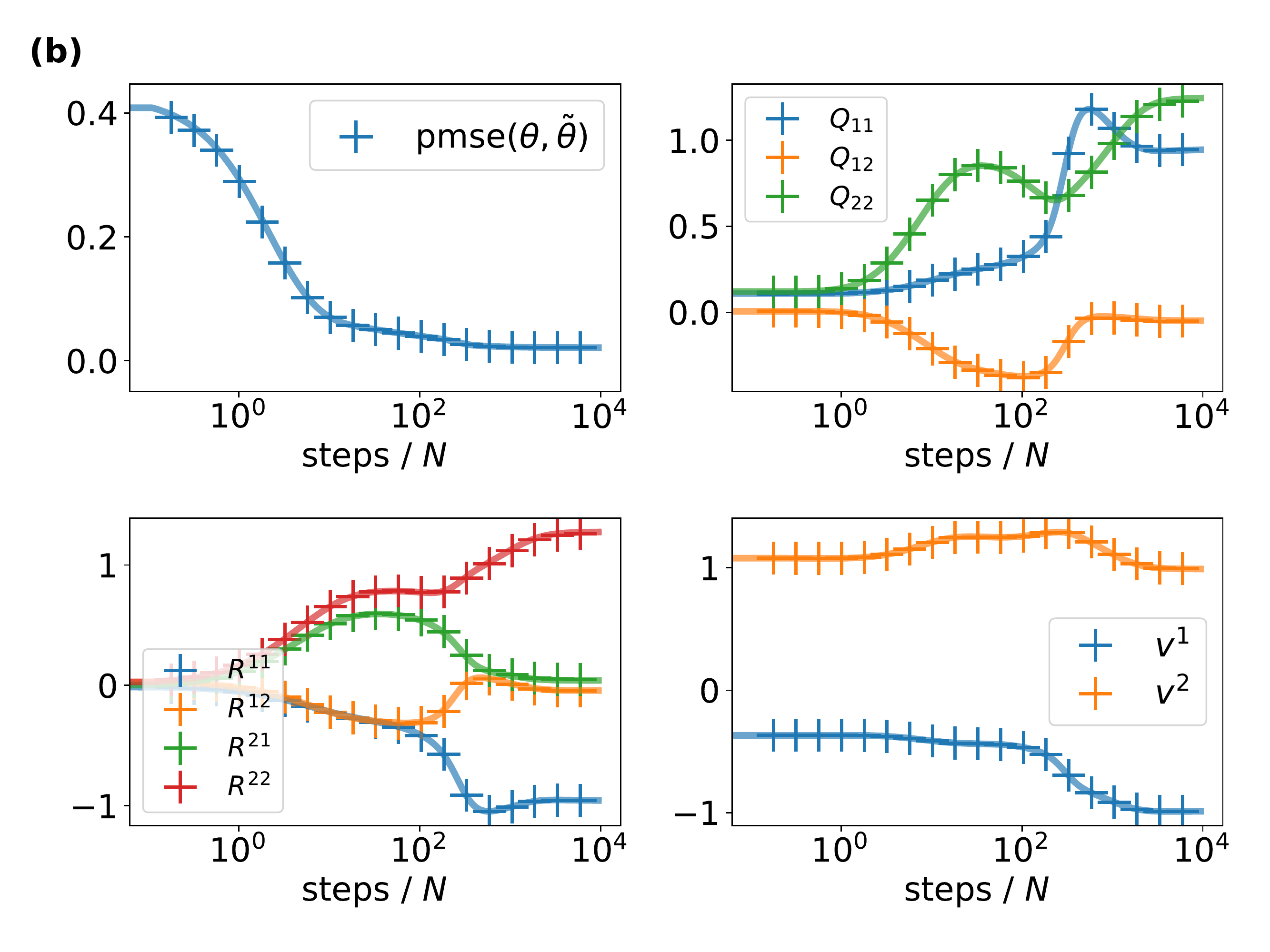}}
  \caption{\label{fig:dcgan_rand} \textbf{Dynamics of two-layer networks: Theory
      vs experiments for random generators.} We compare the evolution of the
    $\pmse$ and the order parameters obtained from integration of
    Eqns.~(\ref{eq:Q_int}-\ref{eq:eom-v}, solid lines) and a single run of SGD
    (crosses). \textbf{(a)} Inputs are generated by a single-layer
    generator~\eqref{eq:proof-generators} with i.i.d.\ weight matrix $A$ and
    sign activation function ($D=800, N=8000$). \textbf{(b)} Inputs were
    generated by the five-layer DCGAN of~\citet{radford2015unsupervised} with
    random weights ($D=100, N=3072$). \emph{In both plots:}
    $M=K=2, \tilde v^m=1, \eta=0.2, g(x) = \tilde g(x) = \erf(x/\sqrt{2})$,
    integration time step $\dd t=0.01$.}
\end{figure}

\section{Experiments}
\label{sec:experiments}

The derivations of both the dynamical equations~(\ref{eq:Q_int}-\ref{eq:eom-v})
for online SGD and the iterative equations~\eqref{eq:saddlepoint} for full-batch
learning with features rely on the deep GEC. While Theorem~\ref{thm:get} gives
verifiable conditions under which the conjecture is true for one-layer
generators, it remains an open problem to establish the deep GEC rigorously. We
thus conducted a set of experiments to compare the predictions for the $\pmse$
made by the theoretical results of Secs~\ref{sec:odes} and~\ref{sec:replicas} to
the test error measured in simulations. For the dynamical equations, this means
comparing the evolution of the $\pmse$ and the order parameters obtained by (i)
integrating Eqns.~(\ref{eq:Q_int}-\ref{eq:eom-v}) and (ii) by evaluating
Eq.~\eqref{eq:order-parameters} explicitly during a \emph{single} run of SGD for
a two-layer student with $K=2$ hidden units. For the full-batch analysis, we
compare the $\pmse$ obtained from iterating Eq.~\eqref{eq:saddlepoint} with the
result obtained by numerically minimising the empirical risk in
Eq.\eqref{eq:erm} with gradient descent for a given sample complexity
$\alpha=\samples/\tilde{N}$. For the dynamical equations, the teacher is taken to be a
two-layer network with $M=2$ hidden units, and for the full-batch learning it is
taken to be a $M=1$ generalised linear model. In both cases, the teacher weights
are drawn i.i.d.\ from the standard normal distribution.

\subsection{Fully-connected and convolutional generators with random weights}

As a first test, we verified that the equations correctly predict the dynamics
of online SGD in a setting where Theorem~\ref{thm:get} applies: a one-layer
generator $\gen(c)$~\eqref{eq:proof-generators} with i.i.d.\ weight matrix $A$
and sign activation function. In a second set of experiments, we drew the inputs
from the deep convolutional GAN (dcGAN) of~\citet{radford2015unsupervised} with
random i.i.d.\ weights. The dcGAN consists of five convolutional layers, each
followed by a Batch Normalisation layer and a ReLU activation function. The
final activation function is $\tanh(x)$ (see Sec.~\ref{sec:furth-exper-results}
for a detailed description). We show an example of the comparison for both
generators in Fig.~\ref{fig:dcgan_rand}, with more runs in
Sec.~\ref{sec:furth-exper-results}. The agreement between equations and
simulations in both experiments is very good.

\subsection{Pre-trained deep convolutional GAN}

\begin{figure}[t!]
  \centering
  \includegraphics[align=c, width=.35\textwidth]{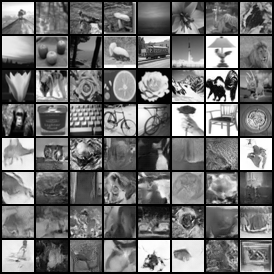}%
  \includegraphics[align=c, width=.55\textwidth]{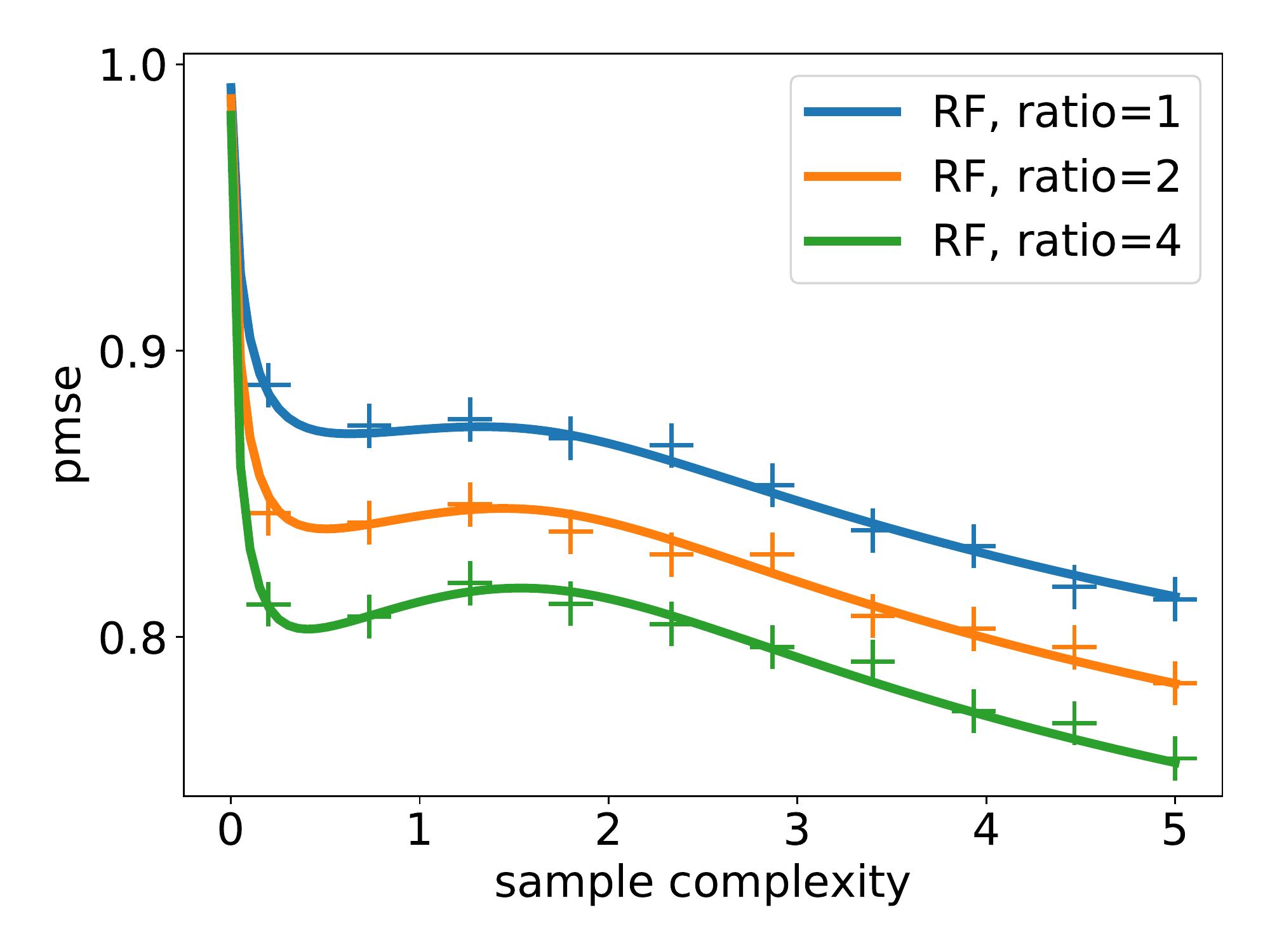}
  \caption{\label{fig:replicas}\textbf{Theory vs experiment for random-features logistic
      regression: } (\emph{Left}) Images drawn from the CIFAR100 dataset in
    grayscale (top four rows) and drawn from the deep convolutional GAN trained
    on CIFAR100 (bottom four rows). (\emph{Right}) Generalisation performance of
    Random Features logistic regression. The random features matrix
    $\mat{F}\in\mathbb{R}^{\tilde{N}\times N}$ was taken to be Gaussian and the
    non-linearity $\sigma = \rm{sign}$. The input data was generated from a
    dcGAN pre-trained on CIFAR100 grayscale data set~\cite{krizhevsky2009learning}  as a function of
    the sample complexity $\alpha = P/\tilde{N}$ and fixed weight decay
    $\lambda = 10^{-2}$. Different curves correspond to different projection
    aspect ratios $\tilde{N}/N$.}
\end{figure}
We also used an instance of a dcGAN that was pre-trained on CIFAR100
dataset~\cite{krizhevsky2009learning} in grayscale, with weights provided
by~\cite{ganweights}. On the left of Fig.~\ref{fig:replicas}, we show 32 samples
of the original dataset (top four rows) and 32 images generated by this network
(bottom four rows). On the level of the replica analysis~\eqref{eq:saddlepoint},
the change of generator weights is reflected in the change of the covariance
matrices $\Omega_{ij}$ and $\Phi_{ir}$~\eqref{eq:covariances}, which need to be
estimated precisely.  In Fig.~\ref{fig:replicas} we compare the $\pmse$ at
different sample complexities predicted by eq.~\eqref{eq:test_error} for
logistic regression with Gaussian features $\mat{F}$ of different sizes with the
result obtained by running gradient descent on the empirical risk. Although we
didn't include the plots for conciseness, we observe the same good agreement for
other tasks and for all the generative models discussed in this section.

\subsection{Normalising flows: the real NVP}
\label{sec:normalising-flows}

\begin{figure}[t!]
  \centering
  \includegraphics[align=c, width=.33\textwidth]{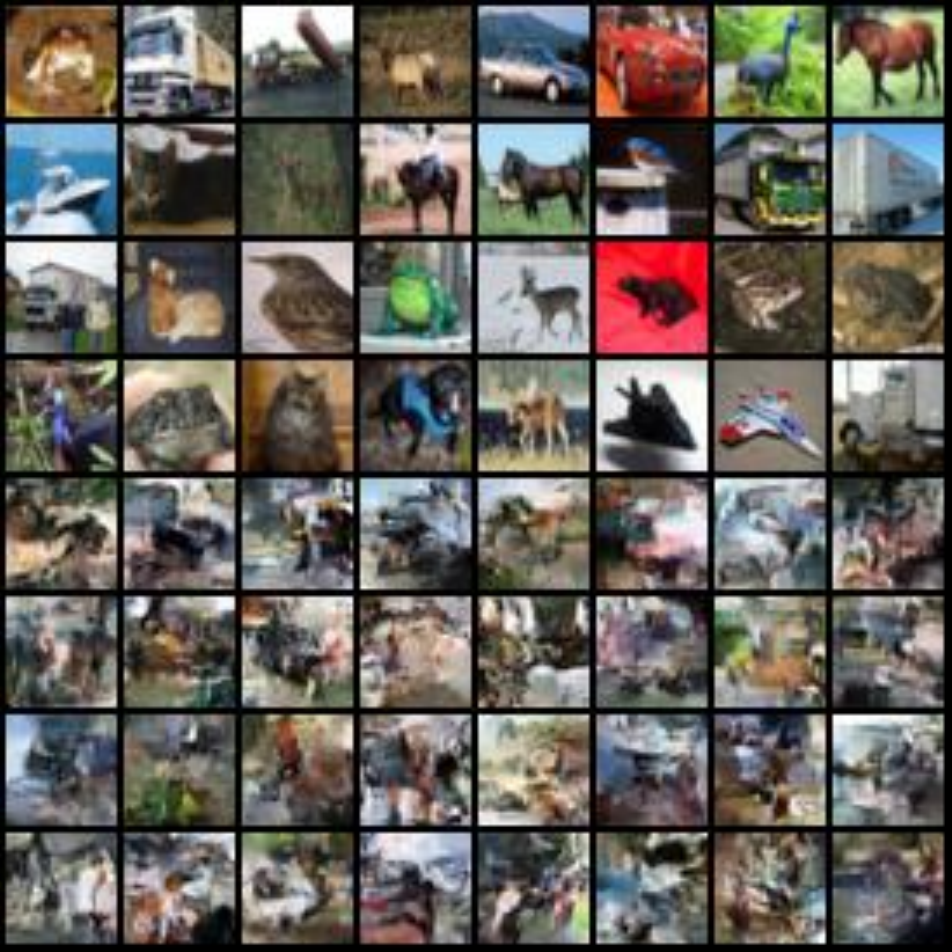}\quad%
  \includegraphics[align=c, width=.6\textwidth]{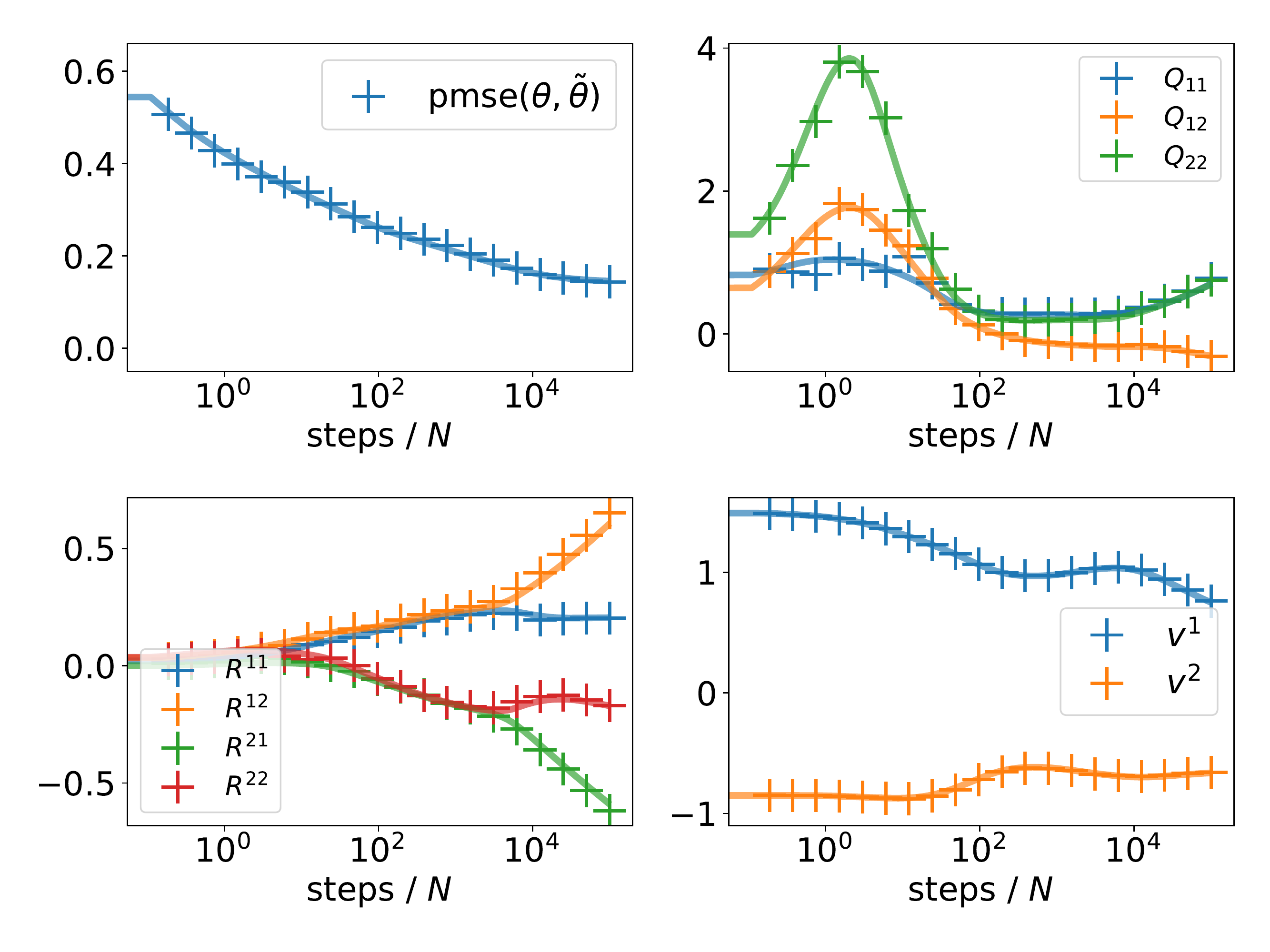}%
  \caption{\label{fig:pretrained}\textbf{Theory vs experiments for online SGD
      with deep, pre-trained realNVP model of~\citet{dinh2017density}}.
    \emph{(Left)} The top four rows show images drawn randomly from the CIFAR10
    data set, the bottom four rows show images drawn randomly from the realNVP
    model trained on CIFAR10. \emph{(Right)} Same plot as
    Fig.~\ref{fig:dcgan_rand} when inputs are drawn from the pre-trained
    realNVP. $D=N=3072$. \emph{In all experiments:}
    $M=K=2, \tilde v^m=1, \eta=0.2, g(x) = \tilde g(x) = \erf(x/\sqrt{2})$,
    integration time step $\dd t=0.01$.}
\end{figure}

We finally tested the validity of the deep GEC with a generative model from the class
of normalising flows~\cite{tabak2010density, tabak2013family,
  rezende2015variational, kobyzev2020normalizing,
  papamakarios2019normalizing}. These models obtain a
given target distribution from a series of bijective transformations of a much
simpler distribution, say the multidimensional normal distribution. Constructing
a probability density in this way has the advantage that the model's output
distribution can be written down exactly, making it possible to minimise the
exact log-likelihood. This should be contrasted with variational
auto-encoders~\cite{kingma2014auto}, where a bound on the log-likelihood is
optimised, or GANs, where the unsupervised problem of density estimation is
transformed into a supervised learning
problem~\cite{goodfellow2014generative}. For the purpose of verifying the GET
via the validity of the dynamical equations, normalising flows have the
desirable property 
that their latent dimension $D$ is equal to the dimension of the output,
i.e.\ for CIFAR10 images, $D=N=3072$, which is close to the regime
$D,N\to\infty$ of our
analysis. 
We trained an instance of the real NVP model of~\citet{dinh2017density} using
the pyTorch port of the original TensorFlow implementation provided
by~\citet{realNVPimplementation}. Using the original
hyper-parameters~\cite{dinh2017density}, we reached an average value of
$\approx 3.5$ bits/dim on the validation set, which agrees with the value of
3.49 bits / dim reported there. Images generated by the trained model are shown
in the bottom four rows of the grid at the bottom of
Fig.~\ref{fig:pretrained}. The comparison between ODEs and simulation (bottom
right of Fig.~\ref{fig:pretrained}) shows very good agreement between the
simulation and the prediction from the ODEs, demonstrating the validity of the
Gaussian Equivalence Property for this instance of a pre-trained generative
model with $\sim 6.3 \cdot 10^6$ trained parameters.

%% file: acknowledgements.tex
\section*{Acknowledgements}

We thank A.~Maillard and F.~Gerace for valuable discussions. We acknowledge funding from the
ERC under the European Union’s Horizon 2020 Research and Innovation Programme
Grant Agreement 714608-SMiLe, from ``Chaire de recherche sur les modèles et
sciences des données'', Fondation CFM pour la Recherche-ENS, and from the French
National Research Agency grants ANR-17-CE23-0023-01 PAIL and ANR-19-P3IA-0001 PRAIRIE.

%% file: appendix.tex

\section{Proof of the Gaussian Equivalence Theorem}
\label{sec:get-proof}
There are two main steps to the proof. First we provide a one-dimensional GET
(Theorem~\ref{thm:get_1d}), which is stated under a more general setting and
then we show how Theorem~\ref{thm:get} of the main text follows as a special
case.

\subsection{One-dimensional GET}\label{sec:get_1d}
Let $Z = (Z_1, \ldots, Z_d)$ be a vector of standard Gaussian variables and let
$X = (X_1, \dots, X_n)$ be generated according to $X_i = \sigma_i(a_i^\top Z)$,
$i = 1, \dots, n$, where each $\sigma_i \colon \reals \to \reals$ and each $a_i$ is a
unit vector in $\reals^d$. Let $\rho$ be the $n \times n$ positive semi-definite
matrix $\rho_{ij} = a_i^\top a_j$ and let $\tilde{\rho} = \rho - I_n$ be the
matrix obtained by setting the diagonal entries to zero. 

The main result of this section provides a Gaussian approximation for a one-dimensional projection of $X$. We define $I$ to be the subset of $[n] = \{1,\dots, n\}$ such that $\sigma_i$ is not affine. Notice that the variables indexed by the complement of the set, namely $\{ X_i , \: \, i \in [n]\backslash I\}$, are jointly Gaussian by construction. 

\begin{assumption}[Weak Correlation]\label{assump:corr}There exists a constant $C_\rho$ such that 
\begin{align}
\sum_{ i,j \in I } \tilde{\rho}_{ij}^4   \le C_\rho^4.
\end{align}
\end{assumption}

\begin{assumption}[Smoothness]\label{assump:smoothness2} Each $\sigma_i$ is twice differentiable. Furthermore, there exists a constant $C_\sigma$ such that for all $i \in I$,
\begin{align}
 \max \left\{ \bEx[\left( \sigma_i(u)\right)^4]^{1/4} , \left (\ex{ (\sigma'_i(u))^2} \right)^{1/2}  , \ex{(\sigma''_i(u))^2}^{1/2}   \right\} \le C_\sigma,
\end{align}
where $u \sim \normal(0,1)$. 
\end{assumption}

Each $\sigma_i$ can be expressed via its Hermite expansion
\begin{align}
\sigma_i(u) = \sum_{k=0}^\infty \hat{\sigma}_i(k) h_k(u),
\end{align}
where $\hat{\sigma}_i(k)$ is the $k$th Hermite coefficient of $\sigma_i$ and
$h_k$ is the $k$th (normalised) probabilist's Hermite polynomial. Note that if
$\sigma_i$ is affine then $\hat{\sigma}_i(k) = 0$ for $k \ge 2$.

\begin{theorem}\label{thm:get_1d}
Let $P$ be the distribution of $\frac{1}{\sqrt{n}} \sum_{i=1}^n X_i$ and let $\hat{P}$ be the Gaussian distribution with the same mean and variance. Under Assumptions~\ref{assump:corr} and~\ref{assump:smoothness2},
\begin{align}
d(P, \hat{P} ) \le \frac{ C  C_\sigma}{ \sqrt{n}} \left( \delta_1 +\sqrt{ n\, \delta_2} +C_\sigma (C_\rho^2 + C_\rho^3) +
C_\sigma^2 (1 + C_\rho^4)   
\right),
\end{align}
where $C$ is a universal constant,
\begin{subequations}
\label{eq:delta} 
\begin{align}
\delta_1 & =\frac{1}{n}  \sum_{i, j, \ell  \in I}   \tilde{\rho}_{ij} \tilde{\rho}_{i \ell}   \left(  \hat{\sigma}_j(1) \hat{\sigma}_\ell(1)   + 2  \rho_{j\ell} \hat{\sigma}_j(2) \hat{\sigma}_\ell(2)\right) + \frac{1}{n} \sum_{i \in I}   \left( \sum_{j  \in [n]\backslash I} \tilde{\rho}_{ij}  \hat{\sigma}_j(1)  \right)^2\\
\delta_2 & = \frac{1}{n} \sum_{i , j, \ell \in I}    \tilde{\rho}^2_{ij} \tilde{\rho}^2_{i\ell}\left( 2 \hat{\sigma}_j(2) \hat{\sigma}_\ell(2)   + 6  \rho_{j\ell} \hat{\sigma}_j(3) \hat{\sigma}_\ell(3)   \right) 
\end{align}
\end{subequations}
and $I $ is the subset of $\{1, \dots, n\}$ such that $\sigma_i$ is not affine.
\end{theorem}

\subsection{Proof of Theorem~\ref{thm:get}}
Having established the one-dimensional GET, we are now in a position to prove Theorem~\ref{thm:get} of the main text. Let $P$ be the distribution on $ \reals^{K + M}$ defined by the variables
\begin{align*}
\lambda^k  = \frac{1}{\sqrt{N}} \sum_{i=1}^N w_i^k x_i , \quad k  =1, \dots, K, \qquad \nu^m  = \frac{1}{\sqrt{D}} \sum_{r=1}^D \tilde{w}_r^m c_r, \quad m = 1, \dots, M
\end{align*}
where $W = (w_i^k) \in \reals^{K \times N}$ and
$\tilde{W} = (\tilde w_r^m) \in \reals^{M \times D}$ are weight matrices and
$c \sim \normal(0, I_D)$ is a vector of latent Gaussian variables.  Recall that
$x \in \reals^N$ is generated according to $ x_i = \sigma( a_i^\top c)$ where
$\sigma : \reals \to \reals$ is a non-linearity and each $a_i$ is a unit vector
in $\reals^D$.

To bound the maximum-sliced distance between $P$ and a Gaussian approximation it is sufficient to bound the difference with respect to  every one-dimensional projection.  Given any unit vector $\alpha \in \reals^{K+M}$ the variable $S \sim \alpha^\top P$ is given by
\begin{align}
S & = \frac{1}{\sqrt{N}} \sum_{i=1}^N    \sum_{k=1}^K \alpha^k  w_i^k  x_i  + \frac{1}{\sqrt{D}} \sum_{r=1}^D    \sum_{m=1}^M  \alpha^{K+m} \tilde{w}_r^m  c_r.
\end{align}
We will now express this variable using the notation in Section~\ref{sec:get_1d} with problem dimensions given by $d = D$ and $n = N+D$. Define $w=(w_i) \in \reals^N$ and $ \tilde{w} = (\tilde{w}_r) \in \reals^D$ according to 
\begin{align}
w_i = \sum_{k=1}^K \alpha^k  w_i^k, \qquad   \tilde{w}_r = \sum_{m=1}^M  \alpha^{K+m} \tilde{w}_{i-N}^m .
\end{align}
Letting  $Z = (Z_1, \dots, Z_d)$ be a vector of i.i.d.\ standard Gaussian variables, the distribution of $S$ is equal to the distribution  $  \frac{1}{\sqrt{n}} \sum_{i=1}^n X_i$
where
\begin{align}
  X_i  & = \begin{dcases}  \sqrt{\frac{n}{N}}  w_i  \sigma( a_i^\top Z) , &  1\le i \le N\\
    \sqrt{\frac{n}{D}}   \tilde{w}_r  e_{i-N}^\top Z , \quad&  N < i \le N+D 
  \end{dcases}
\end{align}
and $e_r$ denotes the $r$th standard basis vector in $\reals^d$. Furthermore, the
assumptions of Theorem~\ref{thm:get_1d} are satisfied where $I = \{ 1, \dots, N\}$ is the set of indices for which $X_i$ is a
non-affine function of $Z$, the constant 
$C_{\sigma}$ is bounded uniformly by the assumptions on $\sigma$ and the students weights, and  $C_\rho = (\sum_{i\ne j} (a_i^\top a_j)^4)^{1/4}$. Applying Theorem~\ref{thm:get_1d} and retaining the dominant terms with respect to $C_\rho$, one finds that the distance between the projection of  $P$
and the projection of the Gaussian distribution $\hat{P}$ with matched first and second moments satisfies
\begin{align}
d(\alpha^\top P  , \alpha^\top \hat{P}  )  \le  \tilde{C} \left( \frac{ \delta_1}{ \sqrt{N}}   + \sqrt{\delta_2}  + \frac{ \sum_{i\ne j} (a_i^\top a_j)^4  }{\sqrt{N}} + \frac{1}{ \sqrt{N}} \right) ,
\end{align}
where $\tilde{C}$ is a constant that depends on the regularity assumption of $\sigma$ and the maximum magnitude of the students weights and 
\begin{align}
\delta_1 & =\frac{1}{N}   \sum_{i, j, \ell = 1}^N  w_j w_\ell   \tilde{\rho}_{ij} \tilde{\rho}_{i \ell}   \left(  \hat{\sigma}^2(1)    + 2  \rho_{j\ell} \hat{\sigma}^2(2) \right) + \frac{1}{D}   \sum_{i = 1}^N \sum_{ r , r'  =1}^D    a_{i r} a_{i r'}  \tilde{w}_r \tilde{w}_{r'}  \\
\delta_2 & = \frac{1}{N}   \sum_{i ,j,  \ell = 1}^N  w_j w_\ell   \tilde{\rho}^2_{ij} \tilde{\rho}^2_{i\ell}\left( 2 \hat{\sigma}^2(2)   + 6  \rho_{j\ell} \hat{\sigma}^2(3)    \right).
\end{align}
Recalling the definitions of the matrices $M_1$ and $M_2$, it follows that 
\begin{align}
 \frac{ \delta_1}{ \sqrt{N}} &= O\left( \left\| \frac{1}{\sqrt{N} } w^\top M_1^{1/2} \right \|^2    + \frac{1}{\sqrt{N}} \left\| \frac{1}{\sqrt{D}} \tilde{w}^\top A^\top\right  \|^2 \right)  \\
 \sqrt{\delta_2} & = O\left(   \left\| \frac{1}{\sqrt{N} } w^\top M_2^{1/2} \right \| \right).
\end{align}
Finally, recalling the definition of $(w, \tilde{w})$ we see that the 
following bounds holds uniformly with respect to $\alpha$: 
\begin{align}
 \frac{ \delta_1}{ \sqrt{N}} &= O\left( \left\| \frac{1}{\sqrt{N} } W M_1^{1/2} \right \|^2 + \frac{\hat{\sigma}(1)^2}{\sqrt{N}} \left\| \frac{1}{\sqrt{N}} W A \right  \|^2  + \frac{1}{\sqrt{N}} \left\| \frac{1}{\sqrt{D}} \tilde{W}^\top A^\top\right  \|^2 \right)  \\
 \sqrt{\delta_2} & = O\left(   \left\| \frac{1}{\sqrt{N} } W M_2^{1/2} \right \| \right).
\end{align}
This completes the proof of Theorem~\ref{thm:get}.

\subsection{Proof of Theorem~\ref{thm:get_1d}}

\subsubsection{Gaussian comparison}
The following results show that it is sufficient to bound the distance between $P$ and a Gaussian distribution that has the same mean but possibly different variance. 

\begin{lemma}\label{lem:distance_Gaussian}
For any $\sindex \in \reals$ and $v_1, v_2  \ge 0$, 
\begin{align}
d(\normal( \sindex, v_1) , \normal(\sindex, v_2) )  = \frac{1}{2} \left| v_1 - v_2\right|
\end{align}
\end{lemma}
\begin{proof}
Without loss of generality assume $v_1 \le v_2$. Letting $U_1, U_2$ be independent standard Gaussian variables we have $X_1 = \sindex + \sqrt{v_1}  U_1 \sim \normal(\sindex, v_1)$ and $X_2 =X_1+ \sqrt{ v_2  - v_1} U_2 \sim \normal(\sindex, v_2)$. For each $f \in \cF$, a second order Taylor series expansion gives
\begin{align}
 f(X_2)  - f(X_1)  \le    \sqrt{ v_2- v_1} U_2  f'(X_1) + \frac{1}{2} (v_2 - v_1) U^2_2 \|f''\|_\infty.
 \end{align}
The first term has zero mean, because $U_2$ is independent of $X_1$. By assumption $\|f''\|_\infty \le 1$ and thus $ \left| \ex{ f(X_2) }  - \ex{  f(X_1) } \right|  \le  \frac{1}{2}  |v_2 - v_1|$ for all $f \in \cF$. To see that this upper bound is tight, note that the inequality is attained for the choice  $f(x) =\frac{1}{2} (x-\sindex)^2$.
\end{proof}

\begin{lemma}\label{lem:d_bound}
Let $P$ be a distribution on $\reals$ with mean $\sindex$ and variance $v$. For  all $\tilde{v} \ge 0$,
\begin{align}
d(P, \normal(\sindex, v)) \le 2 d(P, \normal(\sindex, \tilde{v})) .
\end{align}
\end{lemma}
\begin{proof}
By the triangle inequality,
\begin{align}
d(P, \normal(\sindex, v)) \le  d(P, \normal(\sindex, \tilde{v})) +d(\normal(\sindex, v), \normal(\sindex, \tilde{v})). 
\end{align}
Noting that the function $f(x) = \frac{1}{2} (x-\sindex)^2$ belongs to $\cF$ the first term satisfies $
d(P, \normal(\sindex, \tilde{v})) \ge \frac{1}{2} | v - \tilde{v}  |$. 
By Lemma~\ref{lem:distance_Gaussian}, the second term satisfies $   d(\normal(\sindex, v), \normal(\sindex, \tilde{v})) = \frac{1}{2}|v - \tilde{v}|$. Combining these inequalities gives the stated result. 
\end{proof}


\subsubsection{Replacement method}

We assume with without loss of generality that each $X_i$ has zero mean and thus $\hat{\sigma}_i(0) = 0$.  For the purposes of comparison, we define the Gaussian variables 
\begin{align}
U_i = a_i^\top Z ,  \qquad  \hat{X}_i  =  \hat{\sigma}_i(1) U_i +  \xi_i, 
\end{align}
where $\xi_1, \dots, \xi_n$ are independent Gaussian variables with mean zero and variance $\var(\xi_i) = \var(X_i) - \hat{\sigma}^2_i(1)$ chosen such that $X_i$ and $\hat{X}_i$ have the same second moment. Notice that each $U_i$ has mean zero, unit variance, and $\cov(U_i, U_j) =\rho_{ij}$. Moreover, 
since $\frac{1}{\sqrt{n}} \sum_{i=1}^n \hat{X}_i$ is a Gaussian variable with the same mean as $\frac{1}{\sqrt{n}} \sum_{i=1}^n X_i$ it follows from  Lemma~\ref{lem:d_bound} that $d(P, \hat{P} ) \le 2 \sup_{f \in \cF} \Delta(f)$ where 
\begin{align}
\Delta(f) =  \ex{ f\left( \frac{1}{\sqrt{n}} \sum_{i=1}^n X_i\right )  -  f\left( \frac{1}{\sqrt{n}} \sum_{i=1}^n \hat{X}_i\right )}.
\end{align}

We use the replacement method to bound the term  $ \Delta(f)$. For $i = 1, \dots, n$ define the hybrid random variable  
 \begin{align}
S_{i} = \frac{1}{\sqrt{n}} \sum_{ j =1 }^{i-1} X_j  + \frac{1}{\sqrt{n}}  \sum_{j = i+1}^n  \hat{X}_j,
\end{align}
which excludes the contribution of the $i$th term. Then, we obtain the telescoping sum:
\begin{align} 
\Delta(f)  & = \sum_{i=1}^n \Delta_i(f)  , \quad \text{where} \quad \Delta_i(f)   =  \ex{ f\left(S_i +  \tfrac{1}{\sqrt{n}} X_i \right) - f\left(S_i +  \tfrac{1}{\sqrt{n}}  \hat{X}_i  \right)}.  \label{eq:Delta_decomp} 
\end{align}

The next result provides a useful bound on $\Delta_i(f)$ in terms of auxiliary random variables. 


\begin{lemma}\label{lem:Delta_i} 
Let $(A_i, B_i)$ be a pair of random variables that is independent of $(U_i, \xi_i)$. Then, 
\begin{align}
\Delta_i(f) \le \frac{ C K_i  }{ \sqrt{n}}  \left(    \ex{  B_i^2 }   
 + \ex{  ( S_i  - A_i  )^2 }   + \sqrt{  \ex{  (S_i  - A_i   - B_i  U_i  )^2}  }   + \frac{K_i^2 }{ n}    \right) 
\end{align}
where $C$ is a universal constant and $K_i = (\ex{ X_i^4})^{1/4}$. 
\end{lemma} 
\begin{proof}

For any real numbers $s,x,y$, a third order Taylor series expansion of $f$ about $s$ yields
\begin{align}
\left| f( s  + x  )  -  f( s  +y  )  -    (x- y) f'( s  )    -  (x^2 - y^2)  f''( s)  \right|  \le  \frac{1}{6} (|x|^3 + |y|^3) \|f'''\|_\infty.
\end{align} 
Furthermore, for any real numbers $a,b,u$, we can write
\begin{align}
\left| f'(s) -  f'(a  + b u) \right| &\le  |s - a - b u| \|f''\|_\infty\\
 \left| f'(s) - f'(a) - b u f''(a) \right|  & \le  |s - a - b u| \|f''\|_\infty+ \frac{1}{2} (bu)^2 \|f'''\|_\infty \\
\left| f''(s) - f''(a) \right| & \le  |s - a| \|f'''\|_\infty.
\end{align}

Combining the above displays with the assumption $\|f''\|_\infty, \|f'''\|_\infty \le 1$ yields
\begin{align}
f( s  + x  )  -  f( s  +y  ) & \le   (x- y)  \left[ f'( a   )   + bu   f''(a) \right]   + (x^2 - y^2)  f''( a  )  +  | x- y  |   \, | s -a  -  b u  |  \notag \\
&\quad  + \frac{1}{2}  |x - y |   \, (bu)^2  +  |x^2 - y^2| \, | s - a |   + \frac{1}{6}  \left(  |x|^3 + |y)|^3   \right).
\end{align}
Evaluating this inequality with $(a,b,s,u,x,y)$ replaced by $(A_i, B_i, S_i, U_i, \frac{1}{\sqrt{n}} X_i, \frac{1}{\sqrt{n}} \hat{X}_i)$ and then taking the expectation of both sides leads to
\begin{align}
\Delta_i(f) 
 & \le    \frac{1}{\sqrt{n}} \ex{ X_i - \hat{X}_i } \ex{f'( A_i   ) }  +    \frac{1}{\sqrt{n}} \ex{   (X_i - \hat{X}_i ) U_i  }  \ex{ B_i    f''(A_i)  } \notag \\
 &\quad + \frac{1}{n}  \ex{ (X_i^2 - \hat{X}_i ^2) } \ex{  f''( A_i   ) } + \frac{1}{\sqrt{n}}  \ex{  | X_i - \hat{X}_i  |   \, | S_i A_i   - B_i  U_i  | }   \notag \\
& \quad + \frac{1}{2 \sqrt{n}} \ex{ | X_i - \hat{X}_i  | U_i^2 } \, \ex{  B_i^2 }    + \frac{1}{n} \ex{  |X_i^2 - \hat{X}_i^2| \, | S_i  - A_i  | }   \notag \\
& \quad  + \frac{1}{6 n^{3/2} }  \left( \ex{  |X_i|^3}  + \ex{ |\hat{X}_i|^3}    \right).
\end{align}
Here,  we have used the independence between $(A_i, B_i)$  and  $(U_i, \xi_i)$ to factorise the expectations. By the construction of $\hat{X}_i$ the first three terms on the right-hand side are zero. Using the Cauchy-Schwarz inequality and the Jensen's inequality, the upper bound can be simplified as follows: 
\begin{align}
\Delta_i(f) 
 & \le   \frac{1}{\sqrt{n}}  \sqrt{  \ex{  ( X_i - \hat{X}_i )^2}    \ex{  (S_i A_i   - B_i  U_i  )^2}  }   + \frac{1}{2 \sqrt{n}} \sqrt{ \ex{(  X_i - \hat{X}_i )^2} \ex{  U_i^4 } }  \, \ex{  B_i^2 }    \notag \\
& \quad + \frac{1}{n}  \sqrt{ \ex{  (X_i^2 - \hat{X}_i^2)^2 } \ex{  ( S_i  - A_i  )^2 } }   + \frac{1}{6 n^{3/2} }  \left( \ex{  |X_i|^3}  + \ex{ |\hat{X}_i|^3}    \right). 
\end{align}
From the construction of $\hat{X}_i$ it is straightforward to verify that 
\begin{align*} 
\ex{ (X_i - \hat{X}_i)^2} \le  C_1 K_i^2, \quad \ex{ (X^2_i - \hat{X}^2_i)^2} \le  C_2 K_i^4 , \quad \left( \ex{  |X_i|^3}  + \ex{ |\hat{X}_i|^3}    \right) \le C_3 K_i^3
\end{align*}
for universal constants $C_1, C_2, C_3$, and thus
\begin{align}
\Delta_i(f) \le \frac{ C K_i  }{ \sqrt{n}}  \left(    \ex{  B_i^2 }   +  
\frac{K_i  }{\sqrt{n} } \sqrt{ \ex{  ( S_i  - A_i  )^2 } }   + \sqrt{  \ex{  (S_i  - A_i   - B_i  U_i  )^2}  }   + \frac{K_i^2 }{ n}    \right).
\end{align}
Finally, by  the basic inequality $x y \le  \frac{1}{2} ( x^2 +  y^2)$  we have
\begin{align}
\frac{K_i }{\sqrt{n} } \sqrt{ \ex{  ( S_i  - A_i  )^2 } }  &  \le  \frac{ K_i^2}{ 2 n} +     \frac{1}{2} \ex{  ( S_i  - A_i  )^2 } ,
\end{align}
and combining the last two displays gives the stated bound. 
\end{proof}

\subsubsection{Decomposition argument} 
In view of Lemma~\ref{lem:Delta_i}, the next question is how to specify the variables $(A_i, B_i)$. We use a decomposition argument that leverages the Gaussianity of $U$.   Let $i$ be fixed and for each $j\ne i$ define the Gaussian variables $\tilde{U}_j  = U_j - \rho_{ij} U_i$. Note that $U_i$ and $\tilde{U}_j$ are uncorrelated and thus independent. Further define $V_i = (\tilde{U}_{1}, \dots, \tilde{U}_{i-1}, \tilde{U}_{i+1}, \dots, \tilde{U}_n, \xi_1, \dots, \xi_{i-1}, \xi_{i+1} , \dots, \xi_n)$. Then, we can write $S_i = g_i(U_i ,  V_i) $
where
\begin{align}
g_i(U_i ,  V_i) & =  \frac{1}{\sqrt{n}} \sum_{ j =1 }^{i-1} \sigma_j( \rho_{ij} U_i   +\tilde{U}_j)  + \frac{1}{\sqrt{n}}  \sum_{j = i+1}^n \left( \hat{\sigma}_j(1) ( \rho_{ij} U_i  + \tilde{U}_j)   + \xi_j \right) .
\end{align} 
Since $V_i$ is independent of $(U_i, \xi_i)$ we can define $(A_i, B_i)$ as a function of $V_i$. Specially, we define the variables to be the first and second Hermite coefficients of the mapping  $u \mapsto g_i(u  , V_i)$:
\begin{align}
A_i & = \ex{ S_i \mid V_i} = \hat{g}_i(0 ; V_i)  , \qquad B_i  = \ex{ U_i S_i \mid V_i}  = \hat{g}_i(1 ; V_i).
\end{align}
By Gaussian integration by parts, we can also write  $B_i =  \ex{ g'_i(U_i, V_i)  \mid V_i}$ 
where $g'_i(u,v)$ denotes the partial derivative with respect to the first argument. In conjunction with Jensen's inequality, we obtain the following upper bound: 
\begin{align}
\ex{ B_i^2}  & = \ex{  \ex{ g'_i(U_i, V_i)  \mid V_i}^2} \le  \ex{  (g'_i(U_i, V_i) )^2}. \label{eq:B2_bound}
\end{align}

\begin{lemma}\label{lem:Poincare}
Let $U \sim \normal(0,1)$ and let $g\colon \reals \to \reals$ be a twice differentiable  with $\ex{ g^2(U)} < \infty$. Then, 
\begin{align}
\ex{ (g(U) - \hat{g}(0)  )^2} & \le \ex{ (g'(U))^2} \\
\ex{ (g(U) - \hat{g}(0) - \hat{g}(1) U )^2} & \le \ex{ (g''(U))^2} .
\end{align}
\end{lemma}
\begin{proof}
The first inequality is the  Gaussian Poincar\'e inequality. For the second inequality we use the Plancherel formula~\cite[Proposition 11.36]{odonnell2014analysis} to write
\begin{align}
\ex{ (g(U) - \hat{g}(0) - \hat{g}(1) U )^2} & = \sum_{k=2}^\infty \hat{g}(k)^2 
\le  \frac{1}{\sqrt{2}}  \sum_{k=0}^\infty \widehat{g''}(k)^2  = \frac{1}{\sqrt{2}} \ex{ ( g''(Z) )^2 } 
\end{align} 
where the third step follows from the  relation 
 $\widehat{g''}(k) = \sqrt{k + 1} \sqrt{ k + 2}  \hat{g}(k+2)$ for non-negative inters $k$.
\end{proof}

Using Lemma~\ref{lem:Poincare}, we obtain
\begin{align}
 \ex{  ( S_i  - A_i  )^2 }    & \le  \ex{ ( g'(U_i, V_i))^2 } , \quad 
  \ex{  ( S_i  - A_i - B_i U_i   )^2 }     \le  \ex{ ( g''(U_i, V_i))^2 } . \label{eq:S_moment_bound}
\end{align} 
Combining  Lemma~\ref{lem:Delta_i} with and~\eqref{eq:B2_bound}  and~\eqref{eq:S_moment_bound} yields
\begin{align}
\Delta_i(f) \le  \frac{C K_i}{\sqrt{n}} \left(  \ex{ ( g'_i(U_i, V_i))^2 } +   \sqrt{ \ex{ ( g_i''(U_i, V_i))^2 } }   + \frac{K_i^2}{n}  \right). \label{eq:Delta_bound_c}
\end{align}

\begin{lemma}\label{lem:gi_bound}
 Under Assumptions~\ref{assump:corr} and~\ref{assump:smoothness2},
\begin{align}
\ex{  (g_i'(U_i, V_i))^2}&  \le \left(  \frac{1}{\sqrt{n}}  \sum_{j \in [n]}   \tilde{\rho}_{ij} \hat{\sigma}_j(1) \right)^2 +\frac{2}{n}  \sum_{j_1, j_2  \in [i]  }  \tilde{\rho}_{ij_1} \tilde{\rho}_{ij_2}\tilde{\rho}_{j_1j_2}   \hat{\sigma}_{j_1}(2) \hat{\sigma}_{j_2}(2) \notag \\& \quad   + \frac{C^2_\sigma (1 + C^2_\rho) }{n} \sum_{j   \in I  } \tilde{\rho}_{ij}^2 \\
\ex{  (g''(U_i, V_i))^2}&  \le \left(  \frac{\sqrt{2}}{\sqrt{n}}  \sum_{j \in [i]}   \tilde{\rho}^2_{ij} \hat{\sigma}_j(2) \right)^2 + \frac{6}{n}  \sum_{j_1, j_2  \in [i]  }  \tilde{\rho}^2_{ij_1} \tilde{\rho}^2_{ij_2} \tilde{\rho}_{j_1j_2} \widehat{\sigma}_{j_1}(3) \widehat{\sigma}_{j_2}(3)    \notag \\ &\quad  + \frac{C_\sigma^2 (1 + C_\rho^2)  }{n} \sum_{j   \in I  } \tilde{\rho}_{ij}^4 
\end{align} 
\end{lemma}
\begin{proof}
Recalling that  $\tilde{\rho}_{ij} = \rho_{ij} \one_{i \ne j}$ and using the relation $\hat{\sigma}_j(1) =  \ex{ \sigma'_j( U_j) }$ leads to
\begin{align}
 g'_i(U_i, V_i)  
 & =   \frac{1}{\sqrt{n}} \sum_{ j \in [i] }  \tilde{\rho}_{ij}  ( \sigma'_j( U_j)   - \ex{ \sigma'_j( U_j) })  + \frac{1}{\sqrt{n}}  \sum_{j \in [n]}   \tilde{\rho}_{ij} \hat{\sigma}_j(1).
\end{align} 
Because the first term has zero mean and the second term is non-random, it follows that
\begin{align*}
\var( g_i'(U_i, V_i))
& = \frac{1}{n}  \sum_{j_1, j_2  \in [i]  \, : \, j_1 \ne j_2}  \tilde{\rho}_{ij_1} \tilde{\rho}_{ij_2} \cov(\sigma'_{j1}(U_j) , \sigma'_{j_2}(U_j))   + \frac{1}{n} \sum_{j   \in [i]  } \tilde{\rho}_{ij}^2 \var(\sigma'_j(U_j)) . 
\end{align*} 
Expanding the covariance in terms of the Hermite coefficients yields
\begin{align}
\cov(\sigma'_{j1}(U_{j_1}) , \sigma'_{j_2}(U_{j_2}))   & = \sum_{k=1}^\infty \rho_{j_1j_2} \widehat{\sigma'}_{j_1}(k) \widehat{\sigma'}_{j_2}(k)\\
&\le \rho_{j_1j_2} \widehat{\sigma'}_{j_1}(k) \widehat{\sigma'}_{j_2}(k) +  \rho^2_{j_1j_2} \sum_{k=1}^\infty |\widehat{\sigma'}_{j_1}(k) \widehat{\sigma'}_{j_2}(k)|\\
&\le  2 \rho_{j_1j_2} \widehat{\sigma'}_{j_1}(k) \widehat{\sigma'}_{j_2}(k) +  \rho^2_{j_1j_2} \sqrt{ \var(\sigma'_{j1}(U_{j_1}) \var(\sigma'_{j1}(U_{j_2})}
\end{align} 
where the last line follows from  $\widehat{\sigma'}_{j}(1) = \sqrt{2}  \widehat{\sigma}_{j}(2)$ and the Cauchy-Schwarz inequality. Since  $\var(\sigma'_{j}(U_{j})$ is equal to zero if $\sigma_{j}$ is affine and bounded by $C_\sigma^2$ otherwise, we can write
\begin{align*}
\var( g_i'(U_i, V_i))
&  \le \frac{2}{n}  \sum_{j_1, j_2  \in [i]  }  \tilde{\rho}_{ij_1} \tilde{\rho}_{ij_2}\tilde{\rho}_{j_1j_2}   \hat{\sigma}_{j_1}(2) \hat{\sigma}_{j_2}(2)    + \frac{C_\sigma^2}{n}  \sum_{j_1, j_2  \in I  } | \tilde{\rho}_{ij_1} \tilde{\rho}_{ij_2} | \tilde{\rho}^2_{j_1j_2}    + \frac{C^2_\sigma}{n} \sum_{j   \in I  } \tilde{\rho}_{ij}^2 .
\end{align*} 
Finally, by the Cauchy-Schwarz inequality, the second term can be simplified as follows:
\begin{align}
 \sum_{j_1, j_2  \in  I }  \tilde{\rho}_{ij_1} \tilde{\rho}_{ij_2}  \tilde{\rho}_{j_1j_2}^2  & \le  \sqrt { \sum_{j_1, j_2  \in   I }  \tilde{\rho}^2_{ij_1} \tilde{\rho}^2_{ij_2} } \sqrt{     \sum_{j_1, j_2  \in  I }   \tilde{\rho}_{j_1j_2}^4  }  \le   C_\rho^2  \sum_{j  \in I }  \tilde{\rho}^2_{ij} 
 \end{align}

 Using a similar approach for $g''_i(U_i,V_i)$ and noting that $\widehat{\sigma''}_{j}(0) = \sqrt{2}  \hat{\sigma}_{j}(2)$ and  $\widehat{\sigma''}_{j}(1) = \sqrt{6}  \hat{\sigma}_{j}(3)$ leads to
 \begin{align*}
 \ex{ g''_i(U_i, V_i) } & = \frac{\sqrt{2}}{\sqrt{n}}  \sum_{j \in [i]}   \tilde{\rho}^2_{ij} \hat{\sigma}_j(2)\\
 \var( g''_i(U_i, V_i)) &  = \frac{1}{n}  \sum_{j_1, j_2  \in [i]  \, : \, j_1 \ne j_2}  \tilde{\rho}^2_{ij_1} \tilde{\rho}^2_{ij_2} \cov(\sigma''_{j1}(U_j) , \sigma''_{j_2}(U_j))   + \frac{1}{n} \sum_{j   \in [i]  } \tilde{\rho}_{ij}^4 \var(\sigma''_j(U_j))\\
 &  \le  \frac{1}{n}  \sum_{j_1, j_2  \in [i]  }  \tilde{\rho}^2_{ij_1} \tilde{\rho}^2_{ij_2} \tilde{\rho}_{j_1j_2} \widehat{\sigma''}_{j_1}(1) \widehat{\sigma''}_{j_2}(1)   +   \frac{C_\sigma^2}{n}  \sum_{j_1, j_2  \in I  }  \tilde{\rho}^2_{ij_1} \tilde{\rho}^2_{ij_2}  \tilde{\rho}^2_{j_1j_2}     + \frac{C_\sigma^2 }{n} \sum_{j   \in I  } \tilde{\rho}_{ij}^4 \\
  &  \le  \frac{6}{n}  \sum_{j_1, j_2  \in [i]  }  \tilde{\rho}^2_{ij_1} \tilde{\rho}^2_{ij_2} \tilde{\rho}_{j_1j_2} \widehat{\sigma}_{j_1}(3) \widehat{\sigma}_{j_2}(3)      + \frac{C_\sigma^2 (1 + C_\rho^2)  }{n} \sum_{j   \in I  } \tilde{\rho}_{ij}^4 
 \end{align*}
\end{proof}

\subsubsection{Final steps in proof}

In view of~\eqref{eq:Delta_decomp},~\eqref{eq:Delta_bound_c}, and Lemma~\ref{lem:gi_bound}, we have all the ingredients needed to  bound $\Delta(f)$. To simplify the analysis, 
observe that the replacement method can be applied with respect to any permutation $\pi$ of the problem indices $[n]$. Averaging over all possible permutations of $\pi$ of $[n]$ we can write
\begin{align}
\Delta(f)  & = \frac{1}{n!}  \sum_{\pi}  \sum_{i=1}^n  \Delta_{i,\pi}(f)   
\end{align}
where $ \Delta_{i,\pi}(f)$ is defined with respect to the permuted variables $(X_{\pi(1)}, \dots, X_{\pi(n)})$. Swapping expectation over $\pi$ and the summation over $i$, and combining with~\eqref{eq:Delta_bound_c} and Lemma~\ref{lem:gi_bound} we obtain an bound that holds uniformly for all $i$:
\begin{align}
 \frac{1}{n!}  \sum_{\pi}   \Delta_{i,\pi}(f)   & \le  \frac{1}{n}   \frac{C C_\sigma}{\sqrt{n}} \left( \delta_1 + C_{\sigma}^2 (1 + C_\rho^2)  \frac{1}{n} \sum_{i,j \in I}  \tilde{\rho}_{ij}^2   +  \sqrt{n \delta_2 + C_\sigma^2 (1 + C_\rho^2) C_{\rho}^4}    +C_\sigma^2  \right).
\end{align}
Noting that $\sum_{i,j \in I} \tilde{\rho}_{ij}^2 \le n C_\rho^2$ and simplifying the dependence on the constants $C_\sigma, C_\rho$ gives the stated result. This concludes thee proof of Theorem~\ref{thm:get_1d}

\section{Conditions for the GET}
\label{sec:conditions-get}

In this appendix we explore the conditions for the Gaussian equivalence theorem
in more detail. For an $N \times D$ matrix  $A$, we define the symmetric $N \times N$ matrices $\rho \equiv AA^\top$ and $\tilde{\rho} \equiv AA^\top - \Id_N$. Then, the matrices $M_1$ and $M_2$ appearing in Theorem~\ref{thm:get} can be expressed as
\begin{align}
  M_1 &= \hat{\sigma}^2(1) K_{11}  + \hat{\sigma}^2(2) K_{21} \\
  M_2 &= \hat{\sigma}^2(2) K_{21} + \hat{\sigma}^2(3) K_{22},
\end{align}
where
\begin{align}
  K_{11} & =\frac{1}{\sqrt{N}} \,  \tilde{\rho}^2\\
  K_{12} & =  \frac{1}{\sqrt{N}} \,  \tilde{\rho}^2 \circ \rho\\
  K_{21} & = (\tilde{\rho} \circ \tilde{\rho})^2 \\
  K_{22} & = (\tilde{\rho} \circ \tilde{\rho})^2 \circ \rho
\end{align}
These matrices are positive definite by the Schur product
theorem~\cite[Sec. 7.5]{horn2012matrix}, and thus have positive real eigenvalues. We are
interested in how the leading eigenvalues and eigenvectors depend on $A$.


To gain insight into the scaling behaviour of the matrices, we consider a setting where the entries of $A$ are i.i.d.\ according to
\[
A_{ij}  =  \frac{1}{\sqrt{D}} \left(  \mu     + \sqrt{ 1- \mu^2}  \,  Z_{ij}\right) 
\]
where $\mu \in [0,1]$ is a deterministic parameter and  $\{Z_{ij}\}$ are i.i.d.\ standard Gaussian variables. The normalisation by $1/\sqrt{D}$ ensures that the column norms of $A$ converges to one almost surely as $D \to \infty$.

\subsection{Deterministic setting} 

In the limit where $N$ is fixed and $D \to \infty$, it follows from the law of large numbers that  $\rho = A A^\top$ converges almost surely to the deterministic $N \times N$ matrix given by
\begin{align}
\rho = \mu^2 \one_{N \times N }  + (1- \mu^2) \Id_N. 
\end{align}
Notice that this is the same matrix given Example 2 with $\mu^2 = c/\sqrt{N}$. The matrices $K_{ij}$ can be computed exactly as
\begin{subequations}
\label{eq:K_deterministic} 
%
\begin{align}
K_{11} &
= \frac{\mu^4}{ \sqrt{N}} \left( (N- 2) \one_{N \times N} + \Id_N \right)\\
    K_{12}
    & =  \frac{\mu^4}{ \sqrt{N}} \left( (N- 2)  \mu^2 \one_{N \times N}  +   [ (N-
      2)(1- \mu^2) + 1]  \Id_N \right) 
\\
  K_{21} &
 =\mu^4 N^{1/2} K_{11} \\
K_{22} & = \mu^4 N^{1/2} K_{12}.
\end{align}
\end{subequations}
Since each of these matrices can be expressed as a weighted sum of the all ones matrix and the identity matrix, their eigenvalue decompositions can be described using using the following elementary result.

\begin{lemma}\label{lem:ones_plust_Id} 
If $K = \alpha \one_{N \times N} + \beta \Id_N$ for real numbers $\alpha, \beta$ with $\alpha  \ge 0$, then the leading eigenvector of $K$ is proportional to the all ones vector and the ordered real eigenvalues $\lambda_1(K) \ge \lambda_2(K) \ge \dots \ge \lambda_N(K)$ are given by
\begin{align}
\lambda_i(K) = \begin{dcases} 
 \alpha  N + \beta   , & i = 1\\
\beta, & i  \ge 2
\end{dcases}
\end{align}
\end{lemma}

By Lemma~\ref{lem:ones_plust_Id}, each of the $K_{ij}$ matrices has a leading eigenvector that is proportional to the all ones vector. Furthermore, the leading order terms in the eigenvalues are summarised in the Table~\ref{tab:evals} as a function of $N$ and $\mu$. Here, we see that if the mean parameter satisfies $\mu = O( N^{-\beta})$  for a fixed constant  $\beta > 1/8$ then all of the eigenvalues except for the maximum converge to zero as $N \to \infty$. In other words, the GET holds provided that the weights are orthogonal to the all ones vector. 

Evaluating with $\mu^2 = c/\sqrt{N}$  for fixed constant $c$  (or equivalently $\beta = 1/4$) recovers the scalings given in Example~2.


\begin{table}
\centering
\caption{\label{tab:evals} Leading order terms for the eigenvalues of the matrices  in~\eqref{eq:K_deterministic}.  }
\begin{tabular}{l | c c c c} 
& $K_{11}$ &$K_{12}$ &$K_{21}$ &$K_{22}$ \\
\hline
maximum eigenvalue & $\mu^4 N^{3/2}$ & $\mu^6 N^{3/2}  + \mu^4 N^{1/2} $ & $\mu^8 N^2$ & $\mu^{10} N^2  + \mu^8 N$\\
2nd largest eigenvalue&  $\mu^4 N^{-1/2}$  & $\mu^4 N^{1/2}  $ & $\mu^8$ & $\mu^8 N$
\end{tabular}
\end{table}

\subsection{Fixed aspect ratio} 

Next we consider the setting where $D/N  \to  \delta \in (0, \infty)$. Note that $A$ can be expressed as a rank-one perturbation of an $N \times D$ matrix with i.i.d.\ entries. 
In the high dimensional setting $N \to \infty$, the asymptotic distribution of the
singular values and singular vectors are given by~\cite{benaych2012singular}. 
In particular, the maximum eigenvalue satisfies 
\begin{align}
\lambda_1(AA^\top )
& \to \begin{dcases}
 \frac{(1 -\mu^2  +\mu^2 N  )( (1- \mu^2) /\delta + \mu^2 N )}{ \mu^2 N }  ,  &  \frac{\mu^2 N}{ 1- \mu^2}  \ge \delta^{-1/2}   \\
 (1- \mu^2)   \left(1 + \sqrt{1/\delta} \right)^2  , & \text{otherwise} 
\end{dcases},
\end{align}
and the asymptotic empirical distribution of the remaining eigenvalues converges almost surely to the Marchenko-Pastur distribution. Based on these results, the leading order terms in the first and second eigenvalues of $K_{11}$ satisfy the following bounds almost surely: 
\begin{align}
\lambda_1(K_{11}) &= 
O\left(  [ \delta^{-1} + \delta^{-2} ] N^{-1/2} + \mu^4 N^{3/2} \right)\\
\lambda_2(K_{11}) &= O\left(  [ \delta^{-2} + \delta^{-1}  + \mu^4 ] N^{-1/2}  \right). 
\end{align}
Notice that the $\delta \to \infty$ limit of these conditions recovers the scaling given in Table~\ref{tab:evals}.




The scaling behaviour of the matrices $K_{12}, K_{21}$, and $K_{22}$ is more difficult to characterise theoretically because these matrices involve the Hadamard product of random matrices. In the following section we explore their behaviour
numerically. For fixed $\delta$ and $\mu = O(N^{-\beta})$ we make the following observations:
\begin{itemize}
 
\item Fig.~\ref{fig:scaling_K_1} shows the empirical scaling of the eigenvalues for the case $\mu = 0$ (which corresponds to Example~1) and $\mu = O(1/\sqrt{n})$. In both cases, we see that all of the eigenvalues converge to zero  expect for the maximum eigenvalue of $K_{21}$ which is order one. Moreover, the rate of convergence appears to be the same for these two cases. 
\item Fig.~\ref{fig:scaling_K_2} shows the empirical scaling of the eigenvalues  for $\beta \in  \{1/5,1/6,1/7,1/8\}$. For $\beta \ge 1/6$ the second largest eigenvalues of all matrices appear to be decreasing with 
$N$. However, for $\beta = 1/7$ the eigenvalues in $K_{12}$ do not appear to be decreasing (at least for the scale of $N$ shown) and this suggests that the conditions on $\mu$ needed to ensure convergence are more stringent then in the deterministic setting ($\delta \to \infty$) for which the condition $\beta > 1/8$ is sufficient. 
\end{itemize}

\begin{figure}[t!]
  \includegraphics[width= .5 \textwidth]{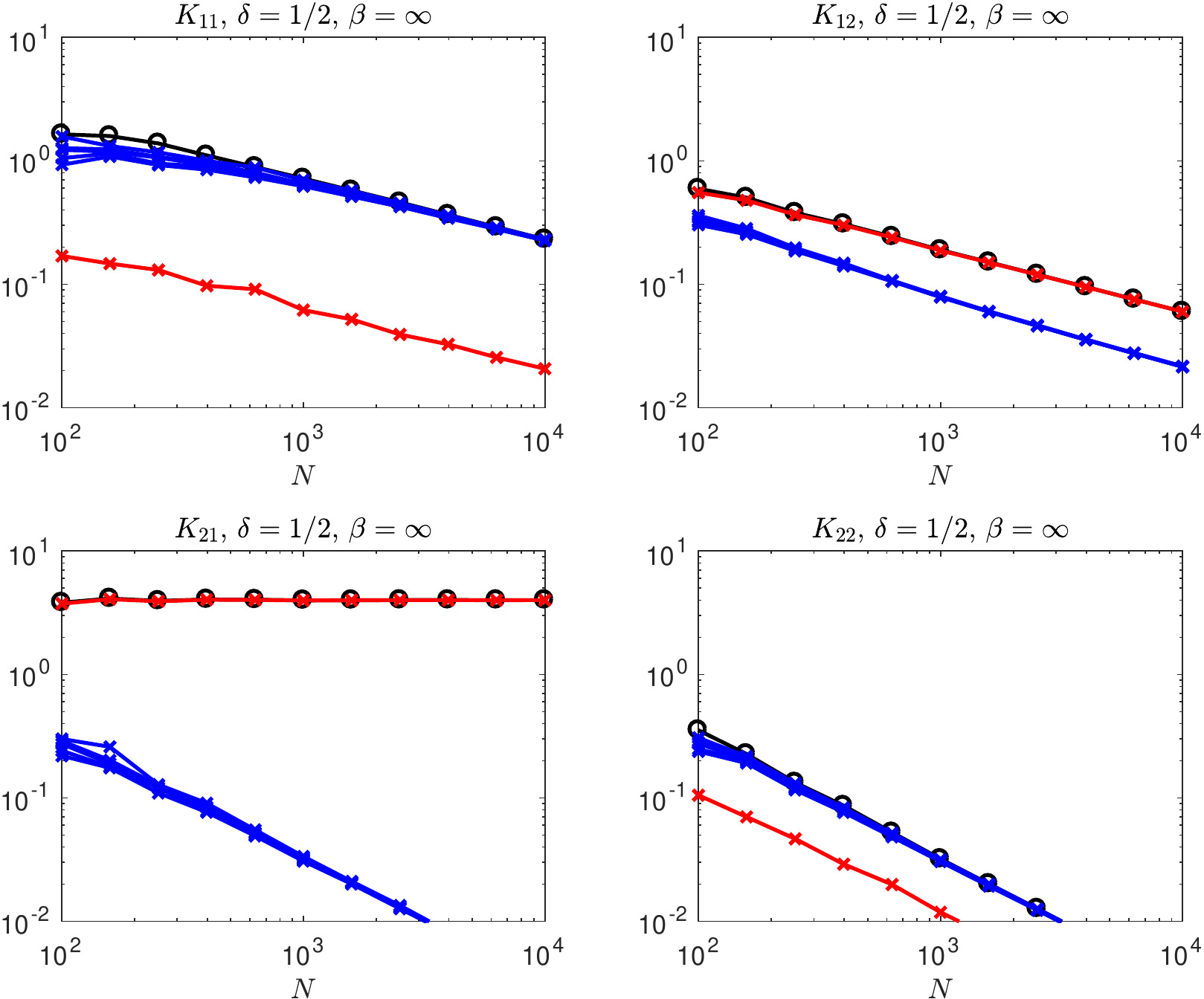}
  \includegraphics[width= .5 \textwidth]{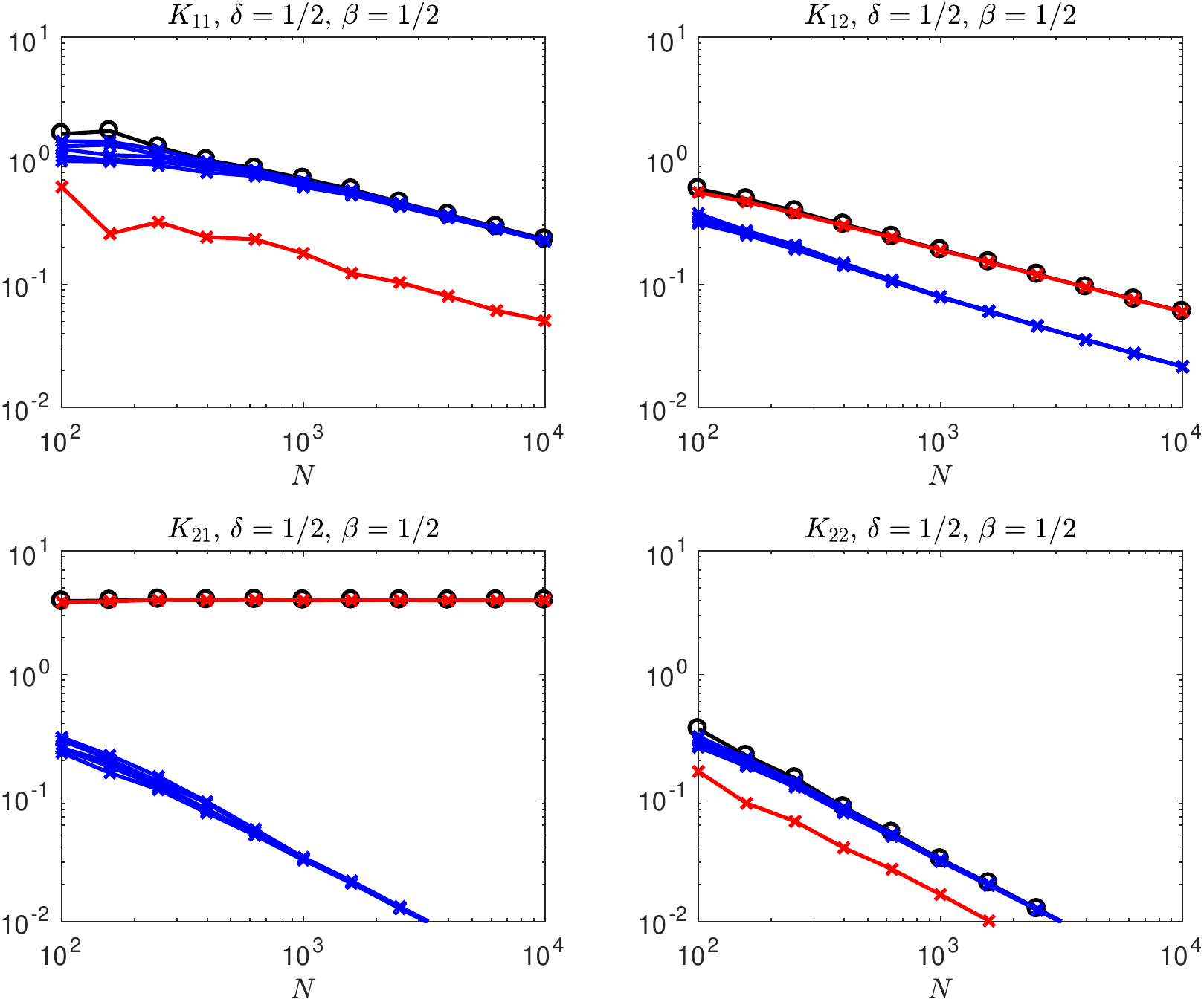}
  \caption{\label{fig:scaling_K_1}Scaling of eigenvalues in the
      $K$ matrices. The black line is maximum eigenvalue $\lambda_1$.  Blue
      lines are $\lambda_i$ for $i \in \{2,6\}$. The red line is the correlation
      with the all ones matrix; when this value is close to the maximum
      eigenvalue it means the leading eigenvector is close to the all ones
      vectors. The left panel is the case $\mu =0$ and the right panel is the
      case $\mu = O( N^{-1/2})$.}
\end{figure}

\begin{figure}[t!]
  \includegraphics[width= .5\textwidth]{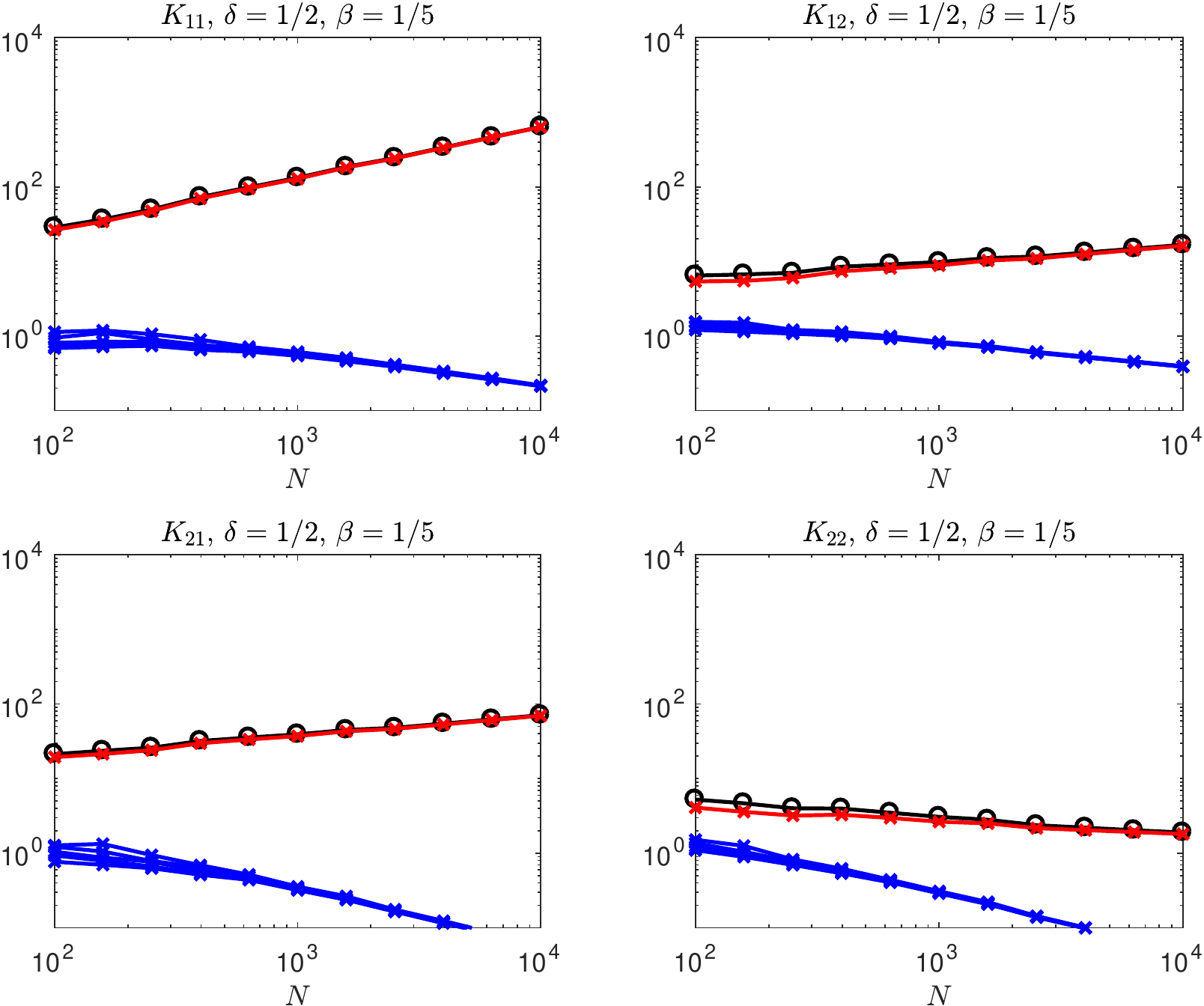}
  \includegraphics[width= .5\textwidth]{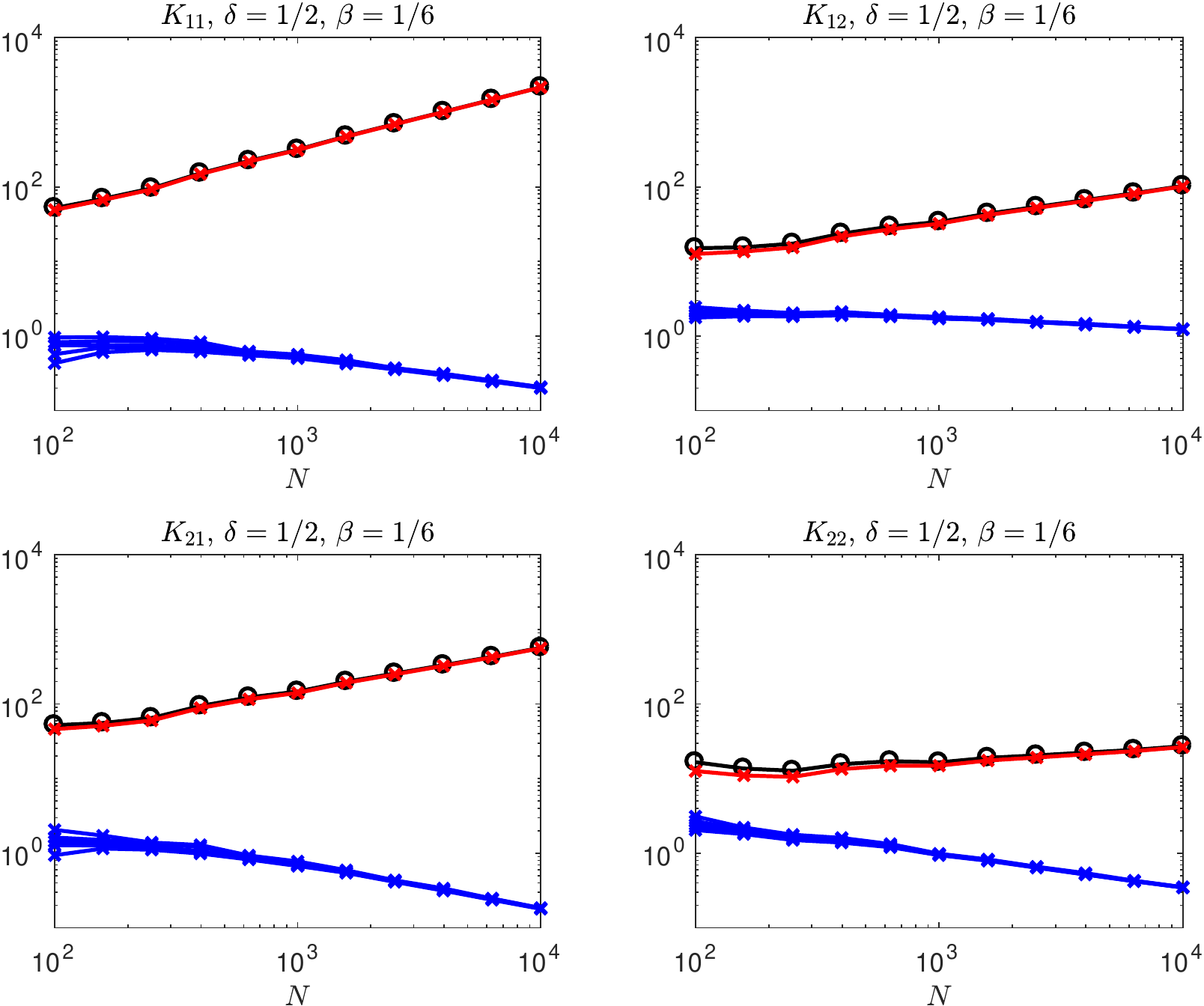}
  \includegraphics[width= .5\textwidth]{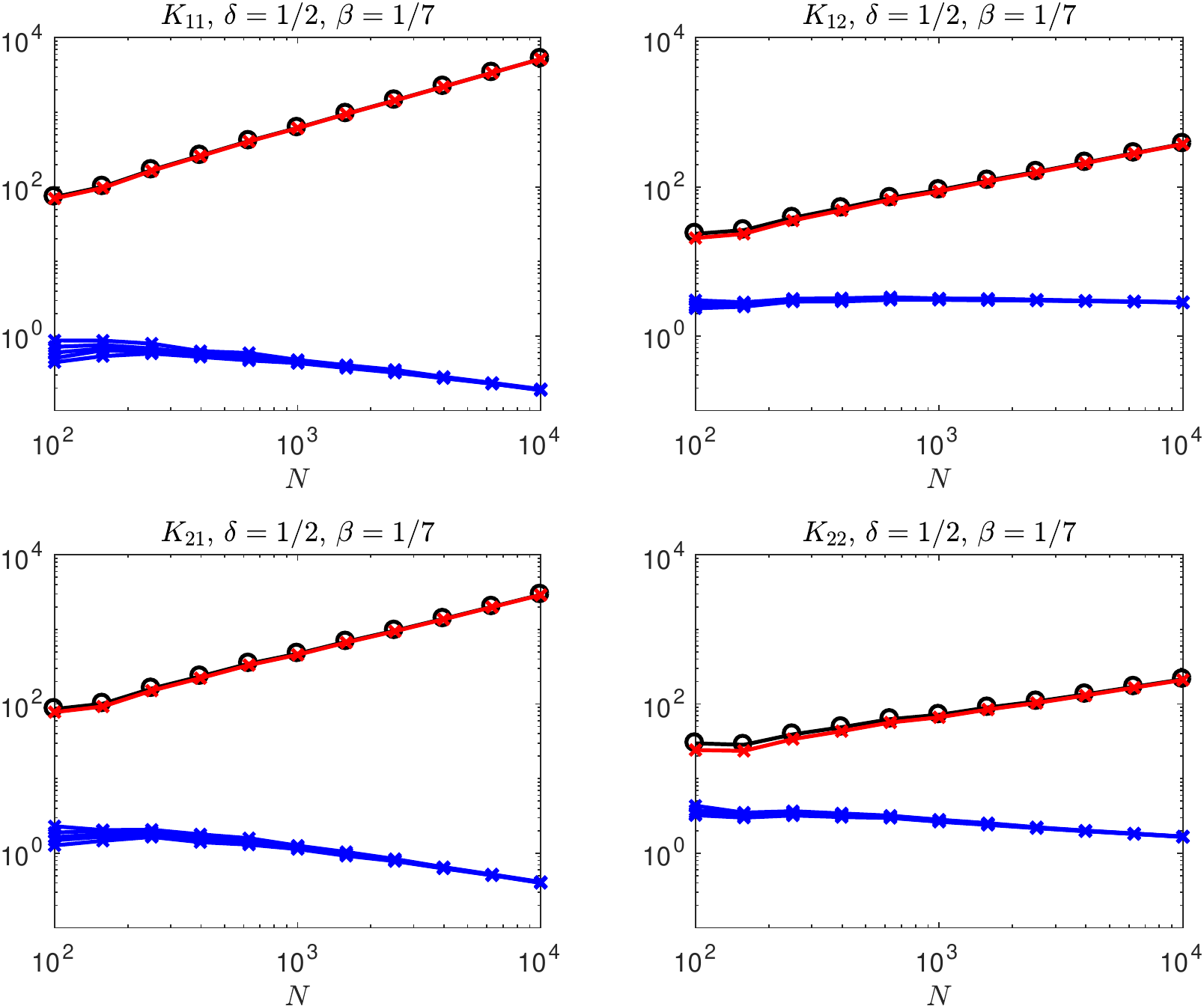}
  \includegraphics[width= .5\textwidth]{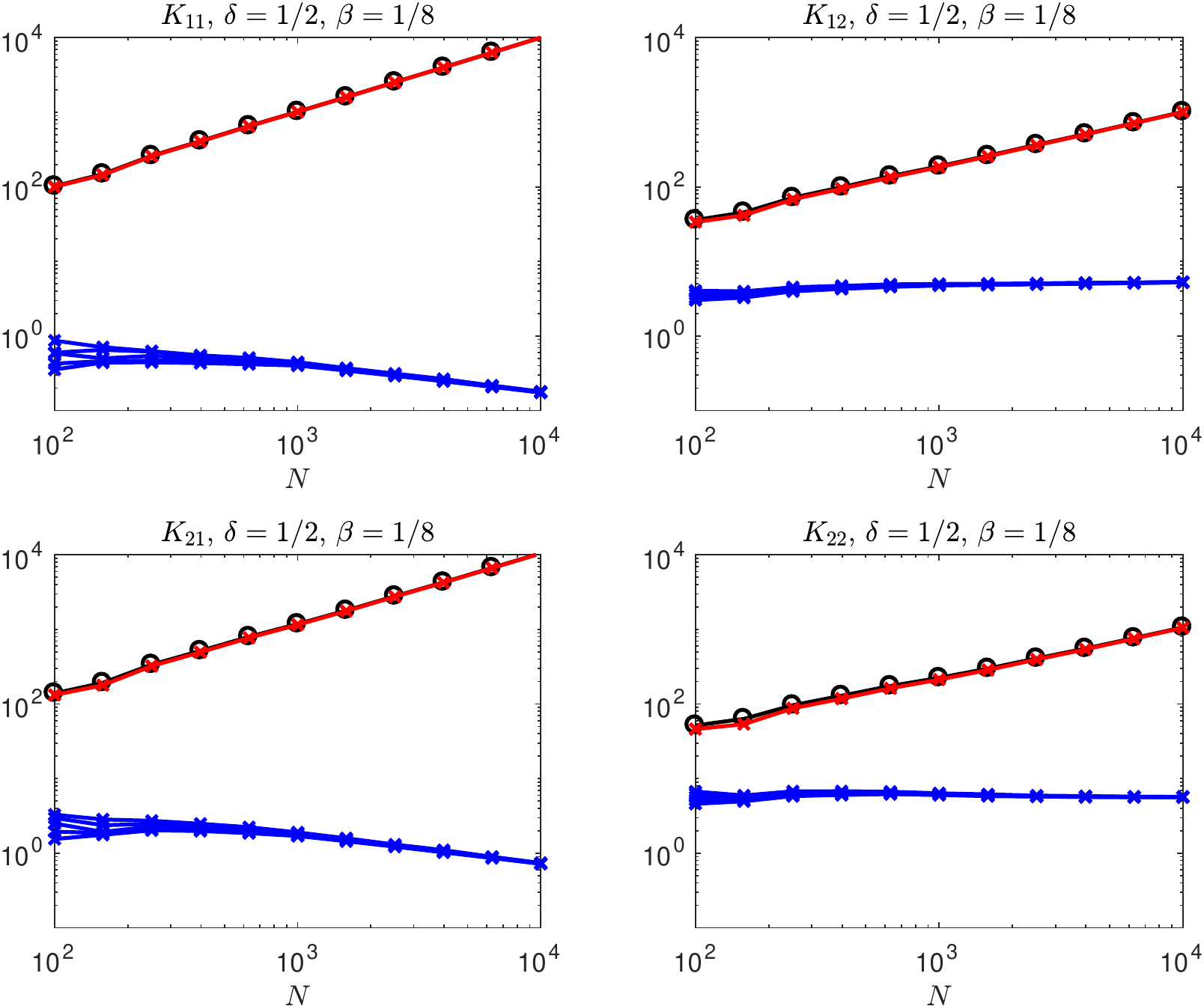}
  \caption{\label{fig:scaling_K_2}Same as in Figure~\ref{fig:scaling_K_1} but
    with $\nu = O(N^{-\beta})$ for $\beta \in \{1/5,1/6,1/7,1/8\}$.}
\end{figure}

\section{Derivation of the equations of motion of Sec.~\ref{sec:odes}}
\label{sec:ode-derivation}
Here we give a detailed derivation of the equations of motion that describe the
dynamics of the two-layer neural net studied in Sec.~\ref{sec:odes}. We
refer to this section for a detailed description of the setup. The GEP allows us
to express the prediction mean-squared error $\pmse$ as a function of the
second-layer weights $v$ and $\tilde v$ as well as the second moments of
$(\lambda, \nu)$, which we can write in terms of the covariance matrices
$\Omega_{ij} = \EE x_i x_j$ and $\Phi_{ir} = \EE x_i c_r $ as
\begin{equation}
  \label{eq:order-params-explicit}
  \begin{gathered}
    Q^{k\ell} \equiv \EE \lambda^k \lambda^\ell = \usN \sum_{i, j}^N w_i^k
    \Omega_{ij} w_j^k, \qquad R^{km} \equiv \EE \lambda^k \nu^m= \ussdelta \usN
    \sum_{i, r}
    w^k_i \Phi_{ir} \tilde w^m_r    \\
    T^{mn} \equiv \EE \nu^m \nu^n = \frac{1}{D} \sum_{r,s}^D \tilde w^m_r \tilde w_r^n.
  \end{gathered}
\end{equation}
We will adopt the notational convention for tensors such as $Q^{k \ell}$ that
extensive indices (taking values up to $D$, $N$) are below the line, while we'll
use upper indices when they take a finite number of values up to $M$ or $K$. The
challenge of controlling the learning in the thermodynamic limit will be to
write closed equations using matrices with only ``upper'' indices left. Finally,
we will adopt the convention that the indices $j,k,\ell,\iota=1,\ldots,K$ always
denote \emph{student} nodes, while $n,m=1,\ldots,M$ are reserved for teacher
hidden nodes.

\paragraph{Rotating the dynamics}

The first step in the derivation is to rotate the order parameters into the
basis given by the eigen-decomposition of the covariance matrix
with eigenvalues $\rho_\tau$ and eigenvectors $\psi_\tau$ that are normalised as
$\sum_\tau \psi_{\tau i}\psi_{\tau j} = N \delta_{ij}$ and
$\sum_i \psi_{\tau i}\psi_{\tau' i} = N \delta_{\tau \tau'}$.  We can then
re-write the ``teacher-student overlap'' $R$~\eqref{eq:order-params-explicit} as
\begin{equation}
  R^{km} =\frac{1}{\sqrt{\delta}N}\sum_\tau \Gamma^k_\tau \tilde{\Gamma}^m_\tau
\end{equation}
where we have introduced the student and teacher projections
\begin{equation}
  \label{eq:Gammas}
  \Gamma_\tau^k \equiv \ussN \sum_i \psi_{\tau i} w_i^k, \qquad
  \tilde{\Gamma}_\tau^m \equiv \ussN \sum_i \psi_{\tau i} \tilde{\omega}^m_i, \qquad
  \tilde \omega^m_i \equiv \sum_r \Phi_{ir} \tilde w^m_r.
\end{equation}
Note the normalisation (or lack thereof); this is due the fact that
$\Phi_{ir} = \EE x_i c_r \sim O(1 / \sqrt{N})$. The student-student overlap
becomes likewise
\begin{equation}
  Q^{kl}=\usN \sum_\tau  \rho_\tau \Gamma^k_\tau \Gamma^\ell_\tau,\label{eq:Qsum}
\end{equation}
and we also introduce a new teacher-teacher overlap, which is given by
\begin{equation}
  \label{eq:tildeT}
  \tilde T^{nm} =\usN \sum_\tau  \tilde{\Gamma}^n_\tau \tilde{\Gamma}^m_\tau =
  \usN \sum_i \sum_{r,s} \tilde w^n_r \Phi_{ir} \Phi_{is} \tilde w^m_s
\end{equation}
This order parameter can be interpreted as a teacher-teacher overlap with the
teacher weights ``rotated'' by~$\left[\Phi^\top \Phi\right]_{rs}$. This
is a key observation: having the teacher act on the latent variables means that
instead of having the actual teacher-teacher overlap, the student also sees a
rotated version, rendering perfect learning impossible.

\paragraph{Teacher-student overlap}

To analyse quantities that are linear in the weights, such as the
teacher-student overlap $R^{km}$, we have to analyse the SGD update
\begin{equation}
  \dd \Gamma^k_\tau =-\frac{\eta}{\sqrt{N}}v^k
  \left[\sum_{j\neq k}^K v^j \mathcal{A}^{jk}_\tau + v^k \mathcal{B}^k_{\tau}-\sum_n^M \tilde v^n \mathcal{C}^{nk}_\tau\right].
\end{equation}
We will use $\dd$~to denote the change in time-dependent quantities during one
step of SGD. We have defined the following averages
\begin{equation}
  \label{eq:avgABC}
  \mathcal{A}^{jk}_\tau= \EE g(\lambda^j)g'(\lambda^k)\beta_{\tau},\qquad
  \mathcal{B}^k_\tau=\EE g(\lambda^k)g'(\lambda^k)\beta_{\tau},\qquad
  \mathcal{C}^{nk}_\tau=\EE \tilde{g}(\nu^n)g'(\lambda^k)\beta_{\tau}.
\end{equation}
where we have introduced the projected input
\begin{equation}
  \label{eq:beta}
  \beta_\tau \equiv \ussN \sum_i \psi_{\tau i} x_i.
\end{equation}
As we discussed in the main text, there are now two crucial facts that make
computing these averages possible. The online assumption asserts that
at each step $\sindex$ of SGD, the input $x_\sindex$ used to evaluate the gradient is
generated from a previously unused latent vector $c_\sindex$, which is uncorrelated
to the students weights at that time. We also \emph{assume} that the $K + M$
variables $\{\lambda^k, \nu^m\}$ are jointly Gaussian, making it possible to
express the averages over $\{\lambda^k, \nu^m\}$ in terms of only their
covariances, and hence later to close the equations. For the special-case of a
single-layer generative network, Theorem~\ref{thm:get} gives us verifiable
conditions on the weights of the generator under which this holds. Using a
simple Lemma~\ref{lemma:1} to evaluate the averages~\eqref{eq:avgABC} yields
\begin{align}
  \begin{split}
    \mathcal{A}^{jk}_\tau &=\frac{1}{Q^{kk}Q^{j
        j}-(Q^{kj})^2}  \left( Q^{jj} \EE\left[ g'(\lambda^k)\lambda^k
    g(\lambda^j) \right] \;
    \EE\left[ \lambda^k \beta_\tau\right] -Q^{kj} \EE\left[ g'(\lambda^k)\lambda^j g(\lambda^j) \right] \;
    \EE\left[ \lambda^k \beta_\tau\right] \right.  \\
     & \hspace*{10em}\left. -Q^{kj} \EE\left[ g'(\lambda^k) \lambda^k g(\lambda^j) \right] \;
     \EE\left[ \lambda^j \beta_\tau\right]  +Q^{kk} \EE\left[ g'(\lambda^k) \lambda^j g(\lambda^j) \right] \;
     \EE\left[ \lambda^j \beta_\tau\right] \right),
   \end{split}
\end{align}
and similarly for $\mathcal{B}^k_{\tau}$ and $\mathcal{C}^{nk}_\tau$. At this
point, it is convenient to introduce a short-hand notation for the
three-dimensional Gaussian averages
\begin{equation}
  \label{eq:I3}
  I_3(k, j, n) \equiv \EE\left[ g'(\lambda^k) \lambda^j
    \tilde{g}(\nu^n) \right] ,
\end{equation}
which was introduced by~\citet{Saad1995a}. Arguments passed to
$I_3$ should be translated into local fields on the right-hand side by using the
convention where the indices $j,k,\ell,\iota$ always refer to student local
fields $\lambda^j$, etc., while the indices $n,m$ always refer to teacher local
fields $\nu^n$, $\nu^m$. Similarly,
$I_3(k, j, j) \equiv \EE\left[ g'(\lambda^k) \lambda^j g(\lambda^j) \right]$,
where having the index $j$ as the third argument means that the third factor is
$g(\lambda^j)$, rather than $\tilde{g}(\nu^m)$ in Eq.~\eqref{eq:I3}. The average
in Eq.~\eqref{eq:I3} is taken over a three-dimensional normal distribution with
mean zero and covariance matrix
\begin{equation}
  \label{eq:Phi3}
  \Phi^{(3)}(k, j, n) = \begin{pmatrix}
    Q^{kk} &  Q^{kj} & R^{kn} \\
    Q^{kj} &  Q^{jj} & R^{jn} \\

    R^{kn} &  R^{jn} & T^{nn}
  \end{pmatrix}.
\end{equation}
There are now two types of averages remaining. We first have
$\EE \lambda^k \beta_\tau = \nicefrac{1}{\sqrt{N}}\rho_\tau \Gamma_\tau^k$, and,
likewise,
$\EE\nu^n \beta_\tau = \nicefrac{1}{\sqrt{\delta N}} \; \tilde \Gamma^n_\tau$.
Putting everything together, we can write down the evolution of $\Gamma_\tau^k$
and identify the equations $h_{(1)}^{kj}$ etc. We have
\begin{equation}
  \label{eq:eom-Gamma}
  \begin{split}
    \dd \Gamma_\tau^k =-\frac{\eta}{N} v^k & \left( \rho_\tau \sum_{j\neq k}
      \left[ \Gamma_\tau^k v^j h_{(1)}^{kj}(Q) + v^j \Gamma_\tau^j
        h_{(2)}^{kj}(Q) \right]
      +  \rho_\tau v^k \Gamma_\tau^k h_{(3)}^k(Q) \right .\\
    &\qquad -\left. \sum_n \left[ \rho_\tau \tilde v^n \Gamma_\tau^k
        h_{(4)}^{kn}(Q,R,T) + \ussdelta \tilde v^n \tilde{\Gamma}_\tau^n
        h_{(5)}^{kn}(Q, R, T) \right] \right)
  \end{split}
\end{equation}
where we have introduced the auxiliary functions
$ h_{(3)}^k = I_3(k, k,k) / Q^{kk}$ and
\begin{subequations}
  \label{eq:h}
  \begin{alignat}{2}
    h_{(1)}^{kj} &= \frac{ Q^{jj} I_3(k, k,j) -Q^{kj} I_3(k, j, j)}
                   {Q^{kk}Q^{jj}-(Q^{kj})^2} &\qquad 
    h_{(2)}^{kj} &= \frac{ Q^{kk} I_3(k,j,j) -
                   Q^{kj} I_3(k, k,j)}
                   {Q^{kk}Q^{jj}-(Q^{kj})^2} \\
    h_{(4)}^{kn} &=\frac{ T^{nn} I_3(k, k, n) -R^{kn} I_3(k,n,n)}
                   {Q^{kk}T^{nn}-(R^{kn})^2}  & \qquad
    h_{(5)}^{kn} &= \frac{Q^{kk} I_3(k, n,n) -R^{kn}I_3(k,k,n)}
              {Q^{kk}T^{nn}-(R^{kn})^2}
  \end{alignat}
\end{subequations}

\paragraph{Introducing order parameter densities}

We are now in a position to write down the equation for $R^{km}$ Performing the
sum over $\tau$ in Eq.~\eqref{eq:eom-Gamma}, two types of terms remain. For the
first four terms, we are left with the sum
$\sum_\tau \rho_\tau \Gamma_\tau^k \tilde{\Gamma}_\tau^m$. This term cannot be
reduced to an order parameter in a straightforward way. Instead, we can make
progress by introducing the continuous function:
\begin{equation}
  \label{eq:r}
  r^{km}(\rho)\equiv\frac{1}{\varepsilon_\rho} \usN \sum_\tau \Gamma_\tau^k
  \tilde{\Gamma}_\tau^m \; \ind\left(\rho_\tau \in
    \mathopen[\rho,\rho+\varepsilon_\rho\mathclose[\right),
\end{equation}
where $\ind(\cdot)$ is the indicator function which evaluates to~1 if the
condition given to it as an argument is true, and which otherwise evaluates to
0. We take the limit $\varepsilon_\rho\to 0$ after the thermodynamic limit. Then
we can rewrite the order parameter $R^{km}$ as an integral over the density
$r^{km}$, weighted by the spectral density of the covariance $\Omega_{ij}$:
\begin{equation}
  \label{eq:supp_R_int}
  R^{km}= \ussdelta \int \dd \sindex_\Omega(\rho)\;  r^{km}(\rho).
\end{equation}
For the final term in eq.~\eqref{eq:eom-Gamma}, we introduce the density
\begin{equation}
  \label{eq:t}
  \tilde t^{nm}(\rho)\equiv\frac{1}{\varepsilon_\rho} \usN \sum_\tau \tilde \Gamma_\tau^n
  \tilde{\Gamma}_\tau^m \; \ind\left(\rho_\tau \in
    \mathopen[\rho,\rho+\varepsilon_\rho\mathclose[\right),
\end{equation}
which allows us to write the first equation of motion, which we state in full in
eq.~\eqref{eq:eom-r}.

\paragraph{Student-student overlap}

It is also convenient to re-write the student-student overlap as an integral
\begin{equation}
  \label{eq:supp_Q_int}
  Q^{k\ell} = \int \dd \sindex_\Omega(\rho) \; \rho \;  q^{k\ell}(\rho).
\end{equation}
over a density $q^{kl}(\rho)$ that is defined analogously to $r^{km}(\rho)$,
\begin{equation}
  \label{eq:q}
  q^{k\ell}(\rho) \equiv \frac{1}{\varepsilon_\rho} \usN \sum_\tau \Gamma^k_\tau
  \Gamma^\ell_\tau \; \ind\left(\rho_\tau \in
    \mathopen[\rho,\rho+\varepsilon_\rho\mathclose[\right),
\end{equation}
The part of the time-derivative of $q^{kl}(\rho)$ that is linear in
$\Gamma_\tau$ can be obtained directly from eq.~\eqref{eq:eom-Gamma} as for
$R^{km}$. For the quadratic part, we have to leading order in $N$
\begin{equation}
  \frac{\eta^2}{N} \sum_\tau v^k v^\ell \EE \Delta^2 g'(\lambda^k)
  g'(\lambda^\ell) \beta_\tau^2  = \eta^2\gamma v^k v^j \EE \Delta^2 g'(\lambda^k) g'(\lambda^\ell)
\end{equation}
where we used that $\EE \beta_\tau^2=\rho_\tau$ and we have defined
$\gamma \equiv \sum_\tau \rho_\tau / N$,
which is a constant of the motion. The remaining averages of the type
$\EE \Delta^2 g'(\lambda^k) g'(\lambda^\ell)$ 
can again be expressed succinctly using the shorthands~\cite{Saad1995a}
\begin{equation}
  \label{eq:I4}
  I_4(k, \ell, j, n) \equiv \EE\left[ g'(\lambda^k)
    g'(\lambda^\ell) g(\lambda^j) g(\nu^n)\right].
\end{equation}
that use the same notational conventions as for $I_3$. 
Putting it all together, we obtain the equation of motion~\eqref{eq:eom-q}
where we have introduced a final auxiliary function,
\begin{multline}
  \label{eq:h6}
  h_{(6)}^{k\ell}(Q, R, T, v, \tilde v) = \sum_{j,\iota}^K v^j v^\iota I_4(k, \ell, j, \iota)
  \\- 2 \sum_j^K \sum_m^M v^j \tilde{v}^m I_4(k, \ell, j, m) + \sum_{n,m}^M
  \tilde{v}^n\tilde{v}^m I_4(k, \ell, n, m).
\end{multline}

\paragraph{Second-layer weights}

Finally, we treat each of the second-layer weights of the student $v$ as an
order parameter in its own right. Their equations of motion~\eqref{eq:eom-v} are
readily found from from their SGD update~\eqref{eq:sgd}
and require only the auxiliary funciton
$h_{(7)}^{kn}(Q, R) \equiv \EE\left[ g(\lambda^k) g(\nu^n)\right]$ using the
same convention for the subscript of $h_{(7)}^{kn}$ that we used for the
integrals $I_3$ and $I_4$.

\paragraph{A simple lemma}
\label{sec:simple-lemma}

The derivation of the dynamical equations uses a simple Lemma that we recently
used to analyse single-layer generators~\cite{goldt2019modelling}. To be as
self-contained as possible, we repeat the Lemma here, and refer the interested reader to their paper for the proof.
\begin{lemma}
  \label{lemma:1}
  Suppose you have $T$ random variables $x^1,\dots,x^T$ with jointly Gaussian
  distribution $p(x^1, \ldots, x^T)$. We assume that the distribution has zero
  first moments that the second moments matrix $q^{t t'}$ is positive
  definite. Suppose that an extra random variable $y$ is jointly distributed
  with the $x^1,\ldots,x^T$ and has mean zero, a finite variance
  $\langle y^2\rangle$, and correlations $\langle x^t y\rangle$ which are
  $O(1/\sqrt{N})$.  Then for any two functions $\phi (x^1,\dots,x^T)$ and
  $\psi(y)$ that are odd in each of their arguments, we have, to leading order
  when $N\to \infty$:
\begin{equation}
\langle \phi(x^1,\dots,x^T) \psi(y)\rangle = \sum_{t,s}(q^{-1})^{ts}\;
\frac{\langle x^s y\rangle}{\langle y^2\rangle}\; \langle x^t \phi
(x^1,\dots,x^T)\rangle \; \langle y\psi(y)\rangle
\end{equation}
\end{lemma}

\subsection{Increasing the number of neurons}

The dynamical equations we derived in this section are valid for any finite
$M,K$ after letting $N\to\infty$. For the simulations, it is thus natural to ask
up to which number of neurons the equations accurately predict the dynamics for
fixed $N$. We tested the accuracy of the equations by focusing on the
single-layer generator~\eqref{eq:proof-generators} with $D=500, N=1000$. In this
case, the Gaussian Equivalence holds rigorously thanks to Theorem~\ref{thm:get},
so as we increase $M, K$, we can expect deviations between theoretical
predictions from the dynamical equations and simulations to arise only due to
problems with the equations, rather than problems with Conjecture 1. We show the
results of such an experiment in Fig.~\ref{fig:varyingK}.

\begin{figure}[t!]
  \centering
  \includegraphics[width=.66\textwidth]{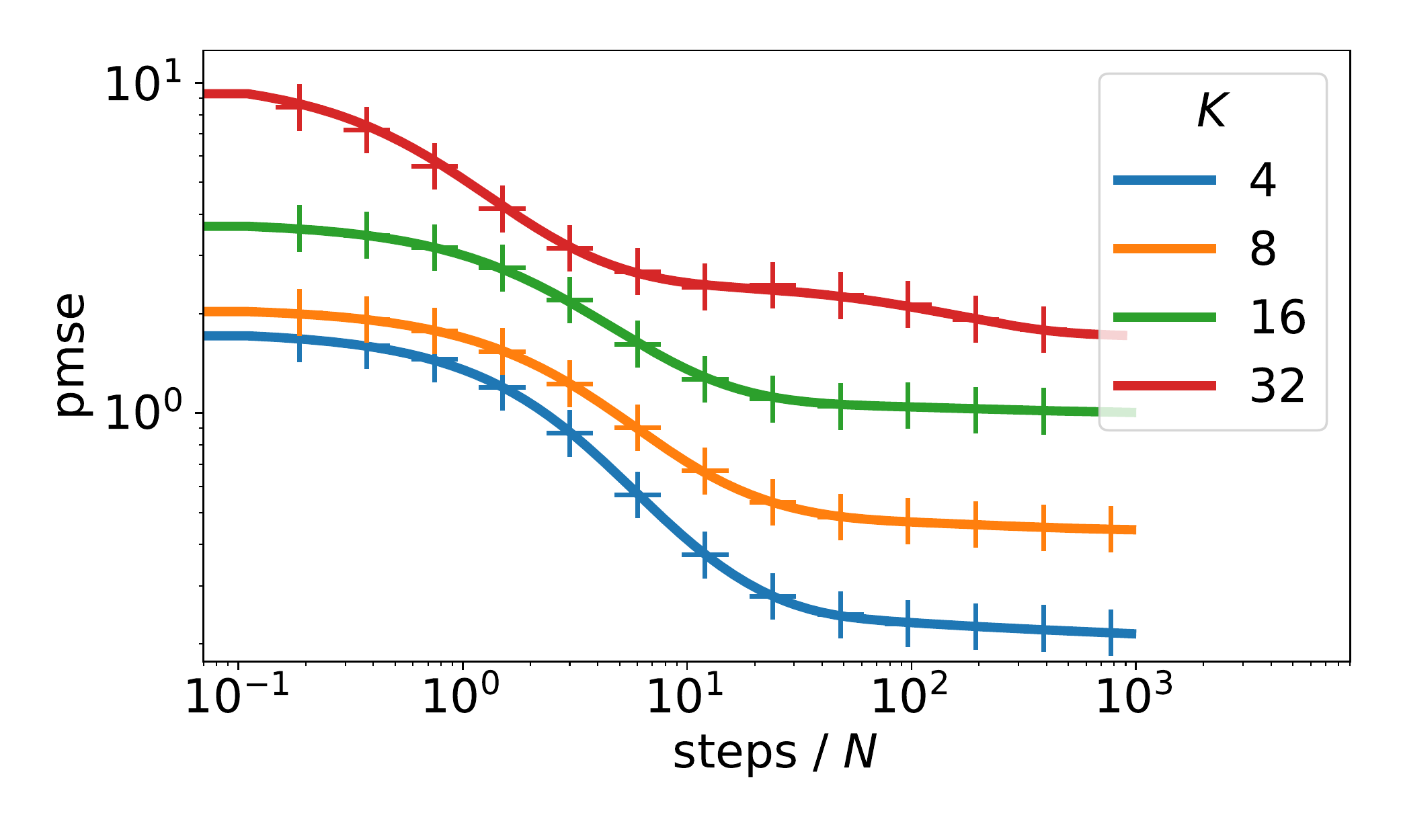}
  \caption{\label{fig:varyingK}\textbf{Theory vs experiments
        for online SGD with increasingly large students}.  We trained students
      with $K$ hidden neurons on teachers with $M=K$ neurons with inputs coming
      from a single-layer generator~\eqref{eq:proof-generators} with random
      weights.
      $D=500, N=1000, \tilde v^m=1, \eta=0.05, g(x) = \tilde g(x) =
      \erf(x/\sqrt{2})$, integration time step $\dd t=0.01$.}
\end{figure}

\section{Replica analysis}
\label{sec:app:replicas}
In this Appendix we give the main steps in the replica derivation of the result
in Section~\ref{sec:replicas} for the full-batch learning. Our analysis,
however, is restricted to the $K=M=1$ case.

\paragraph{Setting: } Consider the supervised learning problem introduced in
Section~\ref{sec:intro} with $K=M=1$. In this case, the model
$y=\phi_{\theta}(\vec{x})$ is simply a \emph{generalised linear model} with
parameter $\vec{w}\in\mathbb{R}^{N}$:
\begin{align}
    \hat{y} = \phi_{\theta}(\vec{x}) = g\left(\frac{1}{\sqrt{N}}\vec{x}\cdot \vec{w}\right)
\end{align}
Similarly, we assume data in independently sampled $(\vec{x}, y)\sim q$ from the generative model introduced in eq.~\eqref{eq:phi_teacher} with $M=1$, which is equivalent to:
\begin{align}
    y = \phi_{\tilde{\theta}}(\vec{c}) = \tilde{g}\left(\frac{1}{\sqrt{D}}\vec{c}\cdot \tilde{\vec{w}}\right), && \vec{x} = \mathcal{G}(\vec{c}), && \vec{c}\sim \mathcal{N}(\vec{0},\mat{I}_{D})
\end{align}
\noindent where $\mathcal{G}:\mathbb{R}^{D}\to \mathbb{R}^{N}$ is a deep generative network as introduced in eq.~\eqref{eq:generator}, $\vec{c}$ is the latent variable and $\tilde{w}\sim P_{\tilde{w}}$ are a fixed set of weights. Different from the online analysis, here we are interested in characterising the generalisation performance of this model when trained on a batch of $\samples$ independent samples from $q$. Let $\mathcal{D}_{\samples} = \{\vec{x}^{\sindex}, y^{\sindex}\}_{\sindex=1}^{\samples}$ denote this training set. Training will consist on finding the set of weights $\hat{\vec{w}}\in\mathbb{R}^{N}$ that minimise the following empirical risk:
\begin{align}
\hat{\vec{w}} = \underset{\vec{w}\in\mathbb{R}^{N}}{\argmin}\left[\sum\limits_{\sindex=1}^{\samples}\ell\left(y^{\sindex},\vec{x}^{\sindex}\cdot\vec{w}\right)+\frac{\lambda}{2}||\vec{w}||^2_{2}\right],
\label{eq:argmin}
\end{align}
\noindent where $\ell$ is a generic loss function and we have added an $\ell_2$ penalty with strength $\lambda>0$. Our aim is to characterise the prediction error on a fresh set of samples $\vec{x},y \sim q$,
\begin{align}
    \epsilon_{g} = \mathbb{E}_{(\vec{x},y)\sim q} \pmse(y,\hat{y}(\vec{x})),
\end{align}
\noindent in the high-dimensional limit where $N,P, D\to\infty$ while the ratios $\alpha = \samples/N$ (the sample complexity) and $\gamma = D/N$ (the compression rate) remain fixed. The key observation in our analysis is that precisely in this limit the asymptotic generalisation error can be fully characterised by only three scalar parameters $(\rho, m^{\star}, q^{\star})$. Indeed, the \emph{Gaussian Equivalence Property} (GEP) introduced in Section~\ref{sec:get-proof} allow us to write
\begin{align}
    \lim\limits_{N\to\infty}\epsilon_g = \mathbb{E}_{\nu,\lambda}\left(\tilde{g}(\nu)-g(\lambda)\right)^2
    \label{eq:app:generror}
\end{align}
\noindent where $(\nu,\lambda)\sim\mathcal{N}(0,\Sigma)$ are jointly Gaussian random variables with covariance $\Sigma = \begin{pmatrix}
    \rho & m^{\star}\\ m^{\star} & q^{\star}\end{pmatrix}$ given by:
\begin{align}
    \rho = \frac{1}{D}||\tilde{\vec{w}}||^2_{2}, && m^{\star} = \frac{1}{\sqrt{ND}}\hat{\vec{w}}^{\top}\Phi \tilde{\vec{w}}, && q^{\star} = \frac{1}{N}\hat{\vec{w}}^{\top}\Omega \hat{\vec{w}}
\end{align}
\noindent with $\Phi = \mathbb{E}_{\vec{c}} \vec{x}\vec{c}^{\top} \in\mathbb{R}^{N\times D}$ and $\Omega = \mathbb{E}_{\vec{c}} \vec{x}\vec{x}^{\top} \in\mathbb{R}^{N\times N}$ being the \emph{exact} covariances of the data. Note that $\rho$ is completely fixed by $P_{\tilde{\vec{w}}}$. The replica analysis will give us $(m^{\star}, q^{\star})$.
\subsection{Replica analysis}
The first step in the replica analysis is to define the following Gibbs measure over $\mathbb{R}^{N}$:
\begin{align}
\mu_{\beta}(\vec{w}) = \frac{1}{\mathcal{Z}_{\beta}} e^{-\beta\left[\sum\limits_{\sindex=1}^{\samples}\ell\left(y^{\sindex},\vec{x}^{\sindex}\cdot \vec{w}\right)+\frac{\lambda}{2}\sum\limits_{i=1}^{N}w_{i}^2\right]}	= \frac{1}{\mathcal{Z}_{\beta}}\underbrace{\prod\limits_{\sindex=1}^{\samples} e^{-\beta\sum\limits_{\sindex=1}^{\samples}\ell\left(y^{\sindex},\vec{x}^{\sindex}\cdot \vec{w}\right)}}_{P_{y}}\underbrace{\prod\limits_{i=1}^{N}e^{-\frac{\beta\lambda}{2}w_{i}^2}}_{P_{w}}\label{eq:app:gibbs}
\end{align}
\noindent where the normalisation $\mathcal{Z}_{\beta}$ is known as the \emph{partition function}, and is a function of the training data $\mathcal{D}$. The factorised densities $P_{y}$ and $P_{w}$ can be interpreted as a (unormalised) likelihood and prior distribution respectively. Note that if we knew how to sample from $\sindex_{\beta}$, we would be able to solve eq.~\eqref{eq:argmin}, since in the limit $\beta\to\infty$, the measure $\sindex_{\beta}$ concentrates around solutions of this minimisation problem. The replica analysis consists in computing the \emph{averaged free energy density}
\begin{align}
\beta f_{\beta} = \lim\limits_{N\to\infty}\frac{1}{N}\mathbb{E}_{\mathcal{D}}\log\mathcal{Z}_{\beta}	
\end{align}
\noindent with the replica trick:
\begin{align}
	\log\mathcal{Z}_{\beta}	= \lim\limits_{r\to 0^{+}}\frac{1}{r}\partial_{r}\mathcal{Z}^{r}_{\beta}.
\end{align}
Linearising the logarithm allow us to average $\mathcal{Z}^{r}_{\beta}$ over the dataset explicitly. As we will see, once this average is taken, $\mathcal{Z}^{r}_{\beta}$ which is a priori a high-dimensional object (defined in terms of integrals in $\mathbb{R}^{N}$) factorise into a simple scalar quantities that will give us access to $(m^{\star}, q^{\star})$.
\paragraph{Averaging over the data set: } The average over the replicated partition function is explicitly given by:
\begin{align}
\mathbb{E}_{\mathcal{D}}\mathcal{Z}^{r}_{\beta} &=\prod\limits_{\sindex=1}^{\samples}\int\dd y^{\sindex}\int_{\mathbb{R}^{D}}\dd\tilde{\vec{w}}~P_{\tilde{w}}(\tilde{\vec{w}})	\int_{\mathbb{R}^{N\times r}}\left(\prod\limits_{a=1}^{r}\dd\vec{w}^{a}~P_{w}(\vec{w}^{a})\right)\times\notag\\
&\qquad\times\underbrace{\mathbb{E}_{\vec{c}^{\sindex}}\left[\tilde{P}_{y}\left(y^{\sindex}\Big|\frac{\vec{c}^{\sindex}\cdot\tilde{\vec{w}}}{\sqrt{D}}\right)\prod\limits_{a=1}^{r}P_{y}\left(y^{\sindex}\Big|\frac{\vec{x}^{\sindex}\cdot \vec{w}^{a}}{\sqrt{N}}\right)\right]}_{(\star)}\notag
\end{align}
Note that since $\vec{x}^{\sindex}=\mathcal{G}(\vec{c}^{\sindex})$ the average in $(\star)$ defines the joint probability between the random variables $\nu_{\sindex}=\frac{\vec{c}^{\sindex}\cdot\tilde{\vec{w}}}{\sqrt{D}}$ and $\lambda^{a}_{\sindex}=\frac{\vec{x}^{\sindex}\cdot\vec{w}^{a}}{\sqrt{N}}$. The \emph{Gaussian Equivalence Principle} states that for certain architectures $\mathcal{G}$, the random variables $(\nu_{\sindex},\lambda^{a}_{\sindex})$ are asymptotically jointly Gaussian, with zero mean and covariance matrix given by:
\begin{align}
\Sigma^{ab} = 
\begin{pmatrix}
 	\rho & m^{a}\\
 	m^{a} & Q^{ab}
\end{pmatrix}.
\end{align}
\noindent where the so-called overlap parameters $(\rho, m^{a}, Q^{ab})$ are related to the weights $\tilde{\vec{w}}, \vec{w}$:
\begin{align}
\rho &\equiv \mathbb{E}\left[\nu_{\sindex}^2\right] = \frac{1}{D}||\tilde{\vec{w}}||^2_{2}, && m^{a} \equiv \mathbb{E}\left[\lambda_{\sindex}^{a}\nu_{\sindex}\right]= \frac{1}{\sqrt{ND}}{\vec{w}^{a}}^{\top}\Phi\tilde{\vec{w}}, && Q^{ab} \equiv \mathbb{E}\left[\lambda_{\sindex}^{a}\lambda_{\sindex}^{b}\right]= \frac{1}{N}{\vec{w}^{a}}^{\top}\Omega \vec{w}^{b}	\notag
\end{align}
\noindent where all the information about the architecture of the generative network $\vec{x}=\mathcal{G}(\vec{c})$ is contained in the covariance matrices $\Omega = \mathbb{E}_{\vec{c}}\left[\vec{x}\vec{x}^{\top}\right]\mathbb{R}^{N\times N}$ and $\Phi = \mathbb{E}_{\vec{c}}\left[\vec{x}\vec{c}^{\top}\right]\in\mathbb{R}^{N\times D}$. We can therefore write the averaged replicated partition function as:
\begin{align}
\mathbb{E}_{\mathcal{D}}\mathcal{Z}^{r}_{\beta} &=\prod\limits_{\sindex=1}^{\samples}\int\dd y^{\sindex}\int_{\mathbb{R}^{D}}\dd\tilde{\vec{w}}~P_{\tilde{w}}(\tilde{\vec{w}})	\int_{\mathbb{R}^{N\times r}}\left(\prod\limits_{a=1}^{r}\dd\vec{w}^{a}~P_{w}(\vec{w}^{a})\right)\mathcal{N}(\nu_{\sindex}, \lambda^{a}_{\sindex};\vec{0},\Sigma^{ab})
\label{eq:avgZr:2}
\end{align}

\paragraph{Rewriting as a saddle-point problem: } The next step is to free the overlap parameters by introducing delta functions $\delta\left(D\rho - ||\tilde{\vec{w}}||^2_{2}\right)$, $\delta\left(\sqrt{ND} m^{a}-\vec{w}^{a}\Phi\tilde{\vec{w}}\right)$, $\delta\left(N Q^{ab}-{\vec{w}^{a}}^{\top}\Omega\vec{w}^{b}\right)$. Inserting in  eq.~\eqref{eq:avgZr:2}, swapping the integrals and going to Fourier space allow us to rewrite:
\begin{align}
\mathbb{E}_{\mathcal{D}}\mathcal{Z}_{\beta}^{r} = 	\int_{\mathbb{R}}\frac{\dd\rho\dd\hat{\rho}}{2\pi}\int_{\mathbb{R}^{r}}\prod\limits_{a=1}^{r} \frac{\dd m^{a}\dd\hat{m}^{a}}{2\pi}\int_{\mathbb{R}^{r\times r}}\prod\limits_{1\leq a\leq b\leq r}\frac{\dd Q^{ab}\dd\hat{Q}^{ab}}{2\pi} e^{D\Phi^{(r)}}
\label{eq:avgZr:3}
\end{align}
\noindent where we have absorbed a $-i$ factor in the integrals\footnote{This won't matter since we will be only interested in the saddle-point of the integrals.} and defined the potential:
\begin{align}
\Phi^{(r)} = -\gamma\rho\hat{\rho}	-\sqrt{\gamma}\sum\limits_{a=1}^{r}m^{a}\hat{m}^{a}-\sum\limits_{1\leq a\leq b\leq r}Q^{ab}\hat{Q}^{ab}+\alpha\Psi^{(r)}_{y}(\rho,m^{a}, Q^{ab}) + \Psi^{(r)}_{w}(\hat{\rho},\hat{m}^{a},\hat{Q}^{ab})\notag
\end{align}
\noindent with $\alpha = \samples/N$, $\gamma = D/N$ and:
\begin{align}
\Psi_{w}^{(r)} &= \frac{1}{N}\log\int_{\mathbb{R}^{D}}\dd\tilde{\vec{w}}P_{w^{\star}}\left(\tilde{\vec{w}}\right)\int_{\mathbb{R}^{N\times r}}\prod\limits_{a=1}^{r}\dd\vec{w}^{a}P_{w}\left(\vec{w}^{a}\right) e^{\hat{\rho}||\tilde{\vec{w}}||^2_{2}+\sum\limits_{a=1}^{r}\hat{m}^{a}{\vec{w}^{a}}^{\top}\Phi\tilde{\vec{w}}+\sum\limits_{1\leq a\leq b\leq r}\hat{Q}^{ab}{\vec{w}^{a}}^{\top}\Omega\vec{w}^{b}}\notag\\
\Psi_{y}^{(r)} &= \log\int_{\mathbb{R}}\dd y\int_{\mathbb{R}}\dd\nu~\tilde{P}_{y}(y|\nu)\int\prod\limits_{a=1}^{r}\dd\lambda^{a}P_{y}(y|\lambda^{a})~ \mathcal{N}(\nu,\lambda^{a};\vec{0},\Sigma^{ab})\label{eq:replica:Psir}
\end{align}
In the high-dimensional limit where $N\to\infty$ while $\alpha = \samples/N$ and $\gamma = D/N$ stay finite, the integral in eq.~\eqref{eq:avgZr:3} concentrate around the values of the overlaps that extremise $\Phi^{(r)}$, and therefore we can write:
\begin{align}
\beta f_{\beta} = -\lim\limits_{r\to 0^{+}}\frac{1}{r}\extr~ \Phi^{(r)}\left(\hat{\rho},\hat{m}^{a},\hat{Q}^{ab};\rho,m^{a},Q^{ab}\right)
\end{align}

\paragraph{Replica symmetric ansatz: } Finding the overlap configuration that minimise $\Phi^{(r)}$ is itself an intractable problem. In order to make progress, we restrict the extremisation above to the following \emph{replica symmetric ansatz}:
\begin{align}
m^{a} = m, && \hat{m}^{a} = \hat{m}, &&\text{ for } a=1,\dots, r	\notag\\
q^{aa} = r, && \hat{q}^{aa} = -\frac{1}{2}\hat{r}, &&\text{ for } a=1,\dots, r	\notag\\
Q^{ab} = q, && \hat{Q}^{ab} = \hat{q}, &&\text{ for } 1\leq a<b\leq r
\end{align}
Inserting this ansatz in eq.~\eqref{eq:replica:Psir} allow us to explicitly take the $r\to 0^{+}$ limit for each term. The first three terms are trivial. The limit of $\Psi_{y}^{(r)}$ is cumbersome, but it common to many replica computations for the generalised linear likelihood $P_{y}$. We refer the curious reader to~\cite{gerace2020generalisation} for more details, and write the end result here:
\begin{align}
\Psi_{y}\equiv\lim\limits_{r\to 0^{+}}\frac{1}{r}\Psi^{(r)}_{w} = \mathbb{E}_{\xi}\left[\int_{\mathbb{R}}\dd y~\tilde{\mathcal{Z}}_{y}\left(y,\frac{m}{\sqrt{q}}\xi, \rho-\frac{m^2}{q}\right)\log\mathcal{Z}_{y}(y,\sqrt{q}\xi,V)\right]
\end{align}
\noindent where $\xi\sim\mathcal{N}(0,1)$, $V=r-q$ and:
\begin{align}
\mathcal{Z}_{y}(y,\omega,V) = \int_{\mathbb{R}}\frac{\dd x}{\sqrt{2\pi V}}e^{-\frac{(x-\omega)^2}{2V}}P_{y}(y|x), && \tilde{\mathcal{Z}}_{y}(y,\omega,V) = \int_{\mathbb{R}}\frac{\dd x}{\sqrt{2\pi V}}e^{-\frac{(x-\omega)^2}{2V}}\tilde{P}_{y}(y|x)	
\end{align}
Note that as in~\cite{gerace2020generalisation}, the consistency condition of the zeroth order term in the free energy fix the parameters $\rho = \mathbb{E}_{P_{\tilde{w}}}\tilde{w}$ and $\hat{\rho} = 0$. The limit of $\Psi^{(r)}_{w}$ is slightly more involved. First, inserting the replica symmetric ansatz allow us to write:
\begin{align}
\Psi_{w}^{(r)} &= \frac{1}{N}\log\int_{\mathbb{R}^{D}}\dd\tilde{\vec{w}}P_{w^{\star}}\left(\tilde{\vec{w}}\right)\int_{\mathbb{R}^{N\times r}}\prod\limits_{a=1}^{r}\dd\vec{w}^{a}P_{w}\left(\vec{w}^{a}\right) e^{-\frac{\hat{V}}{2}\sum\limits_{a=1}^{r}{\vec{w}^{a}}^{\top}\Omega\vec{w}^{a}+\hat{m}\sum\limits_{a=1}^{r}{\vec{w}^{a}}^{\top}\Phi\tilde{\vec{w}}+\hat{q}\sum\limits_{a,b=1}^{r}{\vec{w}^{a}}^{\top}\Omega\vec{w}^{b}}
\end{align}
\noindent where we have defined $\hat{V} = \hat{r}+\hat{q}$. Now using that:
\begin{align}
	e^{\hat{q}\sum\limits_{a,b=1}^{r}{\vec{w}^{a}}^{\top}\Omega\vec{w}^{b}} = \mathbb{E}_{\vec{\xi}}\left[e^{\sqrt{\hat{q}}\vec{\xi}^{\top}\Omega^{1/2}\sum\limits_{a=1}^{r}\vec{w}^{a}}\right]
\end{align}
\noindent for $\vec{\xi}\sim\mathcal{N}(0,\mat{I}_{N})$, we can write:
\begin{align}
\Psi_{w}^{(r)} &= \frac{1}{N}\log\int_{\mathbb{R}^{D}}\dd\tilde{\vec{w}}P_{w^{\star}}\left(\tilde{\vec{w}}\right)\prod\limits_{a=1}^{r}\int_{\mathbb{R}^{N}}\dd\vec{w}^{a}P_{w}\left(\vec{w}^{a}\right) \mathbb{E}_{\vec{\xi}}\left[e^{-\frac{\hat{V}}{2}{\vec{w}^{a}}^{\top}\Omega\vec{w}^{a}+{\vec{w}^{a}}^{\top}\left(\hat{m}\Phi\tilde{\vec{w}}+\hat{q}\Omega^{1/2}\vec{\xi}\right)}\right]\notag\\
&= \frac{1}{N}\log\mathbb{E}_{\vec{\xi}}\int_{\mathbb{R}^{D}}\dd\tilde{\vec{w}}P_{w^{\star}}\left(\tilde{\vec{w}}\right)\left[\int_{\mathbb{R}^{N}}\dd\vec{w}~P_{w}\left(\vec{w}\right) e^{-\frac{\hat{V}}{2}\vec{w}^{\top}\Omega\vec{w}+\vec{w}^{\top}\left(\hat{m}\Phi\tilde{\vec{w}}+\hat{q}\Omega^{1/2}\vec{\xi}\right)}\right]^{r}
\end{align}
\noindent and therefore:
\begin{align}
\Psi_{w}\equiv \lim\limits_{r\to 0^{+}}\frac{1}{r}\Psi_{w}^{(r)} = \frac{1}{N}\mathbb{E}_{\vec{\xi},\tilde{\vec{w}}}\log\int_{\mathbb{R}^{N}}\dd\vec{w}~P_{w}\left(\vec{w}\right) e^{-\frac{\hat{V}}{2}\vec{w}^{\top}\Omega\vec{w}+\vec{w}^{\top}\left(\hat{m}\Phi\tilde{\vec{w}}+\hat{q}\Omega^{1/2}\vec{\xi}\right)}
\end{align}

\paragraph{Summary:} The replica symmetric free energy density is simply given by:
\begin{align}
\beta f_{\beta} = \underset{q,m,\hat{q},\hat{m}}{\extr}~\left\{-\frac{1}{2}r\hat{r}-\frac{1}{2}q\hat{q}+m\hat{m}-\alpha \Psi_{y}(r, m,q) -	 \Psi_{w}(\hat{r}, 	\hat{m},\hat{q})\right\}
\label{eq:freeen:final}
\end{align}
\noindent where
\begin{align}
\Psi_{w} &= \lim\limits_{N\to\infty} \frac{1}{N}\mathbb{E}_{\xi,\tilde{\vec{w}}}\log\int_{\mathbb{R}^{N}}\dd\vec{w}~P_{w}\left(\vec{w}\right) e^{-\frac{\hat{V}}{2}\vec{w}^{\top}\Omega\vec{w}+\vec{w}^{\top}\left(\hat{m}\Phi\tilde{\vec{w}}+\hat{q}\Omega^{1/2}\vec{\xi}\right)}\notag\\
\Psi_{y} &= \mathbb{E}_{\vec{\xi}}\left[\int_{\mathbb{R}}\dd y~\tilde{\mathcal{Z}}_{y}\left(y,\frac{m}{\sqrt{q}}\xi, \rho-\frac{m^2}{q}\right)\log\mathcal{Z}_{y}(y,\sqrt{q}\xi,V)\right]
\end{align}
\noindent and 
\begin{align}
\mathcal{Z}_{y}(y,\omega,V) = \int_{\mathbb{R}}\frac{\dd x}{\sqrt{2\pi V}}e^{-\frac{(x-\omega)^2}{2V}}P_{y}(y|x), && \tilde{\mathcal{Z}}_{y}(y,\omega,V) = \int_{\mathbb{R}}\frac{\dd x}{\sqrt{2\pi V}}e^{-\frac{(x-\omega)^2}{2V}}\tilde{P}_{y}(y|x)	
\end{align}

\paragraph{Simplifying $\Psi_{w}$: } The result summarised above holds for any $P_{w}$ and $P_{\tilde{w}}$, but can be considerably simplified in our case of interest eq.~\eqref{eq:app:gibbs} where these densities are Gaussian. Indeed, we can integrate $\vec{w}$ explicitly in $\Psi_{w}$ to get:
\begin{align}
\int_{\mathbb{R}^{N}}\dd\vec{w}~P_{w}(\vec{w})&e^{-\frac{\hat{V}}{2}\vec{w}^{\top}\Omega\vec{w}+\vec{w}^{\top}\left(\hat{m}\Phi\tilde{\vec{w}}+\sqrt{\hat{q}}\Omega^{1/2}\vec{\xi}\right)} =\int_{\mathbb{R}^{N}}\frac{\dd\vec{w}}{(2\pi)^{p/2}}e^{-\frac{1}{2}\vec{w}^{\top}\left(\beta\lambda\mat{I}_{N}+\hat{V}\Omega\right)\vec{w}+\vec{w}^{\top}\left(\hat{m}\Phi\tilde{\vec{w}}+\sqrt{\hat{q}}\Omega^{1/2}\vec{\xi}\right)}\notag\\
&=\frac{\exp\left(\frac{1}{2}\left(\hat{m}\Phi\tilde{\vec{w}}+\sqrt{\hat{q}}\Omega^{1/2}\vec{\xi}\right)^{\top}\left(\beta\lambda\mat{I}_{N}+\hat{V}\Omega\right)^{-1}\left(\hat{m}\Phi\tilde{\vec{w}}+\sqrt{\hat{q}}\Omega^{1/2}\vec{\xi}\right)^{\top}\right)}{\sqrt{\det\left(\beta\lambda\mat{I}_{N}+\hat{V}\Omega\right)}}
\end{align}
\noindent where we have included a convenient rescaling of $P_{w}$. We can now take the log and average the resulting expression explicitly with respect to $P_{\tilde{w}} = \mathcal{N}(0,\mat{I}_{N})$ and $\vec{\xi}\sim \mathcal{N}(0,\mat{I}_{N})$. After some linear algebra manipulation, we can write the result (up to the limit) as:
\begin{align}
\Psi_{w} &=	-\frac{1}{2N}	\tr\log \left(\beta\lambda\mat{I}_{N}+\hat{V}\Omega\right)+\frac{1}{2N}\tr\left[\left(\hat{m}^2\Phi\Phi^{\top}+\hat{q}\Omega\right) \left(\beta\lambda\mat{I}_{N}+\hat{V}\Omega\right)^{-1}\right]
\end{align}

\subsection{Saddle-point equations}
In order to find the set of overlaps $(r^{\star}, \hat{r}^{\star}, q^{\star}, \hat{q}^{\star},m^{\star},\hat{m}^{\star})$ that solve the extremisation problem in eq.~\eqref{eq:freeen:final}, we look at the gradient of the replica symmetric potential. This give us a set of self-consistent equations known as \emph{saddle-point equations}. 

First, taking the gradient of $\Psi_{y}$ with respect to $(r,q,m)$ and recalling that $V=r-q$:
\begin{align}
    \partial_{r}\Psi_{y} &=- \mathbb{E}_{\xi}\left[\int_{\mathbb{R}}\dd y~\tilde{\mathcal{Z}}_{y}\frac{\partial_{\omega}\mathcal{Z}_{y}^2}{\mathcal{Z}_{y}}\right], &&
    \partial_{q}\Psi_{y} &= \mathbb{E}_{\xi}\left[\int_{\mathbb{R}}\dd y~\tilde{\mathcal{Z}}_{y}f_{y}^2 \right], &&
    \partial_{m}\Psi_{y} &= \mathbb{E}_{\xi}\left[\int_{\mathbb{R}}\dd y~\partial_{\omega}\tilde{\mathcal{Z}}_{y}f_{y}\right]\notag
\end{align}
\noindent where $f_{y} \equiv \log\mathcal{Z}_{y}$. Now looking at the gradient of $\Psi_w$ with respect to $(\hat{r}, \hat{q},\hat{m})$ and recalling that $\hat{V} = \hat{r}+\hat{q}$:
\begin{align}
	\partial_{\hat{r}}\Psi_{w} &= -\frac{1}{2N}\tr\left(\beta\lambda\mat{I}_{N}+\hat{V}\Omega\right)^{-1}\Omega-\frac{\hat{m}^{2}}{2N}\tr\left(\beta\lambda\mat{I}_{N}+\hat{V}\Omega\right)^{-2}\Omega\Phi\Phi^{\top}-\frac{\hat{q}}{2N}\tr\left(\beta\lambda\mat{I}_{N}+\hat{V}\Omega\right)^{-2}\Omega^2\notag\\
\partial_{\hat{q}}\Psi_{w} &= -\frac{\hat{m}^{2}}{2N}\tr\left(\beta\lambda\mat{I}_{N}+\hat{V}\Omega\right)^{-2}\Omega\Phi\Phi^{\top}-\frac{\hat{q}}{2N}\tr\left(\beta\lambda\mat{I}_{N}+\hat{V}\Omega\right)^{-2}\Omega^2\notag\\
\partial_{\hat{m}}\Psi_{w} &= \frac{\hat{m}}{d}\tr \Phi^{\top}\Phi\left(\beta\lambda\mat{I}_{N}+\hat{V}\Omega\right)^{-1}
\end{align}
Putting together give the following set of self-consistent saddle-point equations:
\begin{align}
	\begin{cases}
		\hat{V} = \alpha \mathbb{E}_{\xi}\left[\int_{\mathbb{R}}\dd y~\tilde{\mathcal{Z}}_{y} \partial_{\omega}f_{y}\right]\\
		\hat{q} = \alpha \mathbb{E}_{\xi}\left[\int_{\mathbb{R}}\dd y~\tilde{\mathcal{Z}}_{y}f_{y}^2 \right]\\
		\hat{m} = \frac{\alpha}{\sqrt{\gamma}} \mathbb{E}_{\xi}\left[\int_{\mathbb{R}}\dd y~\partial_{\omega}\tilde{\mathcal{Z}}_{y}f_{y}\right]
	\end{cases} && 
	\begin{cases}
		V =  \frac{1}{N}\tr\left(\beta\lambda\mat{I}_{N}+\hat{V}\Omega\right)^{-1}\Omega\\
		q = \frac{1}{N}\tr\left[\left(\hat{q}\Omega+\hat{m}^{2}\Phi\Phi^{\top}\right)\Omega\left(\beta\lambda\mat{I}_{N}+\hat{V}\Omega\right)^{-2}\right]\\
		m= \frac{\hat{m}}{N\sqrt{\gamma}}\tr \Phi\Phi^{\top}\left(\beta\lambda\mat{I}_{N}+\hat{V}\Omega\right)^{-1}
	\end{cases}
\end{align}
\noindent where we used  $\partial_{\omega}f_{y} = \mathcal{Z}_{y}^{-1}\partial_{\omega}^2\mathcal{Z}-f_{y}^2$. To take the $\beta\to\infty$ limit explicitly, we look at the following ansatz for the scaling of the order parameters:
\begin{align}
V^{\infty}=\beta V && q^{\infty}= q && m^{\infty}= m\notag\\
\hat{V}^{\infty}=\frac{1}{\beta} \hat{V} && \hat{q}^{\infty} = \frac{1}{\beta^2}\hat{q} && \hat{m}^{\infty} = \frac{1}{\beta}\hat{m}.
\end{align}
With this scaling, we can easily get rid of the $\beta$ dependency in the equations for $(V, q, m)$. For the $(\hat{V},\hat{q},\hat{m})$ equations, we note that:
\begin{align}
    \mathcal{Z}_{y}(y,\sqrt{q}\xi,V) = \int\frac{\dd x}{\sqrt{2\pi V}}e^{-\frac{(x-\sqrt{q}\xi)^2}{2V}} e^{-\beta\ell(y,x)} = \int\frac{\dd x}{\sqrt{2\pi V}}e^{-\beta\left[\frac{(x-\sqrt{q^{\infty}}\xi)^2}{2V^{\infty}}+\ell(y,x)\right]}
\end{align}
\noindent and therefore when $\beta\to\infty$, $\mathcal{Z}_{y}$ is dominated by the exponential of the values that minimise the argument in the exponent, which is the \emph{proximal operator} associated to the loss $\ell$:
\begin{align}
\eta(y,\omega,V) = \underset{x\in\mathbb{R}}{\argmin}\left[\frac{(x-\omega)^2}{2V}+\ell\left(y, x\right)\right]
\end{align}
Finally, in the $\beta\to\infty$ limit the saddle-point equations can be written as:
\begin{align}
	\begin{cases}
		\hat{V} = \alpha \mathbb{E}_{\xi}\left[\int_{\mathbb{R}}\dd y~\tilde{\mathcal{Z}}_{y} \left(\frac{1-\partial_{\omega}\eta}{V}\right)\right]\\
		\hat{q} = \alpha \mathbb{E}_{\xi}\left[\int_{\mathbb{R}}\dd y~\tilde{\mathcal{Z}}_{y} \left( \frac{\eta	-\omega}{V}\right)^2 \right]\\
		\hat{m} = \frac{\alpha}{\sqrt{\gamma}} \mathbb{E}_{\xi}\left[\int_{\mathbb{R}}\dd y~\partial_{\omega}\tilde{\mathcal{Z}}_{y}\left(\frac{\eta-\omega}{V}\right) \right]
	\end{cases} && 
	\begin{cases}
		V =  \frac{1}{N}\tr\left(\lambda\mat{I}_{N}+\hat{V}\Omega\right)^{-1}\Omega\\
		q = \frac{1}{N}\tr\left[\left(\hat{q}\Omega+\hat{m}^{2}\Phi\Phi^{\top}\right)\Omega\left(\lambda\mat{I}_{N}+\hat{V}\Omega\right)^{-2}\right]\\
		m= \frac{\hat{m}}{N\sqrt{\gamma}}\tr \Phi\Phi^{\top}\left(\lambda\mat{I}_{N}+\hat{V}\Omega\right)^{-1}
	\end{cases}\label{eq:app:saddlepoint}
\end{align}
\noindent where we have dropped the $\cdot^{\infty}$ superscript to lighten the notation. This is the expression quoted on the main text. Note that for convex loss functions, the problem in eq.~\eqref{eq:argmin} is strongly convex, and therefore admit one and only one solution $\hat{w}$. This implies that the solution for the overlaps $(m^{\star}, q^{\star})$ found by iterating the saddle-point equations above \emph{necessarily} coincides with the overlaps appearing in the expression for the generalisation error given by eq.~\eqref{eq:app:generror}. This means that the replica symmetric fully characterises the generalisation performance in the convex case.

\section{Further experimental results}
\label{sec:furth-exper-results}

\paragraph{Results for online SGD with the pre-trained dcGAN}

\begin{figure}[t!]
  \centering
  \includegraphics[align=c, width=.33\textwidth]{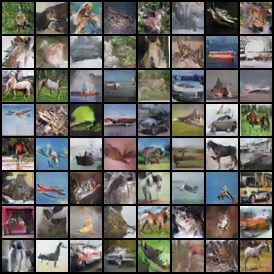}\quad%
  \includegraphics[align=c, width=.6\textwidth]{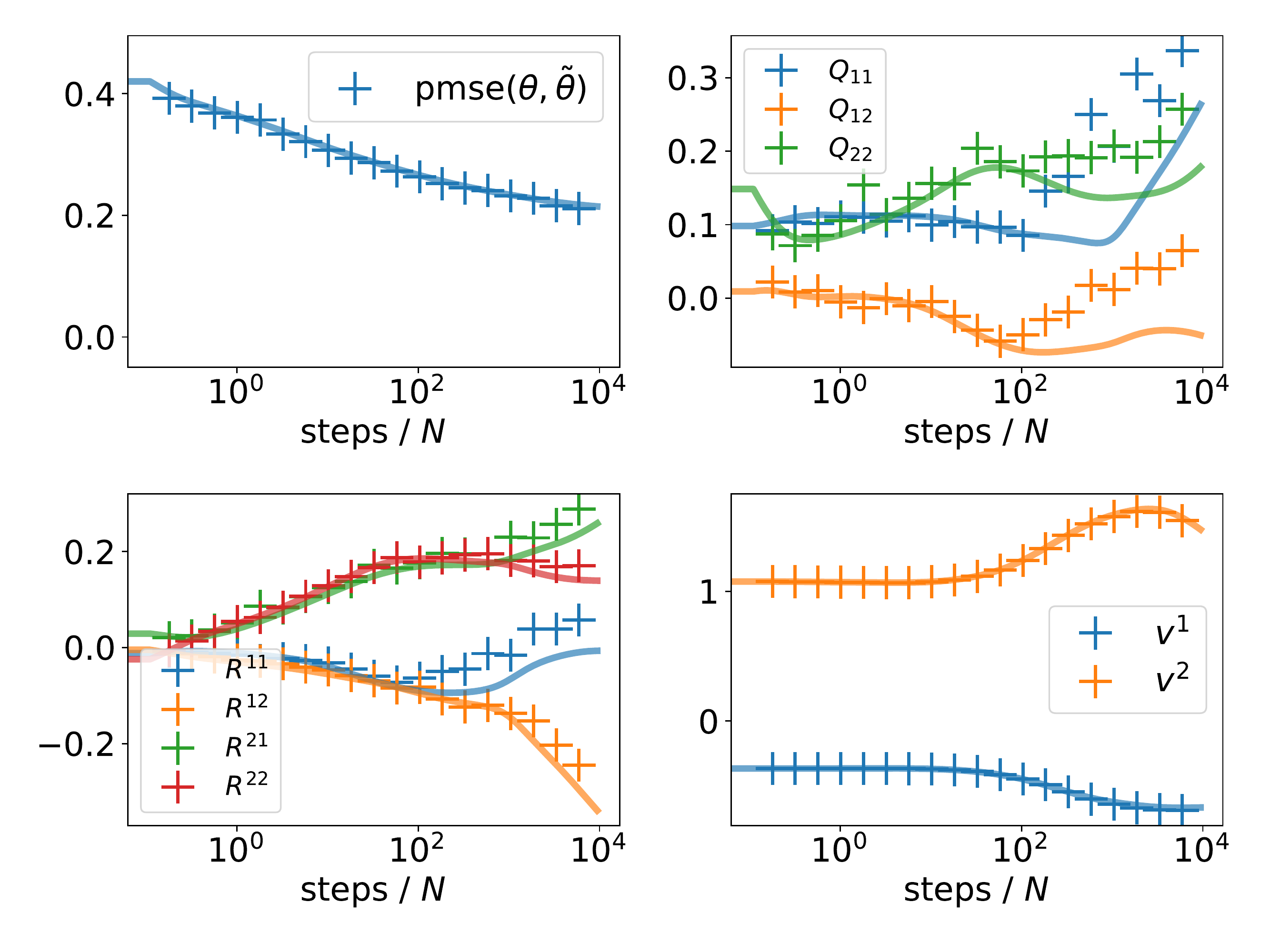}\\
  \caption{\label{fig:pretrained_dcgan}\textbf{Theory vs experiments for online SGD
      with deep, pre-trained dcGAN of~\citet{radford2015unsupervised}}.
    \emph{(Left)} The top four rows show images drawn randomly from the CIFAR10
    data set, the bottom four rows show images drawn randomly from the
    pre-trained dcGAN. \emph{(Right)} Same plot as Fig.~\ref{fig:dcgan_rand}
    when inputs are drawn from the pre-trained realNVP.
    $D=N=3072, M=K=2, \tilde v^m=1, \eta=0.2, g(x) = \tilde g(x) =
    \erf(x/\sqrt{2})$, integration time step $\dd t=0.01$.}
\end{figure}

We also compared the dynamical equations to simulations in the case of the dcGAN
pre-trained on CIFAR10 images, see Fig~\ref{fig:pretrained_dcgan}. We see that
in this case, the equations capture the evolution of the $\pmse$ well and
exactly predict the evolution of the second-layer weights $v$. This is a crucial
result, since we obtain these predictions from analytical expressions for the
functions $h_{(7)}^{kn}$ and $h_{(8)}^{kj}$ that are only valid if the GEP
holds. One can therefore interpret the correct predictions for $v$ based on the
GEP as experimental evidence that the GEP holds for this pre-trained
convolutional generators. The results for the order parameters $Q$ and $R$
reveal larger fluctuations after about $100N\sim 10^5$ SGD steps, for example
for $Q^{11}$ (blue line in top right plot). One source of error here is
numerical and due to the small size of the teacher network $(D=100)$ to which we
are comparing a theory that holds asymptotically, i.e.\ when $N,
D\to\infty$. Such a small teacher would lead to deviations from the ODEs due to
finite-size effects even for i.i.d.\ Gaussian inputs. To confirm that these
deviations are finite-size effects, we also verified our theory for a different
class of generative model, the aforementioned normalising flows, who have a
larger latent dimension $D$. As we see in Sec.~\ref{sec:normalising-flows}, the
ODEs perfectly agree with simulations for this model with larger input
dimension.

\paragraph{Generative model with strongly correlated weights}
\label{sec:supp_fc_inverse}

\begin{figure}[t!]
  \centering
  \includegraphics[align=c, width=.48\textwidth]{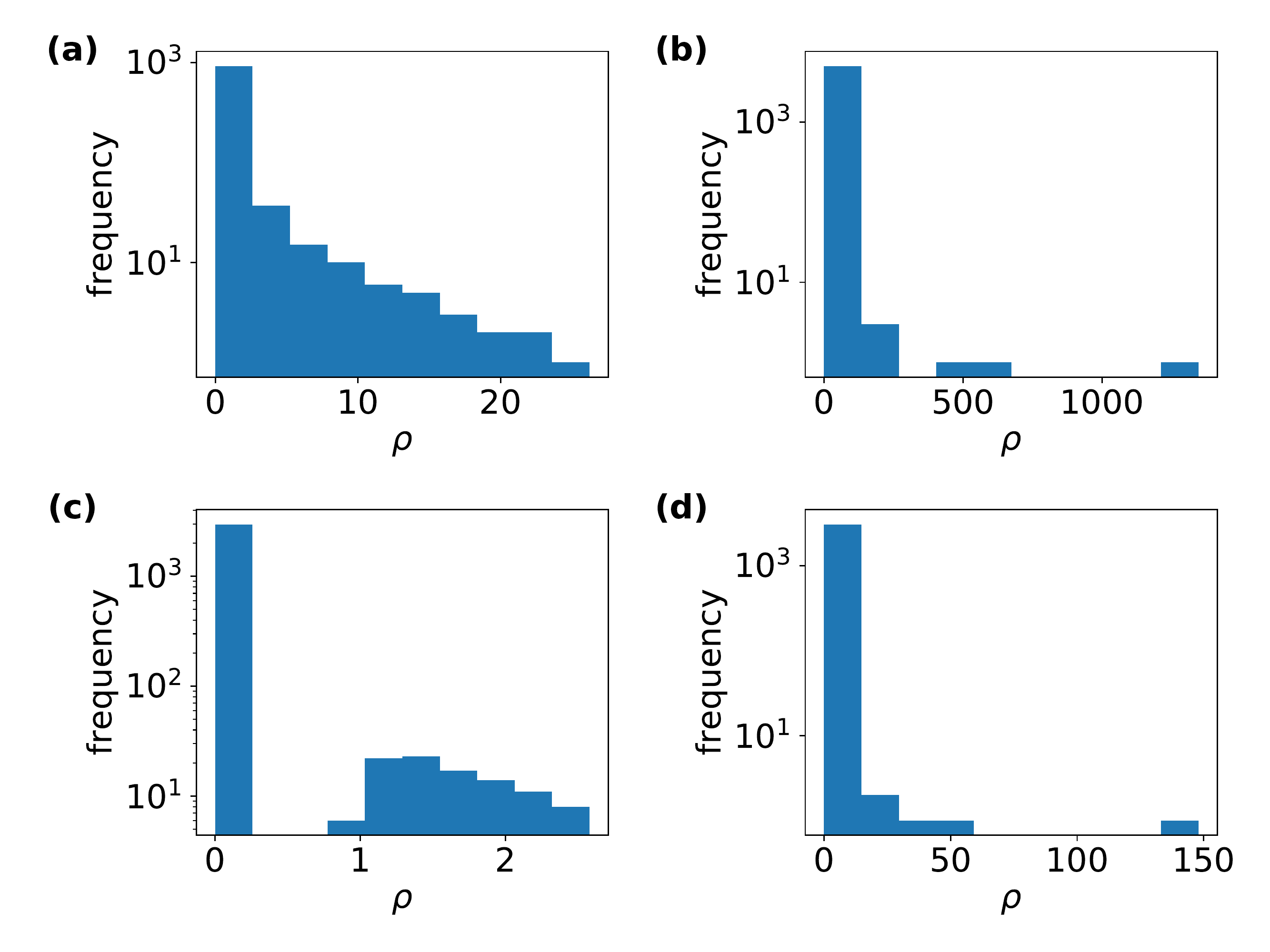}\quad%
  \includegraphics[align=c, width=.48\textwidth]{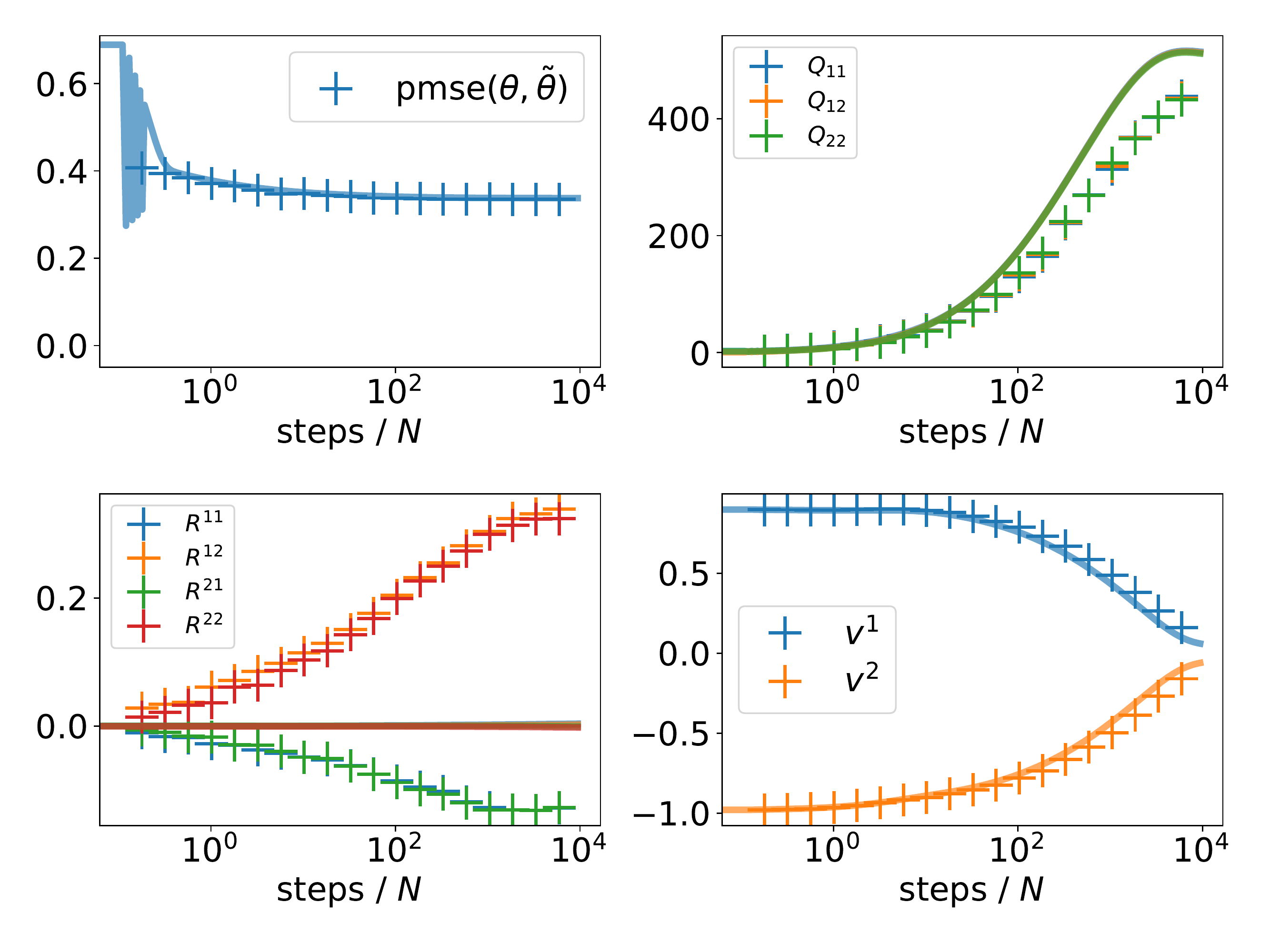}
  \caption{\label{fig:fc_inverse}\textbf{The impact of the spectral density of
      the input-input covariance on learning.} \emph{(Left)}: Spectral density
    of the average covariance matrix of inputs drawn from four generative
    models: (a)~Random fully-connected network of Fig.~\ref{fig:dcgan_rand},
    (b)~fully connected generator with inverse weights (see
    Sec.~\ref{sec:supp_fc_inverse}), (c)~dcGAN with random weights, and
    (d)~dcGAN trained on CIFAR10. \emph{(Right):} We compare theory vs
    simulation for the training of two-layer neural network on inputs $x$ drawn
    from a two-layer, fully connected generative network where the weights of
    the second layer are the matrix inverse of the first layer,
    Eq.~\eqref{eq:fc_inverse}.
    $D=5000, N=5000, M=K=2, \tilde v^m=1, \eta=0.2, g(x) = \tilde g(x) =
    \erf(x/\sqrt{2})$, integration time step $\dd t=0.01$.}
\end{figure}

Finally, we also constructed a generative model with strongly correlated weights
where there exists a dominant direction in the eigenspace of the input-input
covariance matrix $\Omega_{ij}=\EE x_i x_j$. We took a fully connected
generative network $\mathcal{G}: \reals^N \to \reals^N$, with two layers of
weights $A^1\in\mathbb{R}^{N \times N}$ and $A^2\in\mathbb{R}^{N \times N}$. We
drew the elements of $A^1$ element-wise i.i.d.\ from the standard normal
distribution, whereas the second-layer weights $A^2=\mathrm{inv}(A^1)$. After
each layer, we used the sign activation function, so the generator's output
function can be written as
\begin{equation}
  \label{eq:fc_inverse}
  x = \mathcal{G}(c) = \mathrm{sign}\left( \mathrm{inv}(A^1) \mathrm{sign}\left( A^1 c \right) \right)
\end{equation}

On the left of Fig.~\ref{fig:fc_inverse}, we show the spectra of the covariance
matrices of various generators. The leading eigenvalues are smallest for
generators with random weights, such as the fully-connected single-layer
network~\eqref{eq:proof-generators} (a) and the dcGAN with random weights (c)
that we used in Fig.~\ref{fig:dcgan_rand}. The pre-trained dcGAN has a leading
eigenvalue that is about an order of magnitude larger (d). The generator with
inverse weights~\eqref{eq:fc_inverse} has an eigenvalue that is yet another
order of magnitude larger.  

The particular weight structure of the ``inverse'' generator also has a strong
impact on the dynamics of a two-layer network trained on its data, as we show on
the right of Fig.~\ref{fig:fc_inverse}. Notably, the length of the weight
vectors grows exponentially for a large portion of training time, while the
second-layer weights go to zero. We observed this behaviour consistently over
several runs of this setup with different weights for the teacher, generator and
different initial weights for the student in each case. Characterising the
impact of a dominant direction in the data on the dynamics of two-layer neural
networks is an intriguing challenge that we leave for future work.